\documentclass[twoside,11pt]{article}
\usepackage{graphicx}
\usepackage{amsmath,amsfonts,amssymb,bm,eucal,amstext}
\usepackage[usenames,dvipsnames]{xcolor}
\usepackage{xspace}
\usepackage{algorithm}
\usepackage[noend]{algpseudocode}
\usepackage{verbatim}
\usepackage{subcaption}
\usepackage{enumitem}
\usepackage{lastpage}
\usepackage{sidecap}
\usepackage{jmlr2e}

\hypersetup{ hidelinks }

\newtheorem{assumption}{Assumption}

\jmlrheading{24}{2023}{1-\pageref{LastPage}}{11/21; Revised 7/22}{3/23}{21-1350}{Eric Graves, Ehsan Imani, Raksha Kumaraswamy, and Martha White}

\ShortHeadings{OFF-POLICY ACTOR-CRITIC WITH EMPHATIC WEIGHTINGS}{Graves, Imani, Kumaraswamy, and White}
\firstpageno{1} 

\title{Off-Policy Actor-Critic with Emphatic Weightings}

\author{\name Eric Graves \email graves@ualberta.ca \\
  \name Ehsan Imani \email imani@ualberta.ca \\
  \name Raksha Kumaraswamy \email kumarasw@ualberta.ca \\
  \name Martha White \email whitem@ualberta.ca \\
  \addr Reinforcement Learning and Artificial Intelligence Laboratory\\
  Department of Computing Science, University of Alberta\\
  Edmonton, Alberta, Canada T6G 2E8
}

\editor{Joelle Pineau}

\allowdisplaybreaks[4]
\def\1{{\bf 1}}
\def\0{{\bf 0}}

\def\SS{{\cal S}}
\def\AA{{\cal A}}

\newcommand{\emweight}{m}
\newcommand{\emvec}{\mathbf{m}}
\newcommand{\ivec}{\mathbf{i}}
\newcommand{\gvec}{\mathbf{g}}
\newcommand{\gtvec}{\tilde{\mathbf{g}}}
\newcommand{\eye}{\mathbf{I}}
\newcommand{\vpivecdot}{\dot{\mathbf{v}}_\pi}

\newcommand{\vtpivecdot}{{\dot{\tilde{\mathbf{v}}}}_\pi}

\renewcommand{\vec}[1]{\boldsymbol{\mathbf{#1}}}
\newcommand{\intd}{\mathop{}\!\mathrm{d}}

\newcommand{\States}{\mathcal{S}}
\newcommand{\Actions}{\mathcal{A}}
\newcommand{\xvec}{\mathbf{x}}

\newcommand{\Pfcn}{\text{P}}
\newcommand{\Ppig}{\mathbf{P}_{\!\pi,\gamma}}

\newcommand{\dmu}{\mathbf{d}_{\mu}}

\newcommand{\dvec}{\mathbf{d}}

\newcommand{\interest}{i}

\newcommand{\inv}{{-1}}

\newcommand{\defeq}{\mathrel{\overset{\makebox[0pt]{\mbox{\normalfont\tiny\sffamily def}}}{=}}}
\newcommand{\E}{\mathbb{E}}
\newcommand{\RR}{\mathbb{R}}

\newcommand{\stepsize}{\alpha}

\newcommand{\pparams}{{\boldsymbol{\theta}}}
\newcommand{\psivec}{\boldsymbol{\psi}}

\newcommand{\pdim}{d}

\newcommand{\fparams}{{\boldsymbol{\phi}}}

\newcommand{\ppinf}{\pparams_\infty}

\newcommand{\vpi}{v_{\pi}}
\newcommand{\qpi}{q_{\pi}}

\newcommand{\vtpi}{\tilde{v}_{\pi}}
\newcommand{\qtpi}{\tilde{q}_{\pi}}

\newcommand{\zerovec}{\mathbf{0}}

\newcommand{\Gmat}{\mathbf{G}}
\newcommand{\Gtmat}{\mathbf{\tilde{G}}}

\newcommand{\lambdaa}{{\eta}}

\DeclareMathOperator{\diag}{diag}
\DeclareMathOperator{\alias}{alias}

\usepackage{enumitem,amssymb}
\usepackage{pifont}
\usepackage[thinc]{esdiff}
\newlist{todolist}{itemize}{10}
\setlist[todolist]{label=$\square$}

\newcommand{\figwidthfour}{0.3\textwidth}

\begin{document}

\maketitle

\begin{abstract}%
A variety of theoretically-sound policy gradient algorithms exist for the on-policy setting due to the policy gradient theorem, which provides a simplified form for the gradient.
The off-policy setting, however, has been less clear due to the existence of multiple objectives and the lack of an explicit off-policy policy gradient theorem.
In this work, we unify these objectives into one off-policy objective, and provide a policy gradient theorem for this unified objective.
The derivation involves \emph{emphatic weightings} and \emph{interest functions}.
We show multiple strategies to approximate the gradients, in an algorithm called Actor Critic with Emphatic weightings (ACE).
We prove in a counterexample that previous (semi-gradient) off-policy actor-critic methods---particularly Off-Policy Actor-Critic (OffPAC) and Deterministic Policy Gradient (DPG)---converge to the wrong solution whereas ACE finds the optimal solution.
We also highlight why these semi-gradient approaches can still perform well in practice, suggesting strategies for variance reduction in ACE.
We empirically study several variants of ACE on two classic control environments and an image-based environment designed to illustrate the tradeoffs made by each gradient approximation.
We find that by approximating the emphatic weightings directly, ACE performs as well as or better than OffPAC in all settings tested.
\end{abstract}


\section{Introduction}

Policy gradient methods are a general class of algorithms for learning optimal policies for both the on and off-policy settings. In policy gradient methods, a parameterized policy is improved using gradient ascent \citep{williams1992simple}, with seminal work in actor-critic algorithms \citep{witten1977adaptive,barto1983neuronlike} and many techniques since proposed to reduce variance of the estimates of this gradient \citep{konda2000actor,weaver2001optimal,greensmith2004variance,peters2005natural,bhatnagar2007incremental,bhatnagar2009natural,grondman2012survey,gu2016q}. These algorithms rely on a fundamental theoretical result: the \emph{policy gradient theorem}. This theorem \citep{sutton2000policy,marbach2001simulation} simplifies estimation of the gradient, which would otherwise require difficult-to-estimate gradients with respect to the stationary distribution of the policy and potentially of the action-values. 

These gradients can be sampled \emph{on-policy} or \emph{off-policy}. On-policy methods are limited to learning about the agent's current policy: the policy must be executed to obtain a sample of the gradient. Conversely, in off-policy learning the agent can learn about many policies that are different from the policy being executed. Sampling these gradients is most straightforward in the on-policy setting, and so most policy gradient methods are on-policy.

Off-policy methods, however, are critical in an online setting, where an agent generates a single stream of interaction with its environment. Off-policy methods can learn from data generated by older versions of a policy, known as experience replay, a critical factor in the recent success of deep reinforcement learning \citep{lin1992self,mnih2015human,schaul2015prioritized}. They also enable learning from other forms of suboptimal data, including data generated by human demonstration, non-learning controllers, and even random behaviour. Off-policy methods also enable learning about the optimal policy while executing an exploratory policy \citep{watkins1992q}, thereby addressing the exploration-exploitation tradeoff. Finally, off-policy methods allow an agent to learn about many different policies at once, forming the basis for a predictive understanding of an agent's environment \citep{sutton2011horde,white2015developing} and enabling the learning of options \citep{sutton1999between,precup2000temporal,klissarov2021flexible}. With options, an agent can determine optimal (short) behaviours from its current state. 

Off-policy policy gradient methods have been developed, particularly in recent years where the need for data efficiency and decorrelated samples in deep reinforcement learning require the use of experience replay and so off-policy learning.
This work began with the Off-Policy Actor-Critic algorithm (OffPAC) \citep{degris2012model}, where an off-policy policy gradient theorem was provided that parallels the on-policy policy gradient theorem, but only for tabular policy representations.\footnote{See B. Errata in \cite{degris2012offpolicy} that the theorem only applies to tabular policy representations.} This motivated further development, including a recent actor-critic algorithm proven to converge when the critic uses linear function approximation \citep{maei2018convergent}, as well as several methods using the approximate off-policy gradient such as Deterministic Policy Gradient (DPG) \citep{silver2014deterministic,lillicrap2015continuous}, Actor-Critic with Experience Replay (ACER) \citep{wang2016sample}, and Interpolated Policy Gradient (IPG) \citep{gu2017interpolated}.
However, it remained an open question whether the foundational theorem that underlies these algorithms can be generalized beyond tabular representations.

This question was resolved with the development of an off-policy policy gradient theorem, for general policy parametrization \citep{imani2018off}.
The key insight is that the gradient can be simplified if the gradient in each state is weighted with an emphatic weighting. This result, combined with previous methods for incrementally estimating emphatic weightings \citep{yu2015onconvergence,sutton2016anemphatic}, allowed for the design of a new off-policy actor-critic algorithm, called Actor-Critic with Emphatic Weightings (ACE). Afterwards, an algorithm was proposed that directly estimates the emphatic weightings using a gradient temporal difference update, with an associated proof of convergence using the standard two-timescale analysis \citep{zhang2020provably}.

However, ACE and the underlying off-policy policy gradient theorem do not obviously resolve the dilemma of computing the off-policy gradient, because ACE---and previously OffPAC---were introduced under what is called the \emph{excursions} objective. This objective weights states by the visitation distribution of the behaviour policy. This is sensible in the parallel off-policy setting, where many policies are learned in parallel from one stream of experience. The agent reasons about the outcomes of those policies when taking excursions from its current behaviour. In the episodic setting, however, weighting the states by the visitation distribution of the behaviour policy is not appropriate. Instead, an ideal episodic objective should be weighted by the start-state distribution, and this objective should be optimized from off-policy data without changing this weighting. 

Potentially because of this mismatch, most recent methods have pursued strategies
that directly correct state-weightings,\footnote{A notable exception is an approach that uses a nonparametric Bellman equation---essentially using nonparametric estimates of the environment model \citep{tosatto2020nonparametric}. 
This approach allows direct estimation of the gradient by taking derivatives through this nonparametric Bellman equation.
} to obtain gradients of either the episodic objective \citep{liu2020offpolicy} or a finite-horizon objective \citep{kallus2020statistically}. This approach involves estimating the distribution of state visitation under the behaviour policy and the discounted state visitation under the target policy to obtain an importance sampling ratio to reweight the updates. One disadvantage to this approach is that the state visitation distribution must be estimated, and these estimates can be extremely biased---both by the initialization and by the limitations of the parametric function class used for the estimates.
  
In this work, we revisit estimating the off-policy gradient with emphatic weightings.
First, we propose a unifying objective for the episodic setting and the excursions objective. This objective facilitates the design of a single, unifying off-policy algorithm pertinent to both settings. Then, we extend the off-policy policy gradient theorem both to this objective and to objectives that incorporate entropy regularization. We further show that the on-policy episodic objective can also be seen as a special case, through the use of interest functions to specify interest in the set of start states. We then highlight the difference to previous off-policy approaches including OffPAC and DPG, which use a biased gradient that we call the semi-gradient because it omits a component of the gradient. We prove that in a simple three state counterexample, with two states aliased, the stationary point under semi-gradients is suboptimal, unlike ACE which converges to the optimal solution. 

Though such state-aliasing is not pathological, it does seem to contradict the success of methods built on OffPAC in the literature. We show that, under a condition called the \emph{strong growth condition} \citep{schmidt2013fast}, which is obtained if there is sufficient representational capacity, the semi-gradient solution is equivalent to that of ACE. Further, we highlight that even outside this setting, in some cases the semi-gradient can be seen as using the sound update underlying ACE, but with a different state weighting. These two insights help explain why OffPAC has been surprisingly effective and suggest promising directions for a lower-variance version of ACE. 

Finally, we discuss several improvements to the algorithm, including using direct estimates of the emphatic weightings to reduce variance, incorporating experience replay, and balancing between the higher-variance full gradient update and the lower-variance semi-gradient update.
We empirically investigate variants of ACE in several benchmark environments, and find that those with direct estimates of the emphatic weighting consistently perform as well as or better than OffPAC. 

\textbf{Remark on Contributions:}
This paper builds on our conference paper \citep{imani2018off}.
The conference paper presented a policy gradient theorem for one off-policy objective function, and a counterexample showing that OffPAC and DPG can converge to the wrong solution, with experiments limited to the simple three-state counterexample.

An important limitation of that work is that it only applies to the excursions objective, which does not obviously relate to the standard policy gradient objective. Further, the work primarily investigated the true gradient, and did not provide much insight into how to practically use the algorithm.
This paper builds on the conference paper in several substantive ways.
\begin{enumerate}[itemsep=-.5em]
    \item We first clarify the confusion surrounding objectives by providing a more general objective that includes both the standard objective and the excursions objective. We provide all the derivations for this more general objective (their Theorems 1 and 2 and Proposition 1, which are our Theorems 2 and 3 and Proposition 6). The result is nearly identical, but reiterated with the more general definition, and slightly different terminology.
    We also highlight that we can recover a sound on-policy algorithm using interest functions.
    \item We prove that the semi-gradient from OffPAC has suboptimal solutions on the counterexample. The conference paper only showed this empirically. To obtain this result, we needed to extend the derivation to include entropy regularization, so that the optimal policy is stochastic and its parameters are bounded. Then we proved that the stationary points for OffPAC on the counterexample do not include the optimal point.
    \item We provide insight into why semi-gradient methods can perform well, by showing that in some cases the weighting does not affect the solution and that locally we can characterize the semi-gradient as a gradient update with a different state weighting.
    \item We provide several algorithmic improvements to ACE, particularly by directly estimating the emphatic weightings to give a lower variance update. In our experiments, this algorithm outperforms or performs comparably to all others, including OffPAC.
    \item We conduct a more thorough empirical investigation of ACE in more settings, including two classic benchmarks and a new partially-observable image-based domain. The experiments are aimed at deeply investigating the role of different components, including (a) the objective used for learning versus evaluation, (b) the algorithm for estimating the emphatic weightings, and (c) the algorithm used for the critic.
\end{enumerate}

\section{Problem Formulation}
Throughout the paper, we use bolded symbols to represent vectors and matrices, and uppercase italicized symbols to represent random variables.
We consider a Markov decision process ($\SS$, $\AA$, P, $r$),
where $\SS$ denotes the set of states,
$\AA$ denotes the set of actions,
P $: \SS \times \AA \to \Delta(\SS)$ denotes the one-step state transition dynamics,
and $r : \SS \times \AA \times \SS \to \mathbb{R}$ denotes the transition-based reward function.
At each timestep $t=1,2,\ldots$, the agent selects an action $A_t$ according to its behaviour policy $\mu$, where $\mu : \SS \to \Delta(\AA)$.
The environment responds by transitioning into a new state $S_{t+1}$ according to P, and emits a scalar reward $R_{t+1}=r(S_t,A_t,S_{t+1})$.

The discounted sum of future rewards given actions are selected according to some target policy $\pi$ is called the return, and defined as:
\begin{align}
  G_t &\defeq R_{t+1} + \gamma_{t+1}R_{t+2} + \gamma_{t+1}\gamma_{t+2}R_{t+3} + \ldots
  \\
  &= R_{t+1} + \gamma_{t+1}G_{t+1} \nonumber
  .
\end{align}
We use transition-based discounting $\gamma : \SS \times \AA \times \SS \to [0,1]$, as it facilitates specifying different tasks (termination) on top of the same MDP \citep{white2017unifying}. This generality is useful in our setting, where we may be interested in learning multiple (option) policies, off-policy. The state value function for $\pi$ and $\gamma$ is defined as
\begin{equation}
  v_{\pi}(s) \defeq \E_{\pi}[G_{t}| S_t = s] \quad \forall s \in \SS,
\end{equation}
where under discrete states and actions this corresponds to
  \begin{equation*}
  v_{\pi}(s) = \sum_{a \in \AA} \pi(a|s) \sum_{s' \in \SS} \text{P}(s'|s,a)[r(s,a,s') + \gamma(s,a,s') v_{\pi}(s')] \quad \forall s \in \SS
\end{equation*}
and under continuous states and actions this corresponds to
 \begin{equation*}
 v_{\pi}(s) = \int_{\AA} \pi(a|s) \int_{\SS} \text{P}(s'|s,a)[r(s,a,s') + \gamma(s,a,s') v_{\pi}(s')] \intd a \intd s' \quad \forall s \in \SS
 .
\end{equation*}
When possible, we will opt for the more generic expectation notation, to address both the finite and continuous cases. 

In off-policy control, the agent's goal is to learn a target policy $\pi$ while following the behaviour policy $\mu$.
The target policy $\pi_{\boldsymbol{\theta}}$ is a differentiable function of a weight vector $\boldsymbol{\theta} \in \mathbb{R}^{n_{\boldsymbol{\theta}}}, n_{\boldsymbol{\theta}} \in \mathbb{N}$. The goal of the agent is to learn a policy that maximizes total reward. Depending on the setting, however, this goal is specified with different objectives, as we discuss in the next section. 

Throughout, we assume that the limiting distribution $d_\mu(s)$ under $\mu$ exists, where 
\begin{equation}
d_\mu(s) \defeq \lim_{t\to\infty}\text{P}(S_t=s|s_0,\mu)
\end{equation}
and P$(S_t=s|s_0,\mu)$ is the probability---or density---that $S_t=s$ when starting in state $s_0$ and executing $\mu$. Similarly, $d_\pi$ denotes the limiting distribution\footnote{This term is sometimes overloaded to mean the discounted state visitation starting from a single start state $s_0$: $\sum_{t=0}^\infty \gamma^t \text{P}(S_t = s | s_0, \pi)$. This overloading comes from the definition given in the original policy gradient theorem \citep{sutton2000policy}. We only use it to mean state visitation (limiting distribution).} of states under $\pi$.
Note that this limiting distribution exists in either the episodic or continuing settings \citep{white2017unifying,bojun2020steady}. One way to see this is that in the episodic setting under transition-based discounting, there are no explicit terminal states. Rather, the agent transitions between the state right before termination, back to a start state, as it would for a continuing problem. The transition-based discount truncates returns, to specify that the goal of the agent from each state is to maximize return for the episode.

\section{The Weighted-Excursions Objective for Off-Policy Control}\label{sec:obj_off_policy}


In this section we introduce the weighted-excursions objective that unifies two objectives commonly considered in off-policy control. The objective encompasses the standard on-policy episodic objective as well as the excursions objective that allows option policies to be learned off-policy with different termination conditions. Throughout, we will define both the finite state and continuous state versions. 

The standard episodic objective is
\begin{equation}\label{ep_obj}
J_0(\pparams) \defeq \sum_{s \in \States} d_0(s) v_{\pi_\pparams}(s) \hspace{1.0cm} \text{ or } \hspace{1.0cm} J_0(\pparams)  \defeq \int_{\States} d_0(s) v_{\pi_\pparams}(s) \intd s
\end{equation}
where $d_0 \in \Delta(\mathcal{S})$ is the start state distribution. This objective can be optimized in either the on-policy or off-policy settings, with different updates to account for the fact that data is gathered according to $\mu$ in the off-policy setting.

Another objective that has been considered in the off-policy setting is the \emph{excursions objective} \citep{degris2012offpolicy} 
\begin{equation}\label{exc_obj_old}
J_{\text{exc}}(\pparams) \defeq \sum_{s \in \States} d_\mu(s) v_{\pi_\pparams}(s),
\end{equation}
where the state weighting is determined by the behaviour policy's limiting state distribution $d_\mu$. This objective assumes that the target policy $\pi_{\pparams}$ will be executed as an excursion from the current behaviour---or that the agent will use the target policy in planning under the current state visitation. In the options framework \citep{sutton1999between}, for example, the agent learns the optimal policies for a variety of different options, each with different termination conditions. The behaviour policy itself $\mu$ could be learning in a continuing environment, even though the objective for the option policy $\pi_\theta$ is episodic, in that it has termination conditions. 

The excursions objective, however, is not necessarily appropriate for learning an alternative optimal policy. For example, consider an episodic setting where a reasonable behaviour policy is currently being used, and we would like to learn an improved policy. The improved policy will be executed from the set of start states, not from the visitation distribution under the behaviour policy. In this setting, we would like to optimize the standard episodic objective, using off-policy data generated by the current behaviour policy. 

These two settings can be unified by considering the \emph{weighted-excursions} objective
%
\begin{equation}\label{exc_obj}
J_\mu(\pparams) \defeq \sum_{s \in \States} d_\mu(s) i(s) v_{\pi_\pparams}(s)  \hspace{1.0cm} \text{ or } \hspace{1.0cm} J_\mu(\pparams)  \defeq \int_{\States} d_\mu(s) i(s) v_{\pi_\pparams}(s)  \intd s,
\end{equation}
where the weighting $i: \States \rightarrow [0,\infty)$ represents the interest in a state \citep{sutton2016anemphatic}. This objective reduces to the episodic (start-state) formulation by setting $i(s) = d_0(s)/ d_\mu(s)$. This choice is not just hypothetical; we will show how our algorithm designed for the excursions objective can also be used for the start-state formulation by using an easy-to-compute setting of the interest $i(s)$ for each $s \in \States$. This generalization goes beyond unifying the previous two objectives, and also naturally allows us to reweight states based on their importance for the target policy. For example, when learning an option, the interest function could be set to 1 in states from which the policy will be executed (i.e., the initiation set), and zero elsewhere.

More generally, we could consider the generic objective $\sum_{s \in \States} d(s) v_{\pi_\pparams}(s)$ for any state weighting $d$. Then the weighted-excursions objective is simply an instance of this more generic objective, with $d = d_\mu(s) i(s)$. The reason that we opt for the weighted-excursions objective is because it is sufficiently general to incorporate our two settings of interest but sufficiently restricted to facilitate development of a practical off-policy algorithm. In fact, the idea of interest comes from the same work on emphatic weightings \citep{sutton2016anemphatic} on which we build our algorithm. This work also provides insight into how to easily incorporate such interest functions in the off-policy setting.  

\textbf{Remark:} In this work, we focus on episodic objectives---those with termination---but it is worthwhile to contrast this to the objective used for continuing problems.
For continuing problems, there is a compelling case for the average reward objective
\begin{equation}\label{alt_obj}
J_{\text{ave}}(\pparams) \defeq \sum_{s \in \States} d_{\pi_\pparams}(s) r_{\pi_\pparams}(s) \hspace{1.0cm} \text{ or } \hspace{1.0cm} J_{\text{ave}}(\pparams) \defeq \int_{\States} d_{\pi_\pparams}(s) r_{\pi_\pparams}(s) \intd s
\end{equation}
for 
$r_\pi(s) = \mathbb{E}_\pi[r(s,A,S')]$. 
Optimizing $J_{\text{ave}}$ is equivalent to optimizing $\sum_{s \in \States} d_\pi(s) \vpi(s)$ for a constant $0 \le \gamma < 1$ \citep{sutton2018reinforcement}. Intuitively, the agent seeks to maximize state values while shifting its state distribution towards high-valued states.

A sharp distinction has been previously made between \textit{alternative life} objectives and excursions objectives in policy evaluation \citep{patterson2022ageneralized}. In that work, alternative life objectives use the state-weighting $d_\pi$ that specifies: learn the best value function as if data had been gathered under the target policy $\pi$. Such a weighting is obtained by using prior corrections---namely reweighting entire trajectories with products of importance sampling ratios \citep{precup2000eligibility,meuleau2000off}---or estimating $d_\pi/d_\mu$ \citep{hallak2017consistent,liu2018breaking,gelada2019off}. The excursions objective, on the other hand, uses a weighting of $d_\mu$. This distinction is not clearly relevant to control. Instead, typically \emph{on-policy} average reward algorithms are developed for the continuing objective---which has weighting $d_\pi$. The few off-policy algorithms designed for the continuing, average reward setting directly estimate $d_\pi$ and $d_\mu$ \citep{liu2019off}. The episodic objective for control does not include $d_\pi$, and so is different from the alternative life objective in policy evaluation.  

In the remainder of this paper we develop an algorithm for optimizing the weighted-excursions objective in Equation (\ref{exc_obj}) from off-policy data.

\section{An Off-Policy Policy Gradient Theorem using Emphatic Weightings}\label{sec_oppgt}

In this section, we introduce the off-policy policy gradient theorem for the discrete and continuous settings as well as for deterministic policies. We additionally compare this gradient to the gradient underlying OffPAC, and provide a generalized gradient that interpolates between this gradient---called the semi-gradient---and the true gradient.

\subsection{The Off-Policy Policy Gradient Theorem} 
The policy gradient theorem with function approximation in the on-policy setting was a seminal result \citep[Theorem 1]{sutton2000policy}.
The first attempt to extend the policy gradient theorem to the off-policy excursion objective was limited to the setting where the policy is tabular \citep[Theorem 2]{degris2012offpolicy}.\footnote{Note that the statement in the paper is stronger, but in an errata published by the authors, they highlight an error in the proof. Consequently, the result is only correct for the tabular setting.} 
In this section, we show that the policy gradient theorem does hold
in the off-policy setting for the weighted excursions objective, for arbitrary smooth policy parametrizations. The resulting gradient resembles the standard policy gradient, but with emphatic weightings to reweight the states. 
\begin{definition}[Emphatic Weighting]
The emphatic weighting $m : \mathcal{S} \to [0,\infty)$ for a behaviour policy $\mu$ and target policy $\pi$ can be recursively defined as 
  \begin{equation}
    m(s') \defeq d_\mu(s')i(s') + \mathbb{E}[\gamma(S,A,S') m(S) \mid S' = s'],
  \end{equation}
 where the distribution over $S, A$ from next state $S'$ is under $\Pfcn$ and $\pi$.
\end{definition}
Under finite states and actions, this expectation is $\sum_{s\in \States} \sum_{a \in \Actions} \pi(a|s) \Pfcn(s'|s,a) \gamma(s,a,s') m(s)$.
Under continuous states and actions, it is $\int_{\States} \int_{\Actions} \pi(a|s) \Pfcn(s'|s,a) \gamma(s,a,s') m(s)  \intd a \intd s$. If a deterministic policy $\pi: \States \rightarrow \Actions$ is used, then it is $\int_{\States} \Pfcn(s'|s,\pi(s)) \gamma(s,\pi(s),s') m(s) \intd s$.

 Notice that these emphatic weightings involve states leading into $s'$; bootstrapping is back-in-time, rather than forward. The emphatic weighting reflects the relative importance of a state $s'$, based on its own interest and weighting---the first term $d_\mu(s')i(s')$---as well as the discounted importance of states that lead into $s'$. For value estimation, this is sensible because it reflects how much the agent bootstraps off of the values in $s'$, and hence how much it relies on accurate estimates in $s'$. For policy gradients, we have a related role: the importance of a state is due both to its own value, as well as how much other states that lead into it depend on the value obtained from that state.   
 
 \newcommand{\grads}[1]{g(#1, \pparams)}
 
\begin{theorem}[Off-Policy Policy Gradient Theorem] \label{grad_theorem}
For finite states and actions,
\begin{equation}
\frac{\partial J_\mu(\pparams) }{\partial \pparams} 
= \sum_{s \in \States} \emweight(s) \sum_a \frac{\partial \pi(a|s; \pparams) }{\partial \pparams}  \qpi(s,a) = \sum_{s \in \States} \emweight(s) \grads{s} \label{eq_true_grad}
\end{equation}
where 
\begin{equation}
\grads{s} \defeq \mathbb{E}_{\pi_\pparams}\left[\frac{\partial \log \pi(A|S; \pparams) }{\partial \pparams}  \qpi(S,A) \ \Big| \ S= s\right] \label{eq_grads}
\end{equation}
is the gradient for a given state, rewritten using the log-likelihood ratio.
For continuous states, and either continuous or discrete actions,
\begin{equation}
\frac{\partial J_\mu(\pparams) }{\partial \pparams} 
= \int_{\SS} \emweight(s) \grads{s} \intd s \tag{\ref{eq_true_grad} for continuous states}
.
\end{equation}
\end{theorem}
\begin{proof}
We first start with the finite state setting. The proof relies on the vector form of the emphatic weighting
\begin{equation*}
\emvec^\top \defeq \ivec^\top (\eye - \Ppig)^{-1},
\end{equation*}
where the vector $\ivec \in \mathbb{R}^{|\SS|}$ has entries $\ivec(s) \defeq d_\mu(s) \interest(s)$ and $\Ppig \in \mathbb{R}^{|\SS| \times |\SS|}$ is the matrix with entries $\Ppig(s,s') \defeq \sum_{a \in \AA} \pi(a|s;\pparams) \Pfcn(s'|s,a) \gamma(s,a,s')$.
First notice that 
\begin{equation*}
\frac{\partial J_\mu(\pparams) }{\partial \pparams} = \frac{\partial \sum_{s \in \States} \ivec(s) \vpi(s) }{\partial \pparams} = \sum_{s \in \States} \ivec(s) \frac{\partial  \vpi(s) }{\partial \pparams}.
\end{equation*}
Therefore, to compute the gradient of $J_\mu$, we need to compute the gradient of the value function with respect to the policy parameters. 
A recursive form of the gradient of the value function can be derived, as we show below.  
Before starting, for simplicity of notation we will use 
\begin{align*}
 \gvec(s) = \sum_a \frac{\partial \pi(a|s; \pparams) }{\partial \pparams}  \qpi(s,a),
\end{align*}
where $\gvec: \States \rightarrow \RR^\pdim$. 
Now let us compute the gradient of the value function:
\begin{align}
\frac{\partial  \vpi(s) }{\partial \pparams} &= \frac{\partial }{\partial \pparams} \sum_a  \pi(a|s; \pparams)  \qpi(s,a) \nonumber \\
&=  \sum_a \frac{\partial \pi(a|s; \pparams) }{\partial \pparams}  \qpi(s,a) + \sum_a \pi(a|s; \pparams)  \frac{\partial \qpi(s,a) }{\partial \pparams} \label{eq_product_rule} \\
&=  \gvec(s) + \sum_a \pi(a|s; \pparams)  \frac{\partial \sum_{s'} \Pfcn(s'|s,a) (r(s,a,s') + \gamma(s,a,s') \vpi(s'))}{\partial \pparams} \nonumber \\
&=  \gvec(s) + \sum_a \pi(a|s; \pparams) \sum_{s'} \Pfcn(s'|s,a) \gamma(s,a,s') \frac{\partial  \vpi(s')}{\partial \pparams} \nonumber.
\end{align}
We can simplify this more easily using vector form. Let $\vpivecdot \in \RR^{|\States| \times \pdim}$ be the matrix of gradients (with respect to the policy parameters $\pparams$) of $\vpi$ for each state $s$, and $\Gmat \in \RR^{|\States| \times \pdim}$ the matrix where each row corresponding to state $s$ is the vector $\gvec(s)$. Then
\begin{align}
\vpivecdot  &=  \Gmat + \Ppig \vpivecdot \label{eq_bellman_gradient}
\ \ \ \ \implies \vpivecdot = (\eye - \Ppig)^\inv \Gmat .
\end{align}
Therefore, we obtain
\begin{align*}
\sum_{s \in \States} \ivec(s) \frac{\partial  \vpi(s) }{\partial \pparams} 
&= \ivec^\top \vpivecdot  \ \ =   \ivec^\top (\eye - \Ppig)^\inv \Gmat\\
&= \emvec^\top \Gmat\\
&= \sum_{s \in \States} \emweight(s) \sum_a \frac{\partial \pi(a|s; \pparams) }{\partial \pparams}  \qpi(s,a).
\end{align*}
The proof is similar for the continuous case; we include it in Appendix \ref{app_thm_pg_complete}.
\end{proof}
%
We additionally prove the deterministic policy gradient objective, for a deterministic policy, $\pi: \States \rightarrow\Actions$. The objective remains the same, but the space of possible policies is constrained, resulting in a slightly different gradient.
\begin{theorem}[Deterministic Off-policy Policy Gradient Theorem] \label{grad_theorem_deterministic}
  \begin{equation}
    \frac{\partial J_\mu(\vec\theta)}{\partial \vec\theta} = \int_{\States} m(s) \frac{\partial \pi(s;\vec\theta)}{\partial\vec\theta} \left. \frac{\partial q_{\pi}(s,a)}{\partial a}\right\rvert_{a=\pi(s;\vec\theta)} \intd s
  \end{equation}
  where 
  \begin{equation*}
    m(s') = d_\mu(s')i(s') + \int_{\States} \textup{\Pfcn}(s'|s,\pi(s;\vec\theta)) \gamma(s,\pi(s;\vec\theta),s') m(s) \intd s
    .
  \end{equation*}
\end{theorem}
The proof is presented in Appendix \ref{deterministic_pgt_appendix}.

\subsection{Connection to OffPAC and the Semi-gradient}

The off-policy policy gradient above contrasts with the one used by OffPAC
\begin{equation}
\sum_{s \in \States} d_\mu(s) \mathbb{E}_{\pi_\pparams}\left[\frac{\partial \log \pi(a|s; \pparams) }{\partial \pparams}  \qpi(S,A) \ \Big| \ S= s\right] \label{eq_semi_grad}
.
\end{equation}
The only difference is in the state weighting $d_\mu$. This small difference, however, is key for getting the correct gradient. In fact, the OffPAC gradient can be considered a \emph{semi-gradient}, because it drops a part of the gradient when using the product rule:
\begin{align*}
\nabla_\pparams J_{\text{exc}}(\pparams) 
&= \sum_{s \in \States} d_\mu(s) \nabla_\pparams \sum_a \pi_\pparams(a | s) \qpi(s,a) \\
&= \sum_{s \in \States} d_\mu(s) \sum_a \left[\qpi(s,a) \nabla_\pparams \pi_\pparams(a | s) + \pi_\pparams(a | s)  \nabla_\pparams \qpi(s,a) \right].
\end{align*}
The second term with $\nabla_\pparams \qpi(s,a)$ is not obviously easy to compute, because it requires reasoning about how changes to the policy parameters affect the action-values. OffPAC opted to drop this term, and justify why this approximation was acceptable. The resulting semi-gradient has proven to be surprisingly effective, though we provide several examples in this work where it performs poorly. 

Note that using this correct weighting is only critical due to state aliasing. If the policy representation is tabular, for example, the weighting on the gradients does not affect the asymptotic solution as long as it is non-zero.\footnote{Note that the weighting on the gradients could still affect the rate of convergence \citep{laroche2021dr}.} The gradient $\sum_{s \in \States} \emweight(s) g(\pparams, s) = 0$ if and only if $g(\pparams,s) = 0$, because the gradient vector has an independent entry for each state. Therefore, even if the state weighting in the gradient is changed to some other weighting $d$, then we still get $\sum_{s \in \States} d(s) g(\pparams, s) = 0$. This is the reason that the semi-gradient underlying OffPAC converges to a stationary point of this objective in the tabular setting. More generally, for any sufficiently powerful function approximator that can obtain $g(\pparams,s) = 0$ for all states, this result holds. In optimization, this property has been termed the \emph{strong growth condition} \citep{schmidt2013fast}. We state this simple result for the irrelevance of the weighting formally in the following proposition, both to highlight it and because to the best of our knowledge it has not been formally stated.
\begin{proposition}[Irrelevance of Weighting under the Strong Growth Condition] \label{thm_irrelevant}\ \\
If $g(\pparams,s) = 0$ for all states, then $\pparams$ is a stationary point under any state weighting $d: \States \rightarrow \mathbb{R}$.
\end{proposition}
This result states that with a function approximator that can perfectly maximize value in each state, the choice of state weighting in the gradient computation is not relevant. Both the off-policy gradient with emphatic weighting, and the semi-gradient with weighting $d_\mu$, can converge to this stationary point. 
 
Alternatively, we can consider the more generic emphatic weighting that lets us sweep between the gradient and the semi-gradient. The emphatic weightings were introduced for temporal difference (TD) learning with eligibility traces \citep{mahmood2015emphatic}. The weightings depend on the eligibility trace parameter, $\lambda$. For $\lambda = 1$, the weightings effectively become 1 everywhere---implicitly resulting in state weightings of $d_\mu$---because there is no bootstrapping. For $\lambda = 0$, the weightings are exactly as defined above. 

\begin{definition}[Generalized Emphatic Weighting]
For $\lambdaa \in [0,1]$, the generalized emphatic weighting $m_\lambdaa : \mathcal{S} \to [0,\infty)$ for a behaviour policy $\mu$ and target policy $\pi$ is defined as 
  \begin{equation}
    m_\lambdaa(s') \defeq (1-\lambdaa) d_\mu(s')i(s') + \lambdaa m(s')
  \end{equation}
For $\lambdaa = 1$, $m_\lambdaa(s') = m(s')$. 
\end{definition}
Note that in the original work, $m(s')$ was called the follow-on weighting, and $m_\lambdaa(s')$ the emphatic weighting. Because $m_\lambdaa$ is a strict generalization of $m$, we simply call $m$ an emphatic weighting as well.\footnote{Note that the original emphatic weightings \citep{sutton2016anemphatic} use $\lambda = 1-\lambdaa$. This is because their emphatic weightings are designed to balance bias introduced from using $\lambda$ for estimating value functions: larger $\lambda$ means the emphatic weighting plays less of a role. For this setting, we want larger $\lambdaa$ to correspond to the full emphatic weighting (the unbiased emphatic weighting), and smaller $\lambdaa$ to correspond to a more biased estimate, to better match the typical meaning of such trace parameters.}  

The parameter $\lambdaa$ can be interpreted as trading off bias and variance in the emphatic weightings. Notice that for $\lambdaa = 0$, and assuming an interest of 1 everywhere, we get that $m_\lambdaa(s) = d_\mu(s)$. Such a choice would give exactly the semi-gradient weighting. As we increase $\lambdaa$, we interpolate between the semi-gradient and the full gradient. For $\lambdaa = 1$, we get $m_\lambdaa(s) = m(s)$ and so we have the full gradient. 
For $\lambdaa = 0$, therefore, we obtain a biased gradient estimate, but the emphatic weightings themselves are easy to estimate---they are myopic estimates of interest---which could significantly reduce variance when estimating the gradient. 
Selecting $\lambdaa$ between 0 and 1 could provide a reasonable balance, obtaining a nearly unbiased gradient to enable convergence to a valid stationary point but potentially reducing some variability when estimating the emphatic weighting. 
We will develop our algorithm for the general emphatic weighting, as this will give us more flexibility in the weighting, and will also allow us to incorporate OffPAC as one extreme (biased) version of the approach. 

\subsection{Summary}

In this section we proved the off-policy policy gradient theorem for the weighted-excursions objective, for both stochastic (Theorem \ref{grad_theorem}) and deterministic (Theorem \ref{grad_theorem_deterministic}) policies. These gradients are similar to the standard policy gradient, but with states weighted by emphatic weightings. We contrasted this to OffPAC, the (unsound) algorithm on which many off-policy policy gradient algorithms are built, highlighting that the only distinction is in this state weighting. We argued that this state weighting does not always affect the solution (Proposition 4), potentially partially explaining why OffPAC often performs well---especially in deep RL with large neural networks. Finally, we show how we can use generalized emphatic weightings, with a parameter $\eta$ (Definition 5), that allows us to interpolate between the sound, but harder-to-estimate emphatic weighting (at $\eta = 1$) and the unsound, but easy-to-obtain weighting used in OffPAC (at $\eta = 0$).  

\section{Actor-Critic with Emphatic Weightings}\label{sec_eac}

In this section, we develop an incremental actor-critic algorithm with emphatic weightings that uses the above off-policy policy gradient theorem.
To perform a gradient ascent update on the policy parameters, the goal is to obtain a sample of the gradient in Equation \ref{eq_true_grad}. As discussed above, the only difference to the semi-gradient used by OffPAC, in Equation \ref{eq_semi_grad}, is in the weighting of the states: the true gradient uses $m(s)$ and the semi-gradient uses $d_\mu(s)$. 
Therefore, we can use standard solutions developed for other actor-critic algorithms to obtain a sample of $\grads{s}$. 
The key difficulty is in estimating $\emweight(s)$ to reweight this gradient.

\subsection{Sampling the Gradient from a State}

Before addressing this key difficulty, we provide a brief reminder on how to obtain a sample of $\grads{s}$, with more details in Appendix \ref{app_sampling}. 
The simplest approach to compute the gradient for a given state is to use what is sometimes called the \emph{all-actions} gradient
$\sum_a \frac{\partial \pi(a|s; \pparams) }{\partial \pparams}  \qpi(s,a)$.  To avoid summing over all actions, we instead obtain an unbiased sample of the gradient under the action taken by the behaviour policy $A \sim \mu(s, \cdot)$: $\rho(s,A; \pparams) \frac{\partial \log \pi(a|s; \pparams) }{\partial \pparams}  \qpi(s,A)$ where the importance sampling ratio $\rho(s,a; \pparams) \defeq \frac{\pi(a|s;\pparams)}{\mu(a|s)}$ corrects the distribution over actions.  
The estimated gradient has more variance when we only use a sampled action rather than all actions. The standard strategy to reduce this variance is to subtract a
baseline, such as an estimate of the value function $v(s)$. The resulting update is
\begin{align*}
\rho(s,a; \pparams) \frac{\partial \log \pi(a|s; \pparams) }{\partial \pparams}  [\qpi(s,a) - v(s)]
 .
\end{align*}
 
In practice, it is difficult to obtain $\qpi(s,a)$ to compute this gradient exactly. Instead, we obtain (a likely biased) sample of the advantage, $\delta_t \approx \qpi(S_t,A_t) - v(s_t)$. An unbiased but high variance estimate uses a sample of the return $G_t$ from $(s_t,a_t)$, and uses $\delta_t = G_t - v(s_t)$. Then $\mathbb{E}[\delta_t | S_t = s, A_t = a] = q_\pi(s,a) - v(s)$. On the other extreme, we can directly approximate the action-values, $\hat q(s,a)$. In-between, we can use $n$-step returns, or $\lambda$-returns, to balance between using observed rewards and value estimates. In this work, we opt for the simplest choice: $n=1$ with $\delta_t = R_{t+1} + \gamma_{t+1} v(S_{t+1}) - v(S_t)$. 
 
 Finally, we need to estimate the values themselves. Any value approximation algorithm can be used to estimate $v$, such as TD, gradient TD or even emphatic TD. Given we already compute the emphasis weighting, it is straightforward to use emphatic TD. We investigate both a gradient TD critic and an emphatic TD critic in the experiments. 
  
\subsection{Estimating the Gradient with Emphatic Weightings}\label{sec_estimating}


The previous section outlined how to sample the gradient for a given state. We now need to ensure that these updates across states are weighted by the emphatic weighting. Intuitively, if we could estimate $m_\lambdaa(s)$ and $d_\mu(s)$ for all $s$, then we could simply premultiply the update with $m_\lambdaa(s)/d_\mu(s)$. However, estimating the state distribution is nontrivial, and actually unnecessary. In this section, we first explain how to obtain a Monte Carlo estimate of this pre-multiplier, and then how to directly estimate it with function approximation.


 
 \subsubsection{An Unbiased Estimate of the Gradient with Emphatic Weightings}
We can rely on the original work introducing the emphatic weightings to obtain a Monte Carlo estimate with the update
  \begin{align}\label{def_mt_and_ft}
M_t &= (1-\lambdaa) \interest_t + \lambdaa F_t && \ F_t = \gamma_t \rho_{t-1} F_{t-1} + \interest_t
\end{align} 
where $F_{0}=0$ and $\interest_t$ is the interest observed in $S_t$, with $i(s) = \mathbb{E}[\interest_t | S_t = s]$. 
Putting this together with the previous section, the gradient can be sampled by (1) sampling states according to $d_\mu$, (2) taking actions according to $\mu$, (3) obtaining an estimate $\delta_t$ of the advantage $q_\pi(s,a) - v(s)$, and then (4) weighting the update with $M_t$, to get 
\begin{equation*}
\pparams \gets \pparams + \stepsize \rho_t M_t \frac{\partial \log \pi(A_{t}|S_{t}; \pparams) }{\partial \pparams} \delta_{t}
.
\end{equation*}
Note that the all-actions gradient update with this emphatic weighting would be 
\begin{equation*}
\pparams \gets \pparams + \stepsize M_t \sum_{b \in \Actions} \pi(b|S_{t}; \pparams) \frac{\partial \log \pi(b|S_{t}; \pparams) }{\partial \pparams}  \qpi(S_{t},b),
\end{equation*}
 but as before we will assume that we use one sampled action. 
We prove that our update is an unbiased estimate of the gradient for a fixed policy in Proposition \ref{prop_unbiased}. We assume access to an unbiased estimate $\delta_t$ of the advantage for this result, though in practice  $\delta_t$ is likely to have some bias due to using approximate values. 
\begin{proposition}\label{prop_unbiased}
For a fixed policy $\pi$, with the conditions on the MDP from \citep{yu2015onconvergence} and assuming unbiased $\delta_t$, namely that $\mathbb{E}[\delta_t | S_t = s, A_t = a] = q_\pi(s,a) - v(s)$, 
\begin{equation*}
\E_{S_t \sim d_\mu,A_t \sim \mu}[\rho_t M_t \delta_t \nabla_\pparams  \log \pi(A_t|S_t; \pparams)] = \sum_{s \in \States} \emweight_\eta(s) \sum_a \frac{\partial \pi(a|s; \pparams) }{\partial \pparams}  \qpi(s,a)
.
\end{equation*}
\end{proposition}
\begin{proof}
%
Emphatic traces $M_t$ have been shown to provide an unbiased estimate of the true emphatic weighting for policy evaluation in Emphatic TD. We use the emphatic weighting differently, but 
can rely on the proof from \citep{sutton2016anemphatic} to ensure that (a) $d_\mu(s) \lim_{t \to \infty} \E_\mu[M_t | S_t = s] = \emweight_\eta(s)$. Note also that, given $S_t$, the next action does not impact the expectation: (b) $\E_\mu[M_t | S_t = s] = \E_\mu[M_t | S_t = s, A_t]$.
Using these equalities, we obtain
\begin{align*}
&\E_{S_t \sim d_\mu,A_t \sim \mu}[\rho_t M_t \delta_t \nabla_\pparams  \log \pi(A_t|S_t; \pparams)]\\
&= \sum_{s \in \States} d_\mu(s) \E_\mu[\rho_t M_t \delta_t \nabla_\pparams  \log \pi(A_t|S_t; \pparams) | S_t = s]\\
&= \sum_{s \in \States} d_\mu(s) \E_\mu \Big[\E_\mu[\rho_t M_t \delta_t \nabla_\pparams  \log \pi(A_t|S_t; \pparams) | S_t = s, A_t ]\Big] \hspace{.5cm} \triangleright \text{law of total expectation}\\
&= \sum_{s \in \States} d_\mu(s) \E_\mu \Big[\E_\mu[M_t | S_t = s, A_t ] \ \ \E_\mu[\rho_t \delta_t \nabla_\pparams  \log \pi(A_t|S_t; \pparams) | S_t = s, A_t]\Big]\\
&= \sum_{s \in \States} d_\mu(s) \E_\mu[M_t | S_t = s] \E_\mu \Big[\E_\mu[\rho_t \delta_t \nabla_\pparams  \log \pi(A_t|S_t; \pparams) | S_t = s, A_t]\Big] \hspace{.8cm} \triangleright \text{ using (b)}\\
&= \sum_{s \in \States} \emweight_\eta(s) \E_\mu[\rho_t \delta_t \nabla_\pparams  \log \pi(A_t|S_t; \pparams) | S_t = s]  \hspace{1.0cm} \triangleright \text{ using (a)}\\
&= \sum_{s \in \States} \emweight_\eta(s) \sum_a \frac{\partial \pi(a|s; \pparams) }{\partial \pparams}  \qpi(s,a) ,
\end{align*}
where the last line follows from the fact that
\begin{align*}
&\E_\mu[\rho_t \delta_t \nabla_\pparams  \log \pi(A_t|S_t; \pparams) | S_t = s] \\
&= \sum_a \mu(a|s) \rho(s,a) \nabla_\pparams  \log \pi(a|s; \pparams)\frac{\partial \pi(a|s; \pparams) }{\partial \pparams}  \E_\mu[\delta_t | S_t = s, A_t = a]\\
&= \sum_a \mu(a|s) \rho(s,a) \nabla_\pparams  \log \pi(a|s; \pparams)\frac{\partial \pi(a|s; \pparams) }{\partial \pparams}  [q_\pi(s,a) - v(s)] \hspace{1.0cm} \triangleright \text{ by assumption}\\
&= \sum_a \frac{\partial \pi(a|s; \pparams) }{\partial \pparams}  \qpi(s,a) .
\end{align*}
\par\vspace{-1cm}
\end{proof}

 \subsubsection{A Biased Estimate of the Gradient with Emphatic Weightings}\label{sec_direct_f}

For a fixed policy, the emphatic trace provides an unbiased way to reweight the gradient.
However, this is no longer true when the target policy is updated online in an actor-critic algorithm like ACE; the emphatic trace will contain importance sampling ratios for older versions of the policy.
If the target policy changes substantially during the learning process---as one would hope---the older importance sampling ratios could bias the emphatic trace's estimates.
On the other hand, a constant discount rate $\gamma < 1$ would give exponentially less weight to older importance sampling ratios in the emphatic trace, potentially mitigating the bias.
We empirically investigate the impact of this type of bias in section \ref
{sec:classic_control}.

As a Monte Carlo estimator of the emphatic weightings, the emphatic trace can also yield high-variance estimates.
Furthermore, the product of importance sampling ratios in the emphatic trace can lead to even higher-variance estimates, which can slow learning \citep{precup2001offpolicy,liu2020understanding,ghiassian2018online}. 
As such, several algorithms have been proposed that use parametrized functions to estimate the factors required for reweighting the updates \citep{hallak2017consistent,gelada2019off,liu2018breaking,zhang2019generalized}.
We can similarly approximate the emphatic weighting directly.

For our setting, this involves estimating $\E_\mu[M_t | S_t = s] = \emweight_\eta(s)/d_\mu(s)$. Because $M_t$ is a function of $F_t$, we directly approximate $\E_\mu[F_t | S_t = s]$ by learning a parametrized function $f_\fparams(s)$ and use it in place of $F_t$ to compute $M_t$. Notice that the resulting  $f_\fparams(s) \approx \E_\mu[F_t | S_t = s] = {\emweight(s)}/{d_\mu(s)}$, and so $(1-\lambdaa) i(s) + \lambdaa f_\fparams(s) \approx \E_\mu[M_t | S_t = s] = \emweight_\eta(s)/d_\mu(s)$. Because we sample $s \sim d_\mu$, weighting the update with an estimate of $\emweight_\eta(s)/d_\mu(s)$ is effectively an importance sampling ratio that ensures the updates are weighted according to $\emweight_\eta(s)$. 


The direct approximation relies on the recursive equation for the emphatic weighting that allows for a temporal-difference update.
Unlike the usual TD update, this update bootstraps off the estimate from the previous step rather than the next step, because the emphatic weightings accumulate interest back in time, leading into the state.
Recall the recursive equation
  \begin{equation*}
    \emweight(s') = d_\mu(s')i(s') + \sum_{\States, \Actions} \pi(a|s;\pparams) p(s'|s,a) \gamma(s,a,s') \emweight(s)
    .
  \end{equation*}
This equation is a Bellman equation, that specifies the emphatic values for state $s'$ in terms of the immediate interest and the states leading into $s$.  
We approximate $f_\fparams(s) \approx m(s)/d_\mu(s)$, rather than $m(s)$. The recursive formula for an $f(s)$ that equals $m(s)/d_\mu(s)$ is
  \begin{equation*}
    f(s') = i(s') + \tfrac{1}{d_\mu(s')}\sum_{\States, \Actions} \pi(a|s;\pparams) p(s'|s,a) \gamma(s,a,s') d_\mu(s) f(s'),
  \end{equation*}  
  where $i_t + \rho_{t-1}\gamma_t f_\fparams(S_{t-1})$ is a sample of this target,
  as we show in the next proposition. 
%
\begin{proposition} \label{f_update}
Under the conditions stated by \cite{sutton2016anemphatic},
\begin{align*}
\lim_{t\rightarrow\infty} \E_\mu[i_t + \rho_{t-1}\gamma_t f_\fparams(S_{t-1}) | S_t = s'] &= i(s') + \tfrac{1}{d_\mu(s')}\sum_{s,a} \pi(a|s;\pparams) p(s'|s,a) \gamma(s,a,s') d_\mu(s) f_\fparams(s).
\end{align*}
\end{proposition}  
\begin{proof}
The proof follows the strategy showing the unbiasedness of $F_t$ in \citet{sutton2016anemphatic}.
\begin{align*}
&\lim_{t\rightarrow\infty} \E_\mu[i_t + \rho_{t-1}\gamma_t f_\fparams(S_{t-1}) | S_t = s']\\
&\quad=  i(s') +  \lim_{t\rightarrow\infty} \E_\mu[\rho_{t-1}\gamma_t f_\fparams(S_{t-1}) | S_t = s'] \\
&\quad=  i(s') +  \sum_{s,a} \lim_{t\rightarrow\infty} \Pr\{S_{t-1} = s, A_{t-1} = a | S_t = s'\} \frac{\pi(a|s;\pparams)}{\mu(a|s)} \gamma(s,a,s') f_\fparams(s)\\
&\quad=  i(s') +  \sum_{s,a} \frac{d_\mu(s) \mu(a|s) p(s'|s,a)}{d_\mu(s')} \frac{\pi(a|s;\pparams)}{\mu(a|s)} \gamma(s,a,s')f_\fparams(s) \qquad \triangleright \text{ Bayes rule} \\
&\quad=  i(s') + \tfrac{1}{d_\mu(s')}\sum_{s,a} \pi(a|s;\pparams) p(s'|s,a) \gamma(s,a,s') d_\mu(s) f_\fparams(s)
\end{align*}
\par\vspace{-0.5cm}
\end{proof}
The utility of this result is that it provides a way to sample the target for this Bellman equation.
The following is a sample update whose target is equal to the above in expectation, when sampling $s \sim d_\mu$:
  \begin{equation} \label{m_est_update}
     \fparams \gets \fparams + \beta_t \bigl[i_t + \rho_{t-1} \gamma_{t} f_\fparams(S_{t-1}) - f_\fparams(S_t)\bigr] \nabla_\fparams f_\fparams(S_t)
    ,
  \end{equation}
  where $\beta_t$ is a step-size. The update aims to minimize the difference between the current estimate of the emphatic weightings and the target, using a standard semi-gradient TD update. We could alternatively use a gradient TD \citep{sutton2009fast,sutton2018reinforcement} update
    \begin{equation} \label{m_est_grad_update}
     \fparams \gets \fparams + \beta_t \left( \bigl[i_t + \rho_{t-1} \gamma_{t} f_\fparams(S_{t-1}) - f_\fparams(S_t)\bigr] \nabla_\fparams f_\fparams(S_t) - \gamma_t h(S_{t-1}) \nabla_\fparams f_\fparams(S_{t-1}) \right)
     ,
  \end{equation}
  where the auxiliary function $h(s)$ provides an estimate of $\mathbb{E}_\mu[i_t + \rho_{t-1} \gamma_{t} f_\fparams(S_{t-1}) - f_\fparams(S_t) | S_t = s]$ using a regression update. 

Once we use an approximation to the emphatic weighting, we may introduce bias into the gradient estimate. 
As an extreme case, imagine the function is restricted to predict the same value for all states. Using these estimates means replacing $M_t$ in the actor updates with a constant factor that can be subsumed in the step-size and therefore following the semi-gradient updates in Equation \ref{eq_semi_grad}. However, the advantage is that these lower-variance estimates do not have to be computed online, unlike $M_t$. The function $f_\fparams$ can be trained by sampling transitions from a replay buffer \citep{lin1993reinforcement,mnih2015human} and applying the update in Equation \ref{m_est_update}. 

\subsection{Summary}

In this section, we showed how to operationalize the off-policy policy gradient theorem, by discussing how to sample the gradient. The gradient involves a standard actor-critic update from a state $s$, but additionally weighted by the emphatic weighting $m_\eta(s)$. We leverage existing results for sampling the emphatic weighting in Section 5.2.1, and proved that using this unbiased emphatic weighting in our update results in an unbiased update (Proposition 6). These sampled emphatic weightings, however, can be high variance and become biased if the policy changes on each step. We develop a parameterized estimator in Section 5.2.2, by developing a recursive formula for $f_\fparams(s) \approx m(s)/d_\mu(s)$ in Proposition 7. Though a similar form was used to show unbiasedness of a sampled weighting in the original work on emphasis \citep{sutton2016anemphatic}, they did not recognize nor leverage this recursive form to develop a TD algorithm to learn a parameterized estimator. We need this $f_\fparams(s)$, instead of directly estimating $m(s)$, because states are sampled $s \sim d_\mu$, so weighting by $m(s)$ would be incorrect. 

We summarize the final algorithm in Algorithm \ref{ActorCritic1}, in Section \ref{sec_episodic_interest}. This pseudocode is provided later, after incorporating entropy regularization and showing how to use the algorithm for episodic problems.
The algorithm shows both how to use the emphatic trace and directly-estimated emphatic weightings.

\newcommand{\entropy}{\mathcal{H}}

\section{Incorporating Entropy Regularization}

In this section, we discuss how to extend the off-policy policy gradient theorem to incorporate entropy regularization. We focus on discrete states and actions in the main body, and include results for continuous states and actions in the appendix. Entropy regularization is commonly used with policy gradient methods, as it encodes a preference for more stochasticity in the policy. Recall that at each state, the policy induces a probability distribution over actions whose entropy is defined as
\begin{equation*}
\entropy\bigl(\pi(.|s)\bigr) \defeq - \sum_{a\in\Actions} \pi(a|s) \log \pi(a|s)
.
\end{equation*}
This value captures the stochasticity in the policy. A uniform random policy will maximize entropy, while the entropy of a nearly deterministic policy will be a large negative value. 

Entropy regularization is believed to promote exploration, improve the loss surface, and promote faster convergence rates. For exploration, entropy regularization can help make the policy more stochastic and diversify the set of states and actions that the agent learns about \citep{williams1991function,haarnoja2018soft}. To find a good policy, the agent needs accurate estimates of values of different actions in different parts of the state space. A greedy policy only takes actions that are deemed optimal at the current point which results in learning only about a limited number of trajectories. 
Entropy regularization has also shown to help policy optimization by modifying the landscape. The resulting objective is smoother, which allows the use of larger step-sizes and also reduces the chance of getting stuck in bad local maxima \citep{ahmed2019understanding}.

Finally, the use of entropy regularization facilitates convergence analysis to stationary points. \citet{mei2020global} showed that entropy regularization can ensure existence of stationary points and also improve the rate of convergence in tabular policy optimization. The idea is similar to convex optimization, where $\ell_2$ regularization makes the loss function strongly convex and improves the convergence rate from sublinear to linear. In Appendix \ref{suboptimality_appendix}, we extend the proof of existence of a stationary point for entropy regularized policy optimization to state aggregation. Therefore, in addition to extending our algorithm to allow for entropy regularization, we also use this extension to formally prove a counterexample for the semi-gradient updates.

Extending the results to entropy regularization is relatively straightforward, as it relies on the already developed machinery of soft action-values that include entropy in the rewards. 
Entropy regularization augments each reward with a sample of the entropy at the current state so that the return is modified to
\begin{align*}
  \tilde{G}_t &\defeq (R_{t+1} - \tau \log \pi(A_t|S_t))  + \gamma_{t+1}( R_{t+2} - \tau \log \pi(A_{t+1}|S_{t+1})) + \ldots
  \\
  &= R_{t+1} - \tau \log \pi(A_t|S_t) + \gamma_{t+1}\tilde{G}_{t+1},
\end{align*}
where $\tau$ is a parameter that controls the amount of regularization.
Entropy regularized state values and action values can be defined as \citep{geist2019theory}
\begin{align*}
  \vtpi(s) &\defeq \E_{\pi}[\tilde{G}_{t}| S_t = s] \quad \forall s \in \SS\\
  \qtpi(s,a) &\defeq \E_{\pi}[R_{t+1} + \gamma_{t+1}\tilde{G}_{t+1}| S_t = s, A_t = a] \quad \forall s \in \SS \text{ and } a \in \AA\\
  &= \sum_{s' \in \SS} \text{P}(s'|s,a)[r(s,a,s') + \gamma(s,a,s') \tilde{v}_{\pi}(s')] \quad \forall a\in\AA, \forall s\in\SS
  .
\end{align*}
Here, $\vtpi(s) = \E_{\pi}[\qtpi(S,A) | S = s] + \tau \entropy(\pi(\cdot | s))$ for entropy $\entropy(\pi(\cdot | s)) = \E_{\pi}[ -\log \pi(A|s)]$.  
The entropy-regularized objective function simply uses these entropy regularized values,
  \begin{equation} \label{ent_off_policy_objective}
    \tilde{J}_{\mu}(\boldsymbol{\theta}) \defeq \sum_{s\in\States}d_\mu(s) i(s) \tilde{v}_{\pi_{\boldsymbol{\theta}}}(s)
    .
  \end{equation}
  
The off-policy policy gradient theorem can be extended to the entropy regularized objective. The result looks a little different, because the relationship between the soft values and soft action-values is a little different than in the unregularized setting.  
Notice that $- \tau \log \pi(a|s)$ is not included in the first reward for $\qtpi(s,a)$, because this first action is given---not random---and entropy is an expectation over all actions. 
Further, $\log \pi(a|s)$ can be arbitrarily large, though the entropy itself remains nicely bounded. For this reason, we do not define the soft action-values to be $\E_{\pi}[\tilde{G}_{t}| S_t = s, A_t = a]$, as we typically would for the unregularized setting. Nonetheless, deriving the policy gradient theorem for these soft action-values uses exactly the same steps. Further, we do recover the unregularized formulation when $\tau=0$. 
\begin{theorem}[Off-policy Policy Gradient Theorem for Entropy Regularization] \label{ent_grad_theorem}
\begin{align*}
\frac{\partial \tilde{J}_\mu(\pparams) }{\partial \pparams} 
&= \sum_{s \in \States} \emweight(s) \left[\sum_a \frac{\partial \pi(a|s; \pparams) }{\partial \pparams}  \qtpi(s,a) + \tau \frac{\partial \entropy(\pi(\cdot |s; \pparams)) }{\partial \pparams} \right]\\
&= \sum_{s \in \States} \emweight(s) \sum_a \frac{\partial \pi(a|s; \pparams) }{\partial \pparams}  \left[ \qtpi(s,a) - \tau \log \pi(a |s; \pparams)) \right]
\end{align*}
\end{theorem}
The proof is in Appendix \ref{ent_grad_appendix}, as is the extension to continuous states, continuous actions, and other regularizers.

This theorem shows that we can use a similar update as in the unregularized setting, simply by adding $- \tau \log \pi(a |s; \pparams))$. Given the true soft action-values $\qtpi(s,a)$, an update with an unbiased sample of the gradient is
\begin{equation*}
\pparams \gets \pparams + \stepsize \rho_t M_t \frac{\partial \log \pi(A_{t}|S_{t}; \pparams) }{\partial \pparams}  \left[\qtpi(S_{t},A_{t}) - \tau \log \pi(A_{t} |S_{t}; \pparams))\right].
\end{equation*}
In practice, we subtract a baseline and use an estimate $\delta_t$ of the advantage $\qtpi(s,a) - \tau \log \pi(a |s; \pparams) - \vtpi(s)$, by only directly estimating the soft values $v(s) \approx \vtpi(s)$. The 1-step return sample is again the TD error, but with entropy regularization added to the rewards: $\delta_t = R_{t+1} - \tau \log \pi(A_t|S_t) + \gamma_{t+1}v(S_{t+1}) - v(S_t)$.

\section{Formulating Episodic Problems as a Special Case}\label{sec_episodic_interest}
The goal in an episodic problem is to maximize the expected episodic return. Mismatch between the algorithm's objective and this goal, even in the on-policy setting, can lead the agent to suboptimal policies \citep{thomas2014bias,nota2020isthe}. This section describes how we can use the interest function to maximize the right objective.

The objective function in Equation \ref{ep_obj} that weights state values by the start state distribution is equal to expected episodic return. With limited function approximation resources this weighting matters. If the parametrized function is incapable of representing a policy that maximizes all state values, the agent has to settle for lower values in some states to maximize the values of more important ones. For example, if the weighting is $d_\mu$, the agent may incorrectly prioritize states that appear less often at the start of an episode and more frequently at other points in the behaviour policy's trajectories, and fail to optimize the episode return.

The interest function in Equation \ref{exc_obj} allows us to focus the parametrized function's resources on states of interest. Recall that we assume that the agent observes interest $i_t$ in state $S_t$, with $\interest(s) = \E[\interest_t \mid S_t = s]$. So far, we have not used the additional flexibility that the interest itself can be random, but for this unification we will use this property. If we could set $\interest(S_t) = {d_0(S_t)}/{d_\mu(S_t)}$ then the objective function would correspond to the episodic objective. However, neither $d_\mu$ nor $d_0$ is available to the agent, and we would like to avoid estimating those distributions. 

Fortunately, we can show that if the signal is set to one at the beginning of an episode and set to zero thereafter, its expectation in the limit of time will be proportional to the correct ratio. Under transition-based discounting, the agent is informed that an episode has begun whenever it receives a discount factor of zero: namely, that the transition $(S_t, A_t, S_{t+1})$ resulted in termination. Therefore, we can obtain this result by allowing for the interest to be a function of the whole transition.


\begin{proposition}\label{prop_episodic}
If the signal $\interest_t$ is set to
\begin{equation}\label{eq_eps_interest}
    \interest_t \defeq
    \begin{cases}
      1, & \text{if}\ \gamma_{t}=0 \text{ (i.e. the beginning of an episode)} \\
      0, & \text{otherwise}
    \end{cases}
\end{equation}
then its expected value under the behaviour policy is
\begin{align*}
    \E_\mu[\interest_t\mid S_t = s] \propto \frac{d_0(s)}{d_\mu(s)}.
\end{align*}
\end{proposition}
\begin{proof}
\begin{align*}
\E_\mu[\interest_t\mid S_t = s] &= \Pr(\gamma_t = 0 \mid S_t = s)\\
&= \frac{\Pr(S_t = s\mid \gamma_t = 0 ) \Pr(\gamma_t = 0)}{\Pr(S_t = s)}\\
&= \frac{d_0(s) \Pr(\gamma_t = 0)}{d_\mu(s)}\\
&\propto \frac{d_0(s))}{d_\mu(s)},
\end{align*}
where we used the fact that $\Pr(S_t = s\mid \gamma_t = 0 ) = d_0(s)$ and $\Pr(S_t = s) = d_\mu(s)$. 
\end{proof}
The constant term is the same for all states, $\Pr(\gamma_t = 0)$. It simply reflects the probability of termination under the behaviour policy. This constant term does not change the relative ordering between policies, since all weight is still on the start states. Therefore, the resulting objective is equivalent to the episodic objective.  


We can relate the resulting update to the on-policy and off-policy updates used for the episodic setting. 
In the on-policy setting, if we use the interest function given by Equation \ref{eq_eps_interest} and a constant discount factor $\gamma$ during the episode, then the update reduces to the unbiased episodic actor-critic algorithm of \citet{sutton2018reinforcement}, originally proposed by \citet{thomas2014bias}.

\citet{sutton2000policy} proved the policy gradient theorem for the episodic setting
\begin{align*}
&\frac{\partial J_0(\pparams) }{\partial \pparams} = \sum_{s \in \States} \bar{d_\pi}(s) \sum_a \frac{\partial \pi(a|s; \pparams) }{\partial \pparams}  \qpi(s,a) \hspace{1.0cm} \text{for } \hspace{0.3cm} \bar{d_\pi}(s) = \sum_{t=0}^{\infty} \gamma^t \Pr(s;t,\pi)
,
\end{align*}
where $\Pr(s;t,\pi)$ is the probability of going from a start state to state $s$ in $t$ steps under $\pi$, the weighting in the gradient $\bar{d_\pi}$ is the discounted state visit distribution, and the agent should discount updates that occur later in the episode to account for this weighting.
An algorithm that does not discount later updates---and thus samples according to $d_\pi$---results in a biased update.

To fix this problem, \citet{thomas2014bias} proposed the unbiased actor-critic algorithm with the updates rules
\begin{align*}
  \pparams &\gets \pparams + \alpha I \delta \nabla_{\pparams} \log \pi(A|s;\pparams)
  \\
  I &\gets \gamma I.
\end{align*}
Since $I$ is initialized to $1$ and updated after the update to $\pparams$, $I$ will be $1$ during the first weight update, and will decay by $\gamma$ each timestep thereafter.

The on-policy ACE update rule with $\rho=1$ and $\lambdaa=1$ is
\begin{align*}
  F_{t} &\gets \gamma_{t} F_{t-1} + \interest_t \hspace{1.0cm} \text{for } F_0 = 0
  \\
  \pparams &\gets \pparams + \alpha F_{t} \delta \nabla_{\pparams} \log \pi(A|s;\pparams).
\end{align*}
Because $\interest_t=1$ on the first time step and $F$ is initialized to $0$, $F$ will be $1$ on the first weight update and will decay by $\gamma$ each time step thereafter.
From this inspection it is clear that the ACE update 
reduces to the unbiased actor-critic update in the on-policy episodic setting.
The final ACE algorithm with all the discussed techniques is in Algorithm \ref{ActorCritic1}.


 
 \newcommand{\cparams}{\mathbf{w}}

\begin{algorithm}[h]
\caption{Online Actor Critic with Emphatic weightings (ACE)}\label{ActorCritic1}
\begin{algorithmic}
\State {Initialize weights for actor} $\pparams$ and critic $\cparams$
\State {For emphatic trace, initialize } $F_0 = 0$; for directly-estimated 
emphatic weightings, initialize weights $\fparams$ for approximate emphatic weightings $f_\fparams$
\State {Suggested (default) settings of parameters:} $\lambdaa = 0.1$, $\tau=0.01$
\State {Obtain initial feature vector} $\xvec_0$ and set $i_0 \gets 1$
\Repeat  
\State Choose an action $a_t$ according to $\mu(\cdot|\xvec_t)$
  \State Observe reward $r_{t+1}$, next state vector $\xvec_{t+1}$ and $\gamma_{t+1}$
  \State \textbf{For episodic setting:} Set $i_{t+1} = 1$ if a new episode has begun; else $i_{t+1} = 0$
    \State \textbf{For excursions setting:} If no preferences between states, set $i_{t+1} = 1$ 

  \State $\tilde{r}_{t+1} \gets r_{t+1} - \tau \log \pi(a_t|\xvec_t;\pparams)$
  \State $\rho_t \gets \frac{\pi(a_t|\xvec_t;\pparams)}{\mu(a_t|\xvec_t)}$
  \State Update (entropy-regularized) critic $v_\cparams$ 
  \If{using emphatic trace}
\State $M_t \gets (1-\lambdaa) i_t + \lambdaa F_t$
 \State $F_{t+1} \gets \rho_{t} \gamma_{t+1} F_{t} + i_{t+1}$
\Else
  \State $M_t \gets (1-\lambdaa)i_t + \lambdaa f_\fparams(\xvec_t)$
    \State $\fparams \gets \fparams + \beta_t \bigl[i_{t+1} + \rho_t \gamma_{t+1} f_\fparams(\xvec_t) - f_\fparams(\xvec_{t+1})\bigr] \nabla_\fparams f_\fparams(\xvec_{t+1})$
\EndIf
\State $\delta_t \gets \tilde{r}_{t+1} + \gamma_{t+1} v_\cparams(s_{t+1}) - v_\cparams(s_t)$
  \State $\psivec \gets \nabla_\pparams \log \pi(b|\xvec; \pparams)$ 
  \State $\pparams \gets \pparams + \alpha_t \rho_t M_t \delta_t \psivec $
\Until{agent done interaction with environment}
\end{algorithmic}
\end{algorithm}

\section{Stationary Points under the Semi-gradient and Gradient Updates}

In this section, we provide more insight into the difference between the stationary points of the weighted excursions objective---obtained using the true gradient in Equation \ref{eq_true_grad}---versus those obtain with the semi-gradient update underlying OffPAC, in Equation \ref{eq_semi_grad}. We first provide a counterexample showing that all the stationary points for the semi-gradient have poor performance. In fact, initializing at the optimal solution, and then updating with the semi-gradient, moves the solution away to one of these highly suboptimal stationary points. We then discuss that the semi-gradient update can actually be seen as a full gradient update, albeit with a different implicit state weighting. This connection illuminates why the semi-gradient does so poorly on the counterexample, but also potentially sheds light on why OffPAC has generally performed reasonably in practice.   
  
\subsection{A Counter-example for the Semi-Gradient Update}\label{sec_counter}
  

 Recall that, in the derivation of the policy gradient theorem, the product rule breaks down the gradient of the value function into the two terms shown in Equation \ref{eq_product_rule}. The policy gradient theorem in OffPAC only considers the first term, i.e. the gradient of the policy given a fixed value function. The approximated gradient is
  \begin{align*} 
  \frac{\partial J_\mu(\pparams) }{\partial \pparams} \approx \sum_{s \in \States} \dmu(s) \sum_a \frac{\partial \pi(a|s; \pparams) }{\partial \pparams}  \qpi(s,a) .
  \end{align*}
%
The approximation above weights the states by the behaviour policy's state distribution instead of emphatic weightings. 

To see the difference between these weightings, consider the MDP in Figure \ref{fig:counterexample}. For the actor, $s_0$ has a feature vector $[1,0]$, and the feature vector for both $s_1$ and $s_2$ is $[0,1]$. This aliased representation forces the actor to take a similar action in $s_1$ and $s_2$. The behaviour policy takes actions $a_0$ and $a_1$ with probabilities $0.25$ and $0.75$ in all non-terminal states, so that $s_0$, $s_1$, and $s_2$ will have probabilities $0.5$, $0.125$, and $0.375$ under $\dmu$. The target policy is initialized to take $a_0$ and $a_1$ with probabilities $0.9$ and $0.1$ in all non-terminal states, which is close to optimal.

We trained two actors on this MDP, one with semi-gradient updates and one with gradient updates. The actors are initialized to the target policy above and the updates use exact values rather than critic estimates. States and actions are sampled from the behaviour policy. As shown in Figures \ref{fig:semi-grad-obj} and \ref{fig:semi-grad-a0s1}, while both methods start close to the highest attainable value of the objective function (that of a deterministic policy that takes $a_0$ everywhere), semi-gradient updates move the actor towards a suboptimal policy and reduce the objective function along the way. Gradient updates, however, increase the probability of taking $a_0$ in all states and increase the value of the objective function.

\begin{figure*}[ht!]
    \centering
    \begin{subfigure}[b]{0.25\textwidth}
	\includegraphics[width=\textwidth]{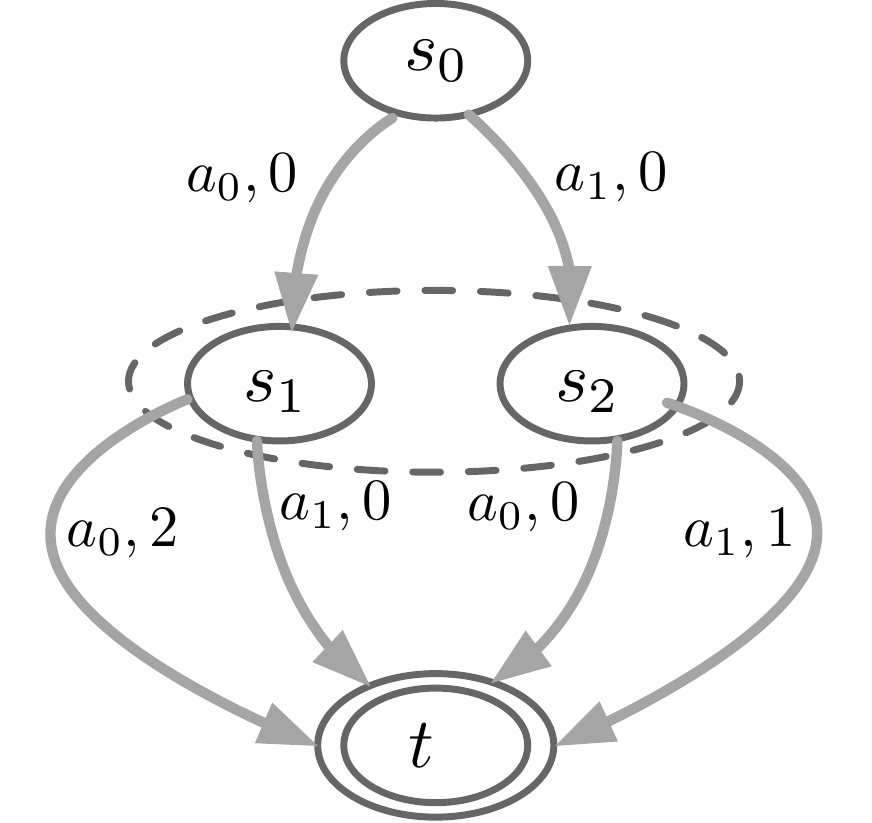}
	\caption{Counterexample}\label{fig:counterexample}
    \end{subfigure}   
    \begin{subfigure}[b]{\figwidthfour}
      \includegraphics[width=\textwidth]{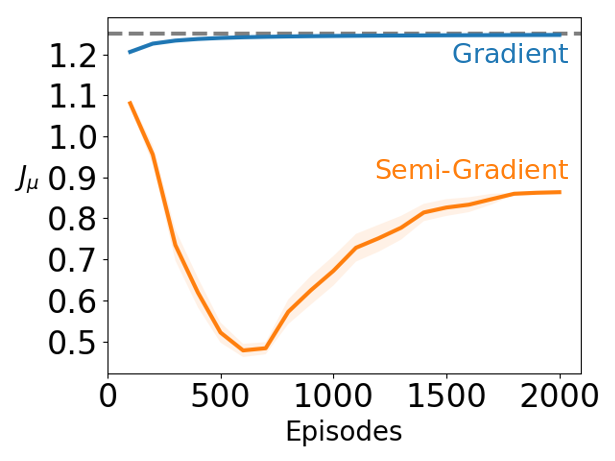}
      \caption{Learning curves}
      \label{fig:semi-grad-obj}
    \end{subfigure}   
    \begin{subfigure}[b]{\figwidthfour}
      \includegraphics[width=\textwidth]{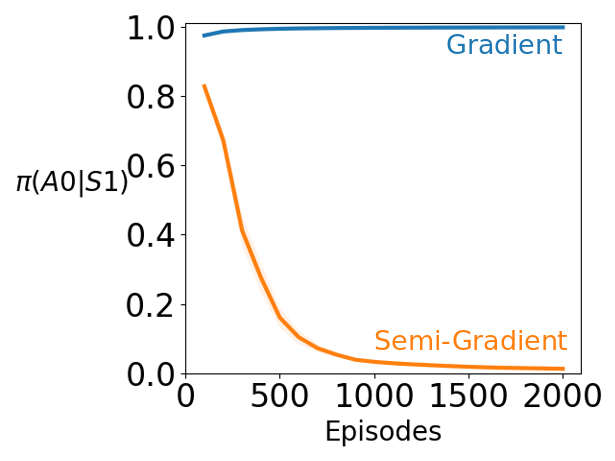}
      \caption{Optimal action probability}
      \label{fig:semi-grad-a0s1}
    \end{subfigure} 
    \caption{(\subref{fig:counterexample}) A counterexample that demonstrates the suboptimal behaviour of semi-gradient updates. The semi-gradients converge for the tabular setting \citep{degris2012offpolicy}, but not necessarily under function approximation---such as with the state aliasing in this MDP. The start state is denoted $s_0$ and the terminal state is denoted $t$. States $s_1$ and $s_2$ are aliased to the actor. The interest $i(s)$ is set to one for all states. (\subref{fig:semi-grad-obj}) Learning curves comparing semi-gradient updates and gradient updates, averaged over 30 runs with negligible standard error bars. The actor has a softmax output on a linear transformation of features and is trained with a step-size of 0.1 (though results were similar across all the stepsizes tested). The dashed line shows the highest attainable objective function under the aliased representation. (\subref{fig:semi-grad-a0s1}) The probability of taking the optimal action ($a_0$) in the aliased states.}
    \label{fig:semi-grad}
    \vspace{-0.3cm}
  \end{figure*}

The problem with semi-gradient updates boils down to the weighting. In an expected semi-gradient update, each state tries to increase the probability of the action with the highest action-value. There will be a conflict between the aliased states $s_1$ and $s_2$ because their highest-valued actions differ. If the states are weighted by $d_\mu$ in the expected update, $s_1$ will appear insignificant to the actor, and the update will increase the probability of $a_1$ in the aliased states. The ratio between $\qpi(s_1,a_0)$ and $\qpi(s_2,a_1)$ is not enough to counterbalance this weighting.

However, $s_1$ has an importance that the semi-gradient update overlooks. Taking a suboptimal action at $s_1$ will also reduce $q(s_0,a_0)$, and after many updates the actor will gradually prefer to take $a_1$ in $s_0$. Eventually, the target policy will be to take $a_1$ at all states, which has a lower value under the objective function than the initial target policy. This experiment highlights why the weight of a state should depend not only on its own share of $d_\mu$, but also on its predecessors. The behaviour policy's state distribution is not the proper deciding factor in the competition between $s_1$ and $s_2$, even when optimizing the excursions objective.
 
Proposition \ref{prop_suboptimality} formalizes the problem with semi-gradient updates by showing that, under any $\tau>0$, semi-gradient updates will not converge to a stationary point of the objective function in the counterexample. The requirement $\tau>0$ is only needed to ensure existence of a stationary point;  this result holds for $\tau$ arbitrarily close to zero. The proof is presented in Appendix \ref{suboptimality_appendix}.

\begin{proposition}\label{prop_suboptimality}
For any $\tau>0$ where $\tau \neq \tau_{i} \approx 0.2779$ in the three-state counterexample, semi-gradient updates do not converge to a stationary point of the objective function.
\end{proposition}

\subsection{The Semi-Gradient as a True Gradient with a Different Weighting}

In this section, we show that the semi-gradient can locally be seen as a gradient update, with a different state weighting in the objective. Locally for the current weights $\pparams_t$ with corresponding policy $\pi$, that state weighting is $\dvec = \dmu (\eye - \Ppig)$, with objective $J_d(\pparams) \defeq \sum_{s \in \States} d(s) v_{\pi_\pparams}(s)$. The resulting $\dvec$ may not be a distribution. In fact, it may not even be positive! If it is negative, the objective tells the agent to minimize the value for that state. This is precisely what occurs in the above counterexample. 

First, to see why this is the implicit state weighting, notice that for the semi-gradient update to be a gradient update, the weighting $d_\mu$ used in the update has to correspond to an emphatic weighting $m_d$ for some state weighting $d$ in the objective. In other words, $d_\mu(s) = m_d(s)$ where $\emvec_d = \dvec (\eye - \Ppig)^\inv$. This requirement implies that $\dvec = \emvec_d (\eye - \Ppig)$. Therefore, by using the weighting $d_\mu$ in the expected gradient, the semi-gradient update locally around $\pparams_t$ can be seen as a gradient update on $J_d$. This weighting $d$ actually changes as $\pparams_t$ changes. In fact, we know that the semi-gradient update cannot be seen as the gradient of a fixed objective function, from our counterexample for OffPAC and the result for the on-policy setting showing that omitting the discount factor results in an update that is not a gradient \citep{nota2020isthe}. However, even if this interpretation is only valid locally at each update, such a negative weighting can be problematic. 

We can compute the implicit weighting $d$, in our counter-example. Let $b = \mu(a_0 | s_0)$ and $p = \pi(a_0 | s_0)$. Then we know that $d_\mu(s_0) = 0.5$, $d_\mu(s_1) = 0.5 b = 0.125$ and $d_\mu(s_0) = 0.5 (1-b) = 0.375$. Further, we know that 
$m(s_0) = d_\mu(s_0)$ and 
\begin{align*}
m(s_1) &= d_\mu(s_1) + \pi(a_0 | s_0) m(s_0) = b d_\mu(s_0) + p m(s_0) =  d_\mu(s_0)(b + p) = 0.5 (b+p)\\
m(s_2) &= d_\mu(s_2) + (1-p) m(s_0) =  d_\mu(s_0)((1-b) + (1- p)) = 0.5 (2 - b - p).
\end{align*}
The implicit $d$ for the semi-gradient update, locally around this $\pi$, has $\tilde{m}(s) = d_\mu(s)$ where $\tilde{m}(s_0) = d(s_0)$ and so $d(s_0) = d_\mu(s_0) = 0.5$ and
\begin{align*}
\tilde{m}(s_1) &= d(s_1) + p \tilde{m}(s_0) = d(s_1) + p d_\mu(s_0) = d(s_1) + 0.5 p\\
\tilde{m}(s_2) &= d(s_2) + (1-p) \tilde{m}(s_0) =  d(s_2) + (1-p) d_\mu(s_0) =  d(s_2) + 0.5 (1-p) .
\end{align*}
Using $\tilde{m}(s_1) = d_\mu(s_1) = 0.5b$ and $\tilde{m}(s_2) = d_\mu(s_2) = 0.5(1-b)$, we get that
\begin{align*}
d(s_1) &= \tilde{m}(s_1)- 0.5 p = 0.5 (b-p)\\
d(s_2) &= \tilde{m}(s_2) - 0.5 (1-p) = 0.5 ((1-b) - (1-p)) = 0.5 (p-b).
\end{align*}
If $b > p$, then $d(s_2) < 0$; if $b < p$, then $d(s_1) < 0$; and if $b=p$ then both are zero. In our counterexample, we set $b < p$, making the update move away from increasing the value in $s_1$, namely preferring to increase the value in $s_2$ and decrease the value in $s_1$. The iterative updates maintain the condition $b < p$, even as the policy changes---which changes $p$---so the implicit weighting systematically causes convergence to a suboptimal policy. 

We note that the implicit weighting is independent of the representation. However, we know that the semi-gradient update converges to a stationary point of the excursions objective for the tabular setting. This might seem odd, given that the implicit weighting is negative for a state in this counterexample. However, this is not contradictory. Recall in Section \ref{grad_theorem} we discussed that in the tabular setting, the condition for the stationary point is that the gradient must be zero at every state, independently. Summing up zeros, even when weighting those zeros with a negative weighting, still results in zero.

There has been other work that has noted that optimizing under a different state weighting can still be effective. Proposition 6 of \citet{ghosh2020operator} shows that a lower bound can be optimized, using data generated under the behaviour policy. The action-values of the behaviour policy is used, and only the log likelihood terms of the target policy are differentiated. This lower bound, however, is only an approximation to the true objective. Our result differs, in that it highlights that equivalent solutions can be obtained, even under different state weightings. 

This view suggests a direction to reduce variance in the ACE updates: consider appropriate implicit weightings $d$, that allow for low variance emphatic updates. One potentially promising approach is to only consider emphatic weightings a small number of steps back-in-time. Another is to err on the side of smaller $\eta$ in the emphatic trace in ACE, knowing that the implicit weighting $d$ may remain reasonable for a broad range of $\eta$. 

\section{Experiments: Studying Properties of ACE}\label{sec:experiments}

In this section we investigate key properties of the ACE algorithm empirically.\footnote{Code is available at: \url{https://github.com/gravesec/actor-critic-with-emphatic-weightings}.}
Specifically, we study the effects of the trade-off parameter and the choice of estimator for the emphatic weightings on the learned policy, as these choices are central to the ACE algorithm.
First, we revisit the simple counterexample from Section \ref{sec_counter} to examine how the trade-off parameter can affect the learned policy.
Next, we illustrate some issues with using the emphatic trace in the ACE algorithm by testing it on a modified version of the counterexample.
We then move to two classic control environments to study the two estimators discussed in Section \ref{sec_estimating} and their effects on the learned policy.
Finally, we test several variants of ACE on a challenging environment designed to illustrate the issues associated with both estimators.
Please see Table \ref{tab:algorithm_descriptions} in Appendix \ref{app:algorithm_descriptions} for a description of each of the algorithms compared.

\subsection{The Impact of the Trade-Off Parameter}\label{lambda_a_experiment}

The parameter $\lambdaa$ can be interpreted as trading off bias and variance.
For $\lambdaa = 0$, the bias can be significant, as shown in the previous section.
A natural question to ask is how the bias changes as $\lambdaa$ ranges from 0 to 1.

To answer this question, we repeated the experiment in Section \ref{sec_counter}, but this time with $\lambdaa$ taking values in $\{0, 0.25, 0.5, 0.75, 1\}$ and the actor's step-size taking values in $\{0.01, 0.02, 0.05, 0.1, 0.2, 0.5, 1\}$, with the best-performing step-size (by total area under the learning curve, averaged over 30 runs) for each value of $\lambdaa$ used in Figures \ref{fig:param-obj} and \ref{fig:param-a0s1}.
To highlight the rate of learning, the actor's policy was initialized to take each action with equal probability.
The results are displayed in Figure \ref{fig:param} with shaded regions depicting standard error.

Figure \ref{fig:param-sens} shows the performance of ACE with each value of $\lambdaa$ over the range of step-sizes tested.
ACE performed well for a range of step-sizes, and even small values of $\lambdaa$ significantly improved over $\lambdaa = 0$ (OffPAC).
However, the performance of ACE with $\lambdaa > 0$ was more sensitive to the choice of step-size, with performance decreasing as the step-size grew large, while the performance of $\lambdaa=0$ was lower overall but remained steady with larger step-sizes.

Figure \ref{fig:param-obj} plots the learning curves for ACE with each value of $\lambdaa$.
For $\lambdaa = 0$ (OffPAC), the algorithm decreased the objective function during learning to get to a suboptimal fixed point, while $\lambdaa > 0$ always improved the objective function relative to the starting point.
For a surprisingly small choice of $\lambdaa = 0.5$, the actor converged to the optimal solution, and even $\lambdaa = 0.25$ produced a much more reasonable solution than $\lambdaa = 0$.

Figure \ref{fig:param-a0s1} shows the probability of taking the optimal action in the aliased states.
The optimal policy is to take $a_0$ in the aliased states with probability 1.
ACE with $\lambdaa=0$ (OffPAC) quickly converged to the suboptimal solution of choosing the best action for $s_2$ instead of $s_1$. Even with $\lambdaa$ just a bit higher than 0, convergence is to a more reasonable solution, choosing the optimal action the majority of the time.

\begin{figure*}[ht]
\centering
\begin{subfigure}[b]{\figwidthfour}
  \includegraphics[width=\textwidth]{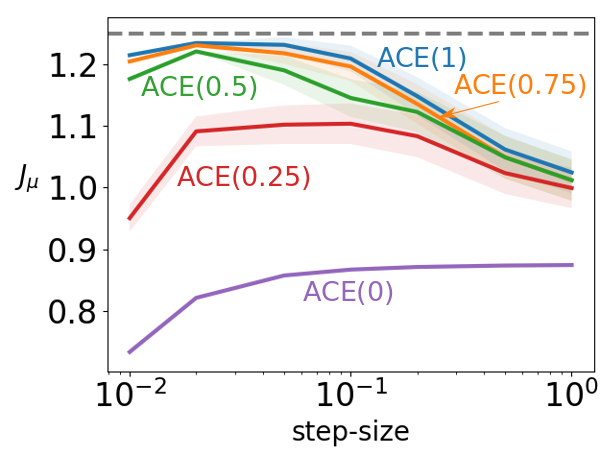}
  \caption{Actor step-size sensitivity}
  \label{fig:param-sens}
\end{subfigure}
\begin{subfigure}[b]{\figwidthfour}
  \includegraphics[width=\textwidth]{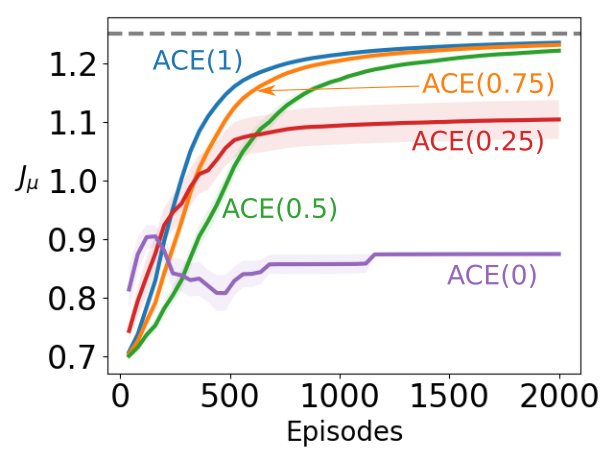}
  \caption{Learning curves }
  \label{fig:param-obj}
\end{subfigure}   
\begin{subfigure}[b]{\figwidthfour}
  \includegraphics[width=\textwidth]{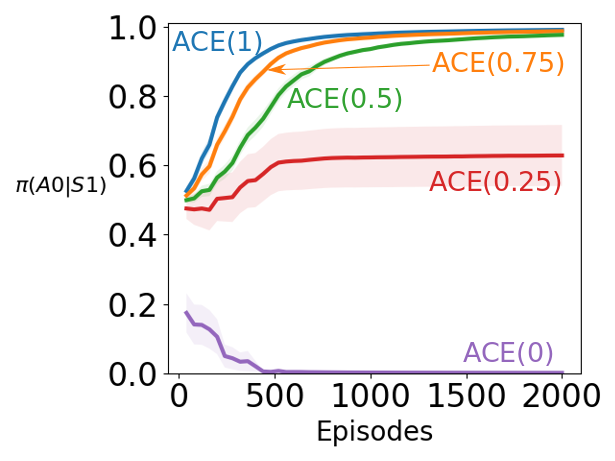}
  \caption{Optimal action probability}
  \label{fig:param-a0s1}
\end{subfigure} 
\caption{Performance of ACE with different values of $\lambdaa$ in the counterexample.}
\label{fig:param}
\end{figure*}

To determine whether the $\lambdaa$ parameter has similar effects when ACE is used with a learned critic, we repeated the previous experiment but this time using value estimates from a critic trained with GTD($\lambda$) \citep{maei2011gradient}.
We checked all combinations of the following critic step-sizes: $\{10^{-5}, 10^{-4}, 10^{-3}, 10^{-2}, 10^{-1}, 10^{0}\}$, the following critic trace decay rates: $\{0, 0.5, 1.0\}$, and the following actor step-sizes: $\{10^{-10}, 10^{-8}, 10^{-6}, 10^{-4}, 10^{-2}\}$, and used the best-performing combination (by area under the learning curve) for each value of $\lambdaa$ in Figures \ref{fig:gtd-param-obj} and \ref{fig:gtd-param-a0s1}.
We averaged the results over 10 runs, and plotted the standard error as shaded regions.
Figure \ref{fig:gtd-param} shows that, as before, even relatively small values of $\lambdaa$ can achieve close to the optimal solution.
However, $\lambdaa=0$ (OffPAC) still finds a suboptimal policy.
Overall, the outcomes are similar to the previous experiment, although noisier due to the use of a learned critic rather than the true value function.

\begin{figure*}[ht!]
   	\centering
   	\begin{subfigure}[b]{\figwidthfour}
   		\includegraphics[width=\textwidth]{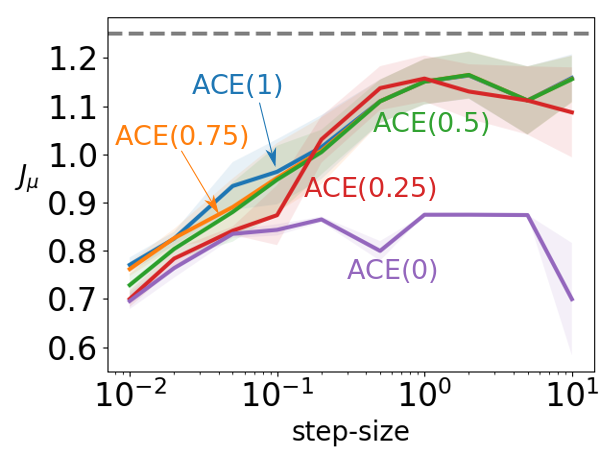}
   		\caption{Actor step-size sensitivity}
   		\label{fig:gtd-param-sens}
   	\end{subfigure}
   	\begin{subfigure}[b]{\figwidthfour}
   		\includegraphics[width=\textwidth]{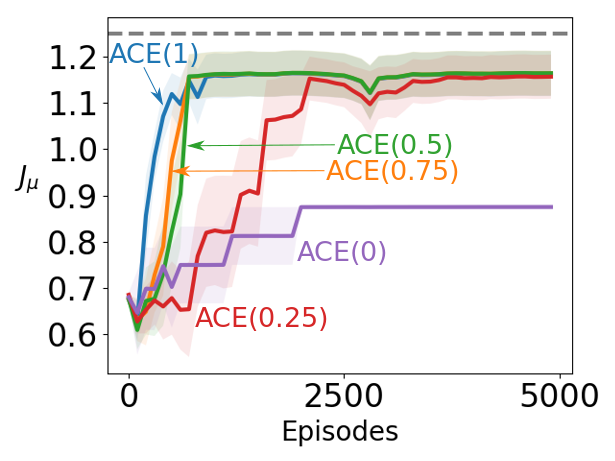}
   		\caption{Learning curves}
   		\label{fig:gtd-param-obj}
   	\end{subfigure}   
   	\begin{subfigure}[b]{\figwidthfour}
   		\includegraphics[width=\textwidth]{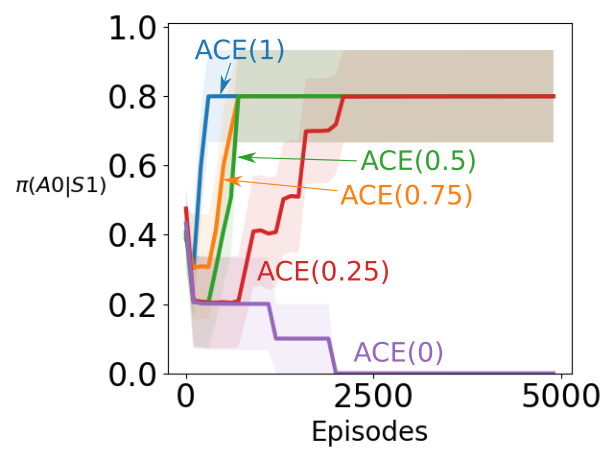}
   		\caption{Optimal action probability}
   		\label{fig:gtd-param-a0s1}
   	\end{subfigure} 
   	\caption{Performance of ACE with a GTD($\lambda$) critic and different values of $\lambdaa$ in the counterexample.}
   	\label{fig:gtd-param}
   	\vspace{-0.3cm}
\end{figure*}

So far we have only considered ACE with a discrete action policy parameterization.
However, an appealing property of actor-critic algorithms is their ability to naturally handle continuous action spaces.
To determine if the findings above generalize to continuous action policy parameterizations, we created an environment similar to Figure \ref{fig:counterexample}, but with one continuous unbounded action.
Taking action with value $a$ at $s_0$ will result in a transition to $s_1$ with probability $1-\sigma(a)$ and a transition to $s_2$ with probability $\sigma(a)$, where $\sigma$ denotes the logistic sigmoid function.
For all actions from $s_0$, the reward is zero.
From $s_1$ and $s_2$, the agent can only transition to the terminal state, with reward $2\sigma(-a)$ and $\sigma(a)$ respectively.
The behaviour policy takes actions drawn from a Gaussian distribution with mean $1.0$ and variance $1.0$. 

Because the environment has continuous actions, we can include both stochastic and deterministic policies, and so can include DPG in the comparison.
DPG is built on the semi-gradient, like OffPAC \citep{silver2014deterministic}.
We include True-DPG with Emphatic weightings (True-DPGE), which uses the true emphatic weightings rather than estimated ones to avoid the issue of estimating the emphatic weightings for a deterministic target policy, and focus the investigation on whether DPG converges to a suboptimal solution in this setting.
Estimation of the emphatic weightings for a deterministic target policy is left for future work.
The stochastic actor in ACE has a linear output unit and a softplus output unit to represent the mean and the standard deviation of a Gaussian distribution.
All actors are initialized with zero weights.

Figure \ref{fig:cont-param} summarizes the results.
The first observation is that DPG demonstrates suboptimal behaviour similar to OffPAC.
As training goes on, DPG prefers to take positive actions in all states, because $s_2$ is updated more often.
This problem goes away in True-DPGE.
The emphatic weightings emphasize updates in $s_1$ and, thus, the actor gradually prefers negative actions and surpasses DPG in performance. Similarly, True-ACE learns to take negative actions but, being a stochastic policy, it cannot achieve True-DPGE's performance on this domain. ACE with different $\lambdaa$ values, however, cannot outperform DPG, and this result suggests that an alternative to importance sampling ratios is needed to effectively extend ACE to continuous actions.

\begin{figure*}[ht!]
  \centering
  \begin{subfigure}[b]{\figwidthfour}
    \includegraphics[width=\textwidth]{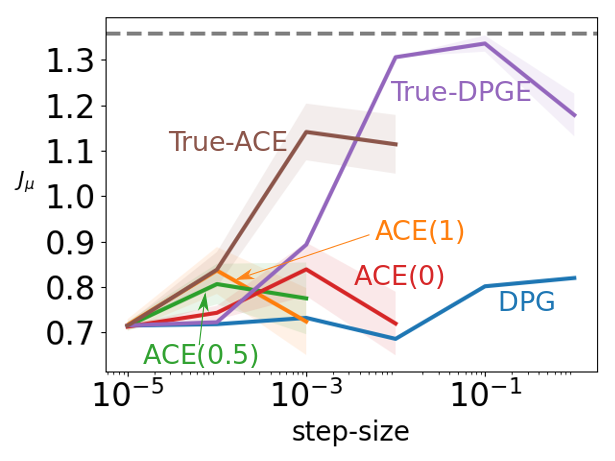}
    \caption{Stepsize sensitivity}
    \label{fig:cont-param-sens}
  \end{subfigure}
  \begin{subfigure}[b]{\figwidthfour}
    \includegraphics[width=\textwidth]{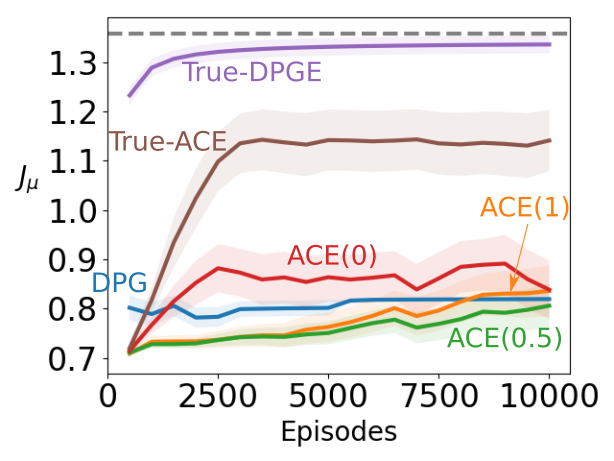}
    \caption{Learning curves}
    \label{fig:cont-param-obj}
  \end{subfigure}   
  \begin{subfigure}[b]{\figwidthfour}
    \includegraphics[width=\textwidth]{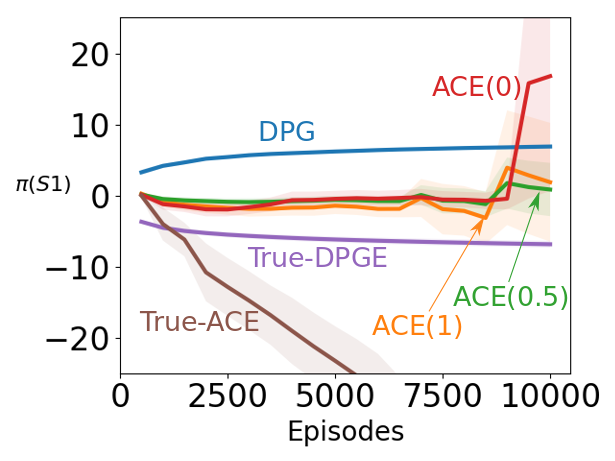}
    \caption{Mean action}
    \label{fig:cont-param-a0s1}
  \end{subfigure} 
  \caption{Performance of ACE with different values of $\lambdaa$, True-ACE, DPG, and True-DPGE on the continuous action MDP. The results are averaged over 30 runs. For continuous actions, the methods have even more difficulty getting to the optimal solutions, given by True-DPGE and True-ACE, though the action selection graphs suggest that ACE for higher $\lambdaa$ is staying nearer the optimal action selection than ACE(0) and DPG.}
  \label{fig:cont-param}
\end{figure*}

\subsection{Challenges in Estimating the Emphatic Weightings}\label{sec_challenges}

Up to this point we have been using the emphatic trace originally proposed by \citet{sutton2016anemphatic} to estimate the emphatic weightings.
As discussed in Section \ref{sec_eac}, there can be multiple sources of inaccuracy from using this Monte Carlo estimate in an actor-critic framework.

However, it is unclear how the issues affecting the emphatic trace manifest in practice, and hence whether introducing a parametrized function to directly estimate the emphatic weightings is worth the added bias---especially if the representation is poor.
To study the effects of this choice of estimator on the ACE algorithm, we conducted a series of experiments starting with a modified version of the counterexample in Figure \ref{fig:counterexample}, moving to two classic control environments, and ending with a challenging environment designed to illustrate the issues associated with both estimators.

\begin{figure*}[ht!]
  \centering
  \includegraphics[width=0.6 \textwidth]{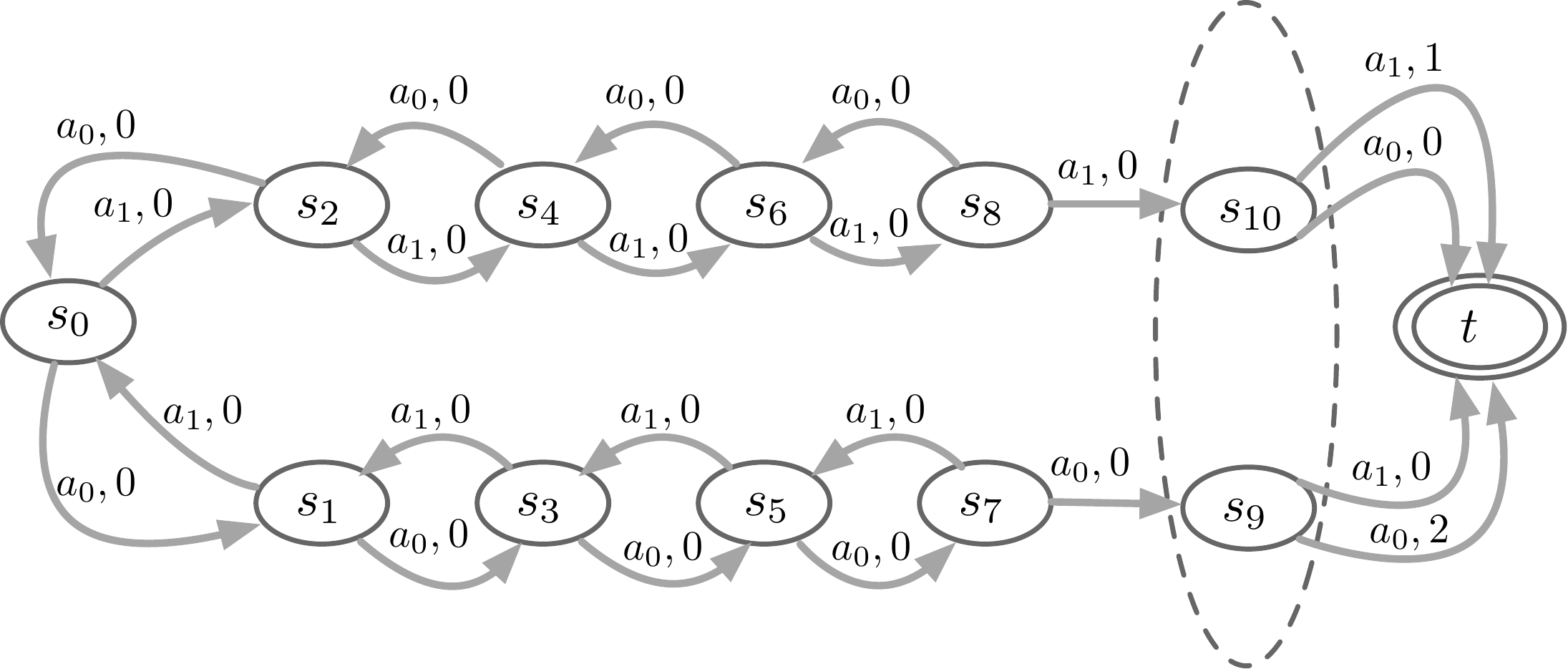}
  \caption{An 11-state MDP that makes estimating the emphatic weightings with the emphatic trace more difficult.}
  \label{fig:long-counterexample}
  \vspace{-0.3cm}
\end{figure*}

The first environment, shown in Figure \ref{fig:long-counterexample}, is an extended version of the counterexample with two long chains before the aliased states.
Like the original counterexample, the behaviour policy takes $a_0$ with probability $0.25$ and $a_1$ with probability $0.75$ in all non-terminal states, and the interest $i(s)$ is set to $1$ for all states.
Each state is represented by a unique feature vector except states $s_9$ and $s_{10}$, which are aliased and appear the same to the actor.
The addition of the new states makes trajectories considerably longer, which may exacerbate the issues with the emphatic trace.

We repeated the experiment from Figure \ref{fig:param} on the long counterexample.
The following actor step-sizes were tested: $\{5 \cdot 10^{-5}, 10^{-4}, 2 \cdot 10^{-4}, 5 \cdot 10^{-4}, 10^{-3}, 2 \cdot 10^{-3}, 5 \cdot 10^{-3}, 10^{-2}\}$,
with the best-performing value (by total area under the learning curve, averaged over 10 runs) for each value of $\lambdaa$ used in Figures \ref{fig:long-param-obj} and \ref{fig:long-param-a0s1}.
The actor was again initialized to take each action with equal probability, and the true state values were again used in the updates in order to isolate the effects of the emphatic trace.

We also trained an actor called True-ACE that uses the true emphatic weightings for the current target policy and behaviour policy, computed at each timestep.
The performance of True-ACE is included here for the sake of comparison, as computing the exact emphatic weightings is not generally possible in an unknown environment.

The results in Figure \ref{fig:long-param} show that, even though performance improves as $\lambdaa$ is increased, there is a significant gap between ACE with $\lambdaa=1$ and True-ACE.
Unlike Figure \ref{fig:param}, the methods have more difficulty reaching the optimal solution, although ACE with larger $\lambdaa$ does still find a significantly better solution than $\lambdaa = 0$.
Additionally, values of $\lambdaa$ greater than $0$ result in high-variance estimates of the emphatic weightings, which lead to more variable performance, as shown by the larger shaded regions representing standard error.
These results overall show the inaccuracies pointed out in Section \ref{sec_eac} indeed disturb the updates in long trajectories.

\begin{figure*}[ht!]
	\centering
	\begin{subfigure}[b]{\figwidthfour}
		\includegraphics[width=\textwidth]{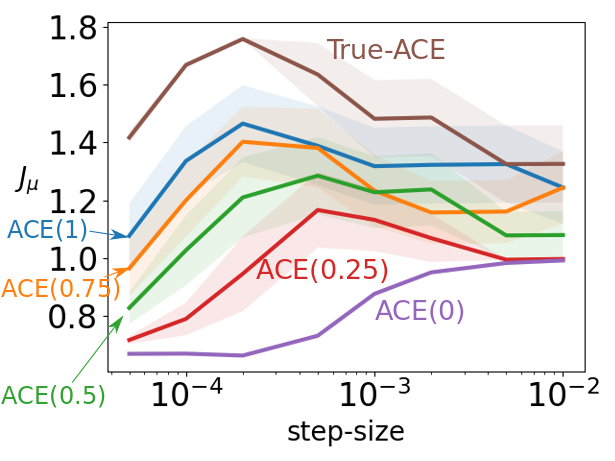}
		\caption{Actor step-size sensitivity}
		\label{fig:long-param-sens}
	\end{subfigure}
	\begin{subfigure}[b]{\figwidthfour}
		\includegraphics[width=\textwidth]{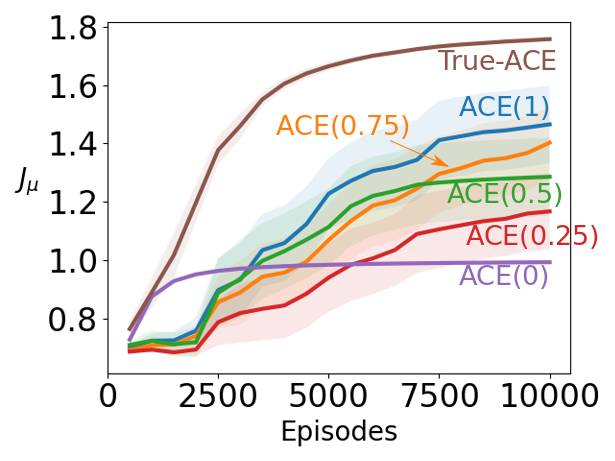}
		\caption{Learning curves }
		\label{fig:long-param-obj}
	\end{subfigure}   
	\begin{subfigure}[b]{\figwidthfour}
		\includegraphics[width=\textwidth]{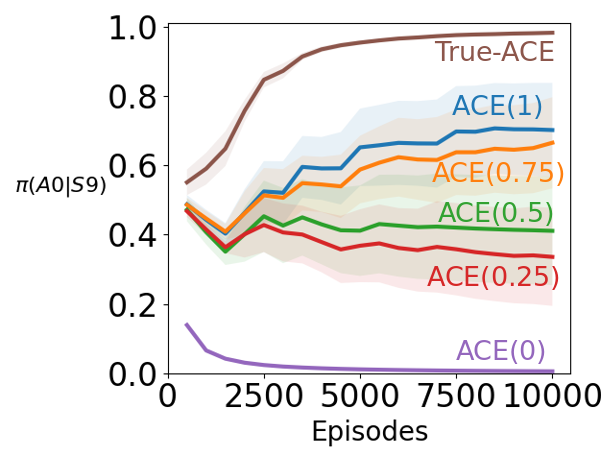}
		\caption{Optimal action probability}
		\label{fig:long-param-a0s1}
	\end{subfigure} 
	\caption{Performance of ACE with different values of $\lambdaa$ on the 11-state MDP.}
	\label{fig:long-param}
	\vspace{-0.3cm}
\end{figure*}

\subsection{Estimating the Emphatic Weightings in Classic Control Environments}\label{sec:classic_control}

The results of the previous experiments suggest that using the emphatic trace to estimate the emphatic weightings can significantly impact the performance of ACE.
To better understand how the issues with the emphatic trace affect the learning process beyond our counterexample, we tested several variants of ACE on off-policy versions of two classic control environments: Puddle World \citep{degris2012offpolicy} and Mountain Car \citep{moore1990efficient}.

As an off-policy Monte Carlo estimator of the emphatic weightings, the emphatic trace can yield extremely high-variance estimates which can interfere with learning \citep{ghiassian2018online}.
To see how the variance of the emphatic trace affects ACE, we tested a variant called ACE-direct that uses the one-step temporal-difference update from Section \ref{sec_direct_f} to estimate the emphatic weightings.
Temporal-difference methods often have lower variance than Monte Carlo methods---at the cost of introducing bias---which makes using a TD-style update to estimate the emphatic weightings an appealing alternative to the emphatic trace \citep{sutton2018reinforcement}.

As discussed in Section \ref{sec_direct_f}, the emphatic trace can also be biased when used in an actor-critic algorithm where the target policy is changing.
To determine how detrimental this bias is to the performance of ACE, we tested a variant called ACE-ideal where all importance sampling ratios in the emphatic trace are re-computed at each time step using the current target policy, yielding an unbiased Monte Carlo estimate of the emphatic weighting for the current state.
ACE-ideal is not a practical algorithm, as the computation and memory required grow with the number of time steps, but it allows us to isolate the effects of the bias introduced by using the emphatic trace with a changing target policy.

Other baselines for comparison included OffPAC, and ACE using the emphatic trace (referred to as ACE-trace).
To determine the effects of different choices of critic, we included versions of each algorithm using an ETD critic \citep{sutton2016anemphatic} and a TDRC critic \citep{ghiassian2020gradient}.
We also included versions of each algorithm using the uniform interest function (i.e., $i=1$ for all time steps), as well as the episodic interest function from Section \ref{sec_episodic_interest}---with the exception of OffPAC.
OffPAC scales policy updates using only the interest for the current time step, which when combined with the episodic interest function leads to a single policy update on the first time step followed by no updates thereafter.

To get a clear picture of the performance of each algorithm, we conducted a grid search on the free parameters of each algorithm.
For the step size parameter of the actor, critic, and the direct estimator of emphatic weightings, we tested values of the form $\frac{1}{2}^{i}$ where $i$ ranged from 0 to 15.
For the trace decay rate of the critic, we tested values of the form $1 - \frac{1}{2}^{j}$ where $j$ ranged from 0 to 6.
The discount factor was .95.
We sought to establish the performance of each variant for $\lambdaa=1$, as they all reduce to OffPAC when $\lambdaa=0$.

Each combination of parameters for each algorithm was run on 5 different trajectories generated by a fixed behaviour policy interacting with the environment for 100,000 time steps.
The learned policies were saved every 1,000 time steps and evaluated 50 times using both the episodic and excursions objective functions from Section \ref{sec:obj_off_policy}.
For the episodic objective function, the policies were evaluated by creating 50 different instances of the environment and executing the target policy from the starting state until termination or 1000 time steps had elapsed.
The excursions objective function was evaluated similarly, but with the environment's starting state drawn from the behaviour policy's steady state distribution, chosen by running the behaviour policy for 50,000 time steps and saving every thousandth state.
The results were averaged, and the best-performing combinations (by area under the learning curve for the appropriate objective function) were re-run on enough different trajectories to reduce the standard error to an acceptable level (100 runs for Puddle World, 30 runs for Mountain Car).
The initial parameter sweep was intended to find good parameter settings for each method at a reasonable computational cost, and not necessarily the absolute best performance.

\subsubsection{Puddle World}
The Puddle World environment is a 2-dimensional continuous gridworld containing a start location, goal location, and puddles through which it is costly to move.
While it first appeared in \citet{boyan1995generalization}, we use the version from \citet{degris2012offpolicy} that includes an additional puddle near the goal state which makes the task more difficult.\footnote{Please see \citet{degris2012offpolicy} for a picture of the Puddle World environment.}
The behaviour policy took the North, East, South, and West actions with probabilities .45, .45, .05, and .05 respectively.
The observations were tile coded with a fairly low-resolution tile coder (4 tilings of $2 \times 2$ tiles plus a bias unit) to generate feature vectors with a large degree of generalization.

Figure \ref{fig:pw} presents the results for the Puddle World environment.
The left-hand column shows learning curves, while the right-hand column contains sensitivity analyses for the actor's step size.
The first four plots (figures \ref{fig:pw_exc_etd}, \ref{fig:pw_exc_etd_sens}, \ref{fig:pw_epi_etd}, and \ref{fig:pw_epi_etd_sens}) show results when using ETD as the critic, while the last four plots (figures \ref{fig:pw_exc_tdrc}, \ref{fig:pw_exc_tdrc_sens}, \ref{fig:pw_epi_tdrc}, and \ref{fig:pw_epi_tdrc_sens}) show results for a TDRC critic.
The first and third rows show results using the excursions objective function, and the second and fourth rows show results for the episodic objective function.

\begin{figure*}[hp!]
  \vspace{-0.3cm}
  \centering
  \begin{subfigure}[b]{0.45\textwidth}
    \centering
    \includegraphics[width=\textwidth]{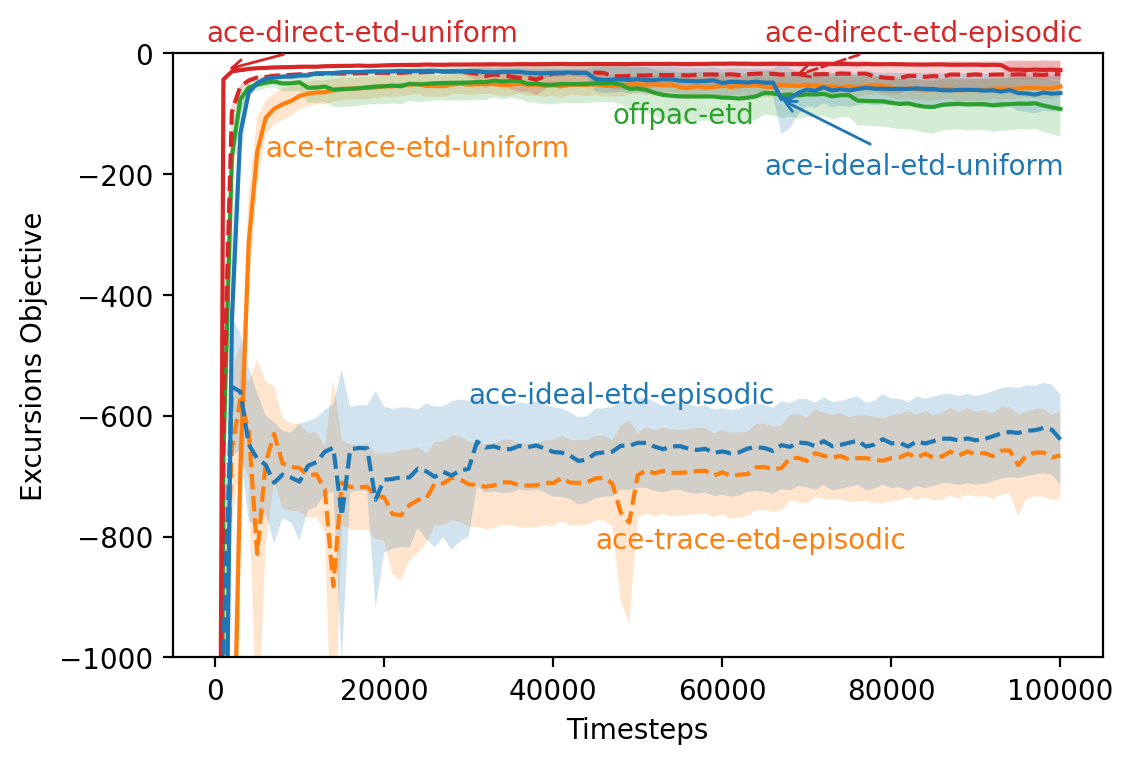}
    \caption{Excursions learning curves (ETD)}
    \label{fig:pw_exc_etd}
  \end{subfigure}
  \hfill
  \begin{subfigure}[b]{0.45\textwidth}
    \centering
    \includegraphics[width=\textwidth]{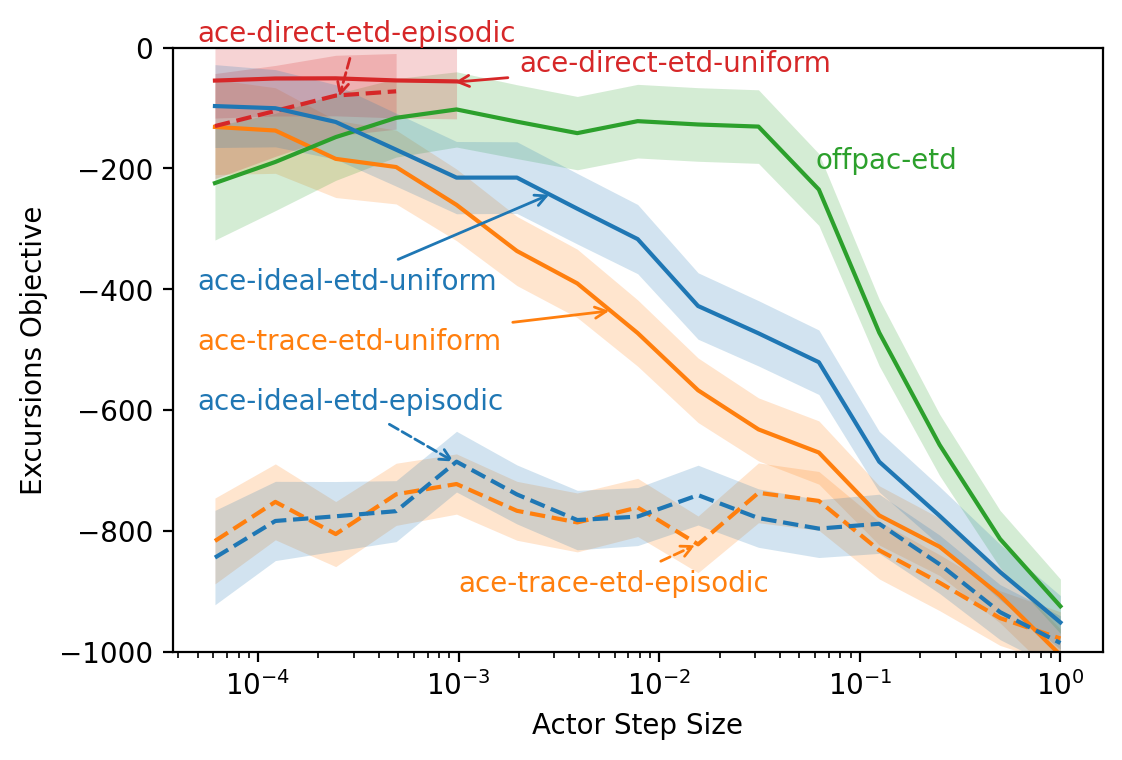}
    \caption{Excursions step-size sensitivity (ETD)}
    \label{fig:pw_exc_etd_sens}
  \end{subfigure}
  \begin{subfigure}[b]{0.45\textwidth}
    \centering
    \includegraphics[width=\textwidth]{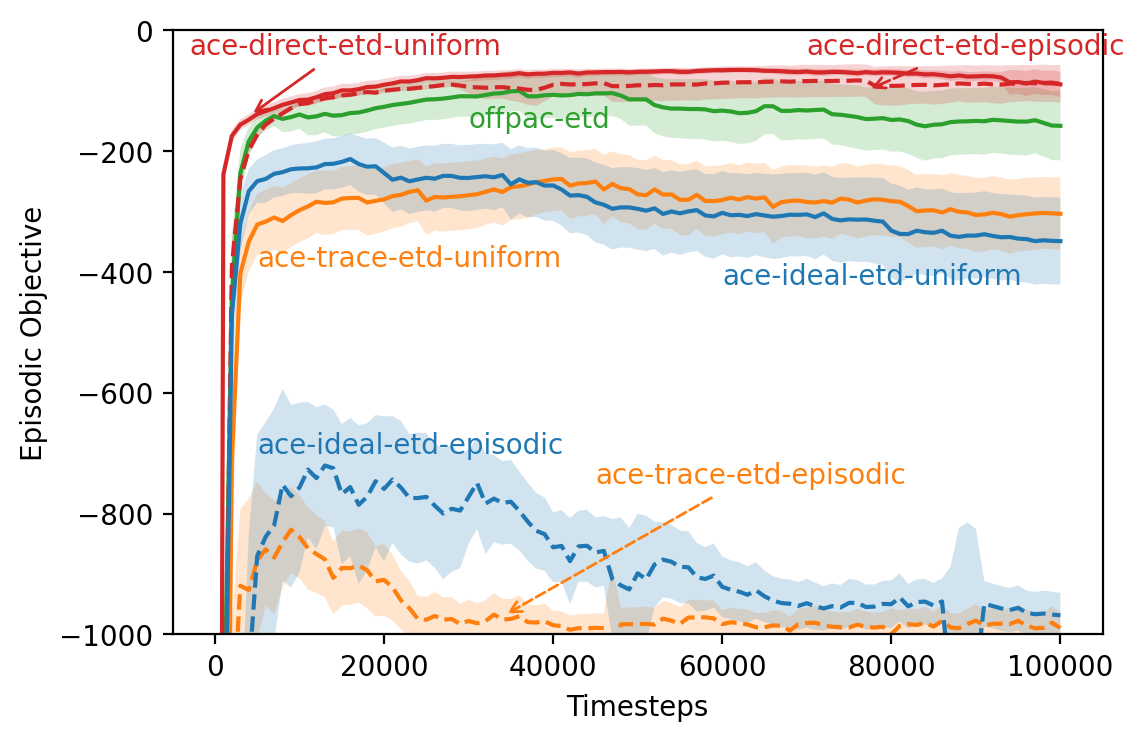}
    \caption{Episodic learning curves (ETD)}
    \label{fig:pw_epi_etd}
  \end{subfigure}
  \hfill
  \begin{subfigure}[b]{0.45\textwidth}
    \centering
    \includegraphics[width=\textwidth]{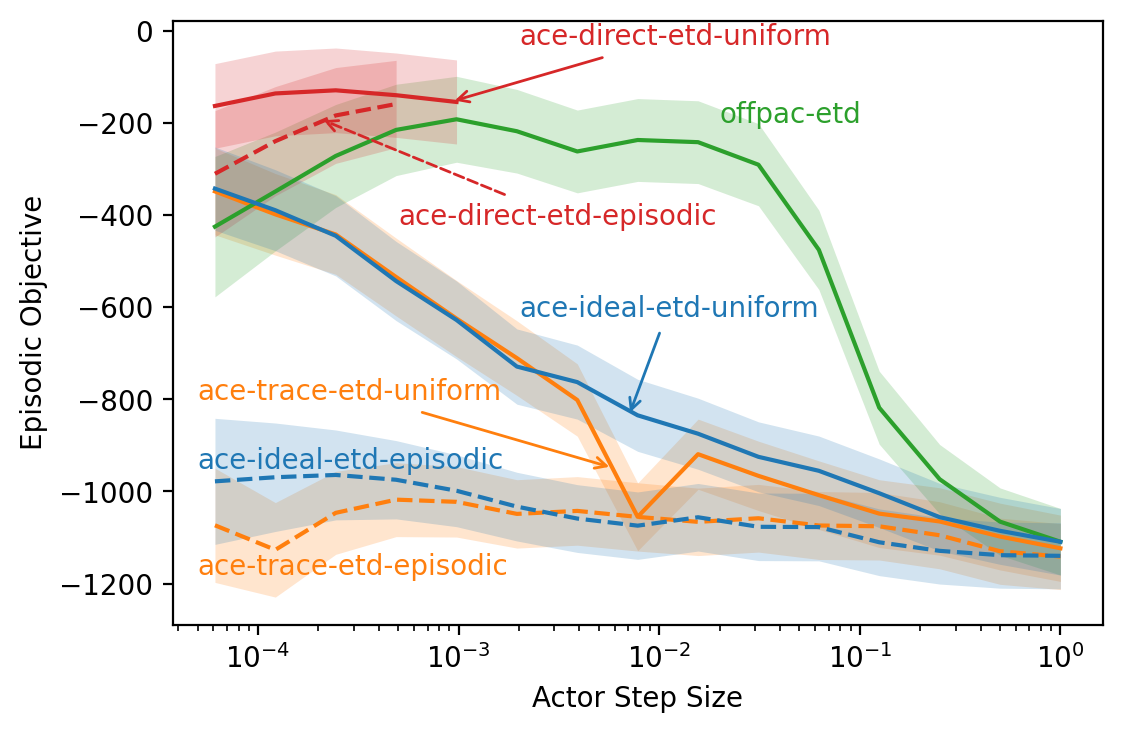}
    \caption{Episodic step-size sensitivity (ETD)}
    \label{fig:pw_epi_etd_sens}
  \end{subfigure}
  \begin{subfigure}[b]{0.45\textwidth}
    \centering
    \includegraphics[width=\textwidth]{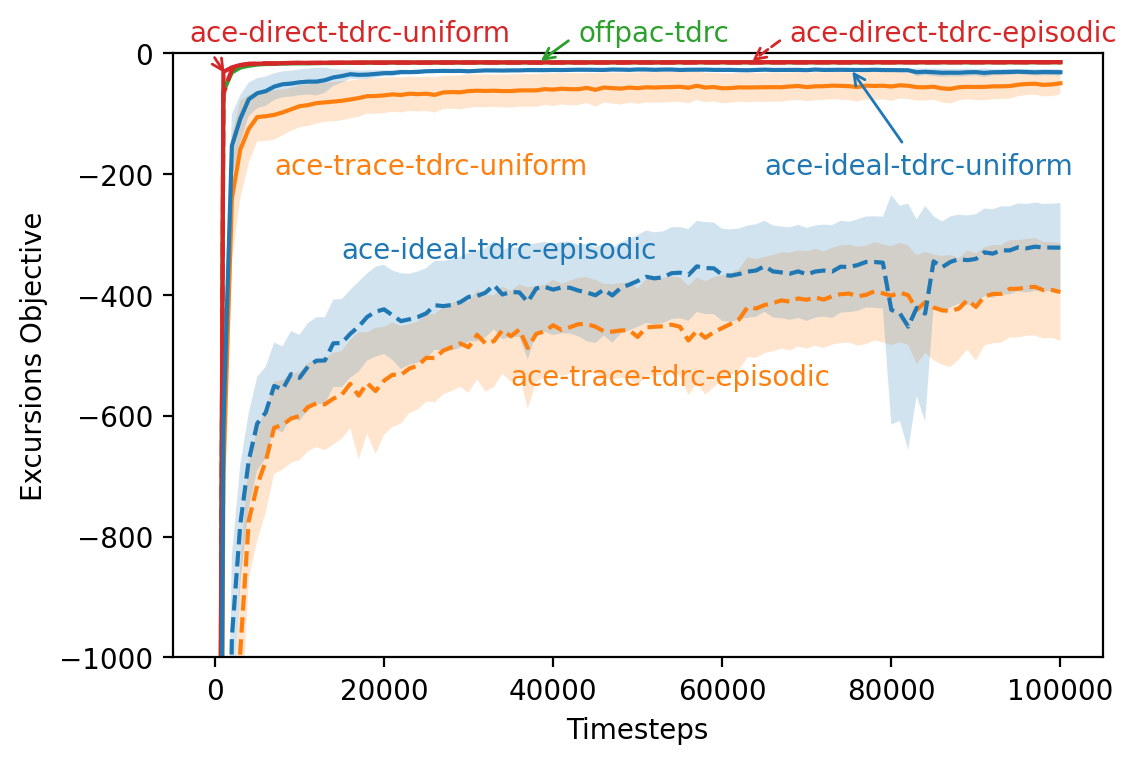}
    \caption{Excursions learning curves (TDRC)}
    \label{fig:pw_exc_tdrc}
  \end{subfigure}
  \hfill
  \begin{subfigure}[b]{0.45\textwidth}
    \centering
    \includegraphics[width=\textwidth]{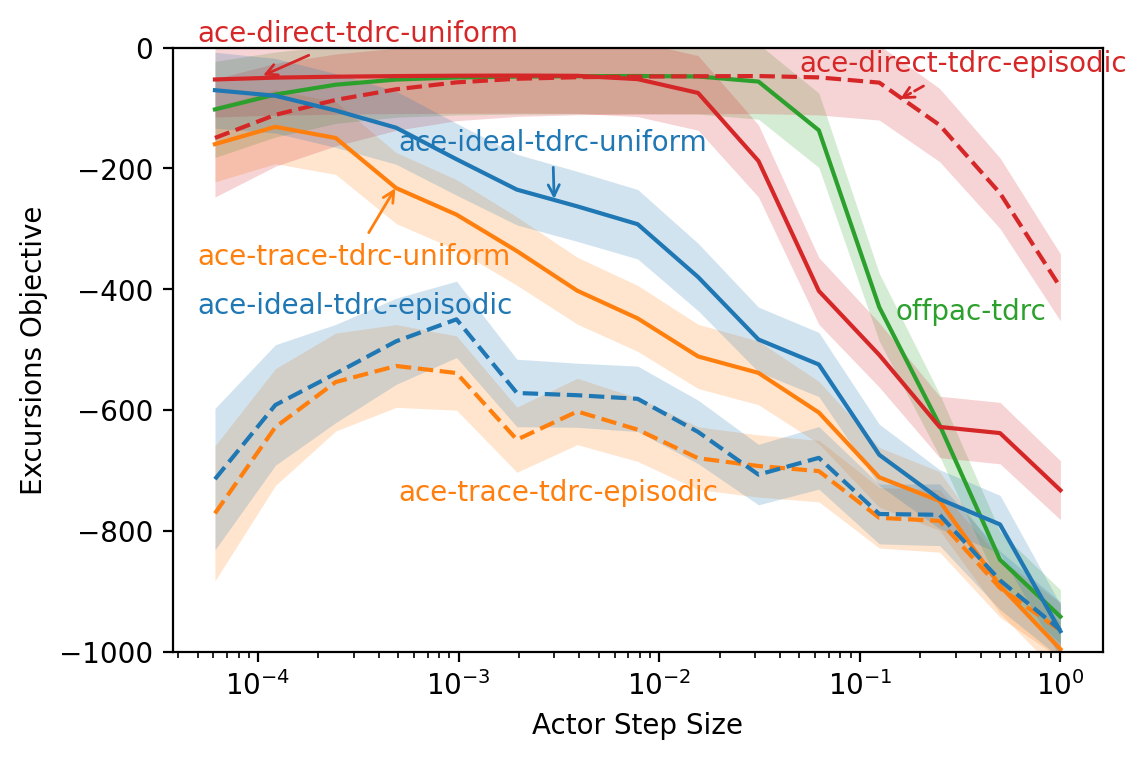}
    \caption{Excursions step-size sensitivity (TDRC)}
    \label{fig:pw_exc_tdrc_sens}
  \end{subfigure}
  \begin{subfigure}[b]{0.45\textwidth}
    \centering
    \includegraphics[width=\textwidth]{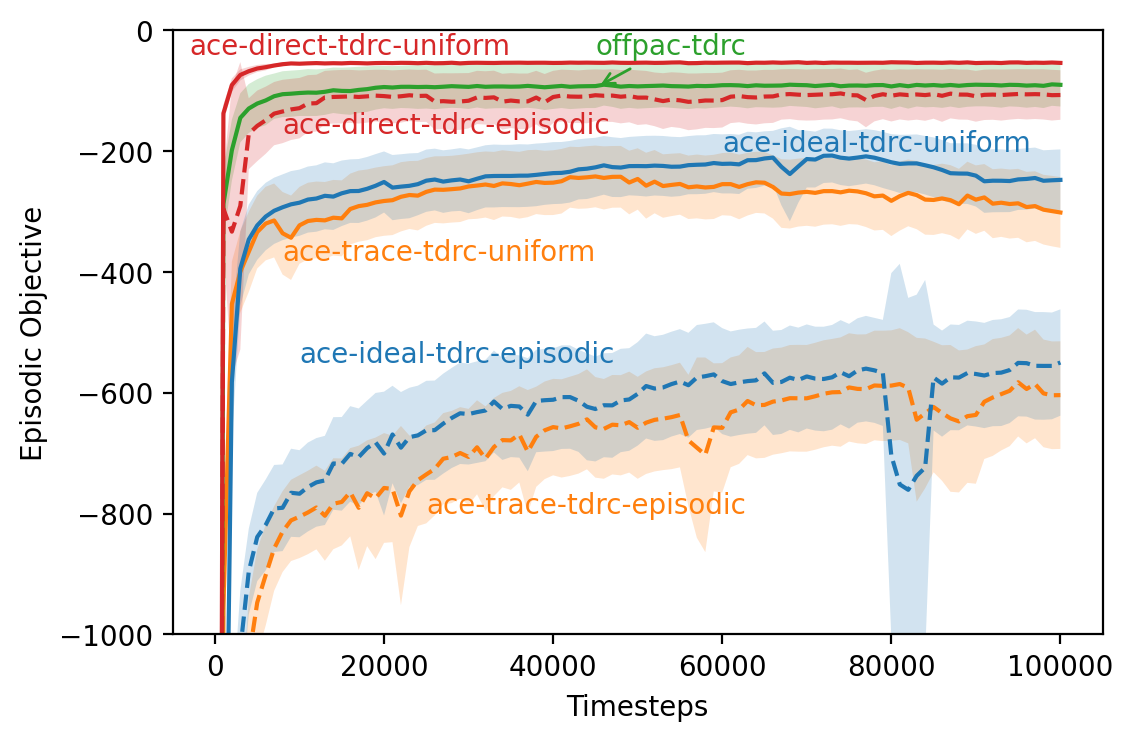}
    \caption{Episodic learning curves (TDRC)}
    \label{fig:pw_epi_tdrc}
  \end{subfigure}
  \hfill
  \begin{subfigure}[b]{0.45\textwidth}
    \centering
    \includegraphics[width=\textwidth]{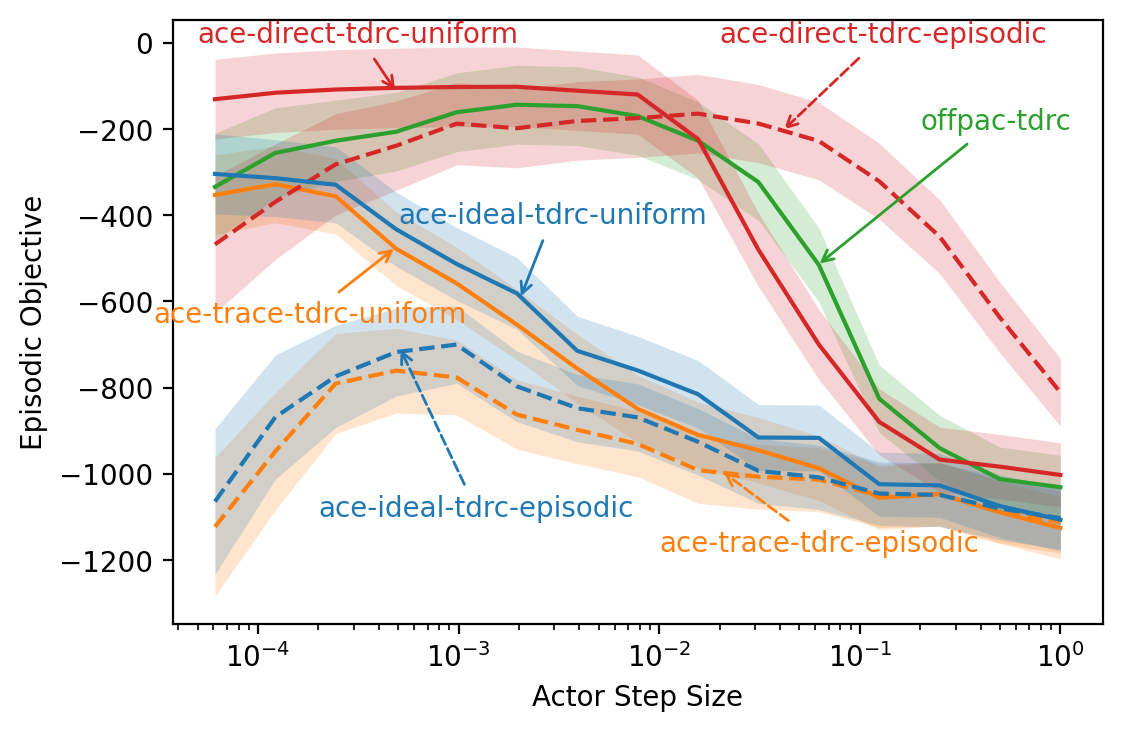}
    \caption{Episodic step-size sensitivity (TDRC)}
    \label{fig:pw_epi_tdrc_sens}
  \end{subfigure}
  \caption{Results on Puddle World. Shaded regions are 95\% confidence intervals.}
  \label{fig:pw}
\end{figure*}

Comparing the performance of each algorithm under the excursions and episodic objective function, we can see that the algorithms all performed better in the excursions setting than the episodic setting.
This is due to the excursions setting being easier than the episodic setting for the classic control problems in this section.
In the excursions setting, the agent starts from the behaviour policy's state distribution, which in the problems studied is closer to the goal state than the start state.
This is not necessarily true for all problems, however, and in general there is no strict ordering of objectives in terms of difficulty.

ACE variants that used the episodic interest function (dashed lines) generally performed much worse than their counterparts using the uniform interest function (solid lines),
with the exception of ACE-direct.
This is consistent with the discussion in Section 7 of \citet{thomas2014bias}---although they deal with the on-policy case---which states scaling policy updates by an accumulating product of discount factors hurts the sample efficiency of the algorithm.
However, the off-policy case can exacerbate this issue as the accumulating product includes both discount factors and importance sampling ratios.
The fact that the direct method of estimating emphatic weightings allowed ACE to perform well when using the episodic interest function is intriguing and merits further study.

The choice of critic affected the performance of the algorithms to differing degrees.
Replacing an ETD critic with a TDRC critic improved the performance of each of the algorithms, with episodic performance showing the greatest improvement (compare figures \ref{fig:pw_epi_etd} and \ref{fig:pw_epi_tdrc}), perhaps due to the excursions objective being easier for this environment as previously mentioned.
ACE-direct with uniform interest performed well regardless of the choice of critic, but improved substantially in the episodic setting when using a TDRC critic.
Using a TDRC critic instead of an ETD critic also improved the sensitivity of the algorithms to the actor's step size, allowing a slightly wider range of step sizes to perform well for most methods (compare Figure \ref{fig:pw_exc_etd_sens} to \ref{fig:pw_exc_tdrc_sens}, and Figure \ref{fig:pw_epi_etd_sens} to \ref{fig:pw_epi_tdrc_sens}).

When comparing ACE-trace (orange) to ACE-ideal (blue)---where the emphatic trace's importance sampling ratios were computed using the current policy, resulting in an unbiased Monte Carlo estimate of the emphatic weightings---we can see that correcting the bias introduced by using the emphatic trace with a changing policy did not improve performance significantly in all but one situation.
The one exception was when performance was evaluated with the excursions objective function (figures \ref{fig:pw_exc_etd} and \ref{fig:pw_exc_tdrc}).
When using an ETD critic (Figure \ref{fig:pw_exc_etd}), ACE-ideal outperformed ACE-trace early in learning before deteriorating to a similar level of performance by the end of the experiment.
Conversely, ACE-ideal outperformed ACE-trace throughout the experiment when using a TDRC critic (Figure \ref{fig:pw_exc_tdrc}), although the difference was barely significant.
In all other situations, the performance of ACE-ideal was not significantly different from the performance of ACE-trace.

Overall, ACE-direct performed better than or equal to ACE-trace, ACE-ideal, and OffPAC across all combinations of critic, interest function, and objective function.
However, it diverged for some actor step sizes when used with an ETD critic (figures \ref{fig:pw_exc_etd_sens} and \ref{fig:pw_epi_etd_sens}).
This could be due to the use of a semi-gradient update rule (analogous to off-policy TD(0) which is not guaranteed to converge when used with function approximation) for the direct method of estimating emphatic weightings.
To resolve this issue, one could use the gradient-based update rule in equation \ref{m_est_grad_update} of Section \ref{sec_direct_f}.
We chose to use the semi-gradient update to keep both the explanation and experiments simple, as there were already a large number of concepts and algorithmic parameters involved.

\subsubsection{Off-policy Mountain Car}
In the original Mountain Car environment, the agent attempts to drive an underpowered car out of a valley \citep{moore1990efficient}.\footnote{Please see \citet{degris2012offpolicy} for a picture of the Mountain Car environment.}
In the off-policy version, the agent learns from experience generated by a fixed behaviour policy, in this case the uniform random policy.
The observations were tile coded using 8 tilings of $4 \times 4$ tiles plus a bias unit.
ACE-ideal was omitted from the Mountain Car experiments due to computational constraints; the uniform random behaviour policy rarely completed an episode, which resulted in extremely long trajectories that prevented ACE-ideal from running in a reasonable amount of time.

Figure \ref{fig:mc} contains the results for the Off-policy Mountain Car environment.
The left-hand column shows learning curves, while the right-hand column shows sensitivity analyses for the actor's step size.
The first four plots (figures \ref{fig:mc_exc_etd}, \ref{fig:mc_exc_etd_sens}, \ref{fig:mc_epi_etd}, and \ref{fig:mc_epi_etd_sens}) show results when using ETD as the critic, while the last four plots (figures \ref{fig:mc_exc_tdrc}, \ref{fig:mc_exc_tdrc_sens}, \ref{fig:mc_epi_tdrc}, and \ref{fig:mc_epi_tdrc_sens}) show results for a TDRC critic.
The first and third rows show results using the excursions objective function, and the second and fourth rows show results for the episodic objective function.

\begin{figure*}[hp!]
  \centering
  \begin{subfigure}[b]{0.45\textwidth}
    \centering
    \includegraphics[width=\textwidth]{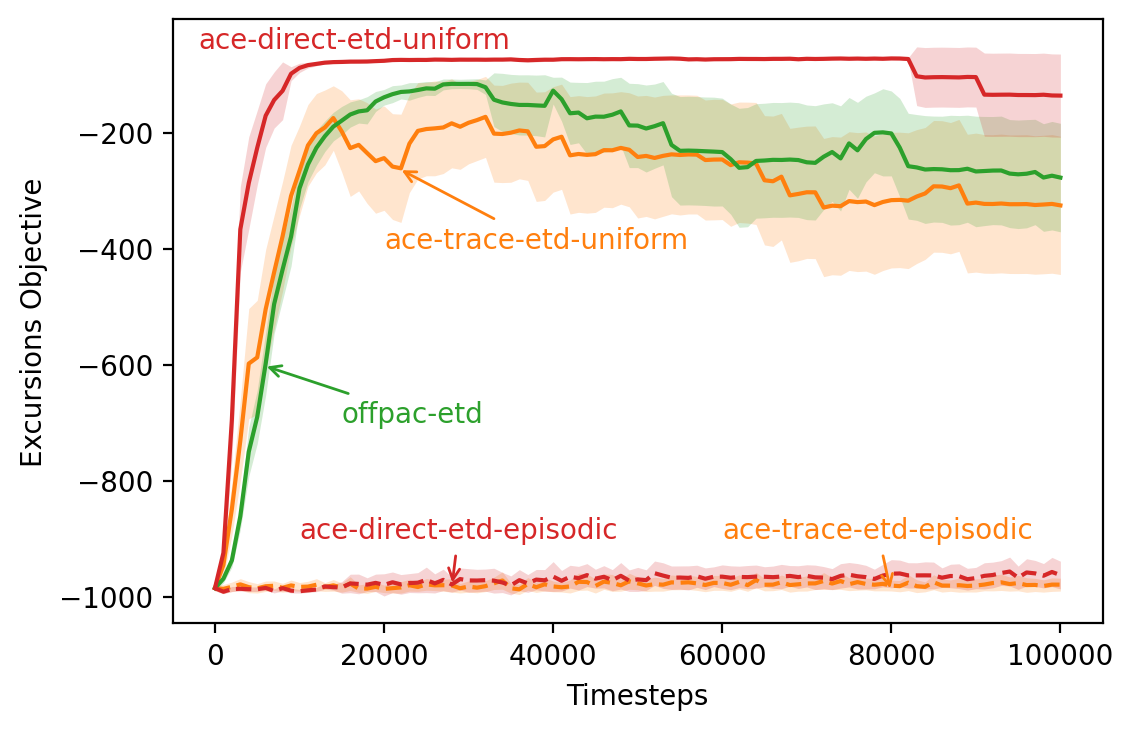}
    \caption{Excursions learning curves (ETD)}
    \label{fig:mc_exc_etd}
  \end{subfigure}
  \hfill
  \begin{subfigure}[b]{0.45\textwidth}
    \centering
    \includegraphics[width=\textwidth]{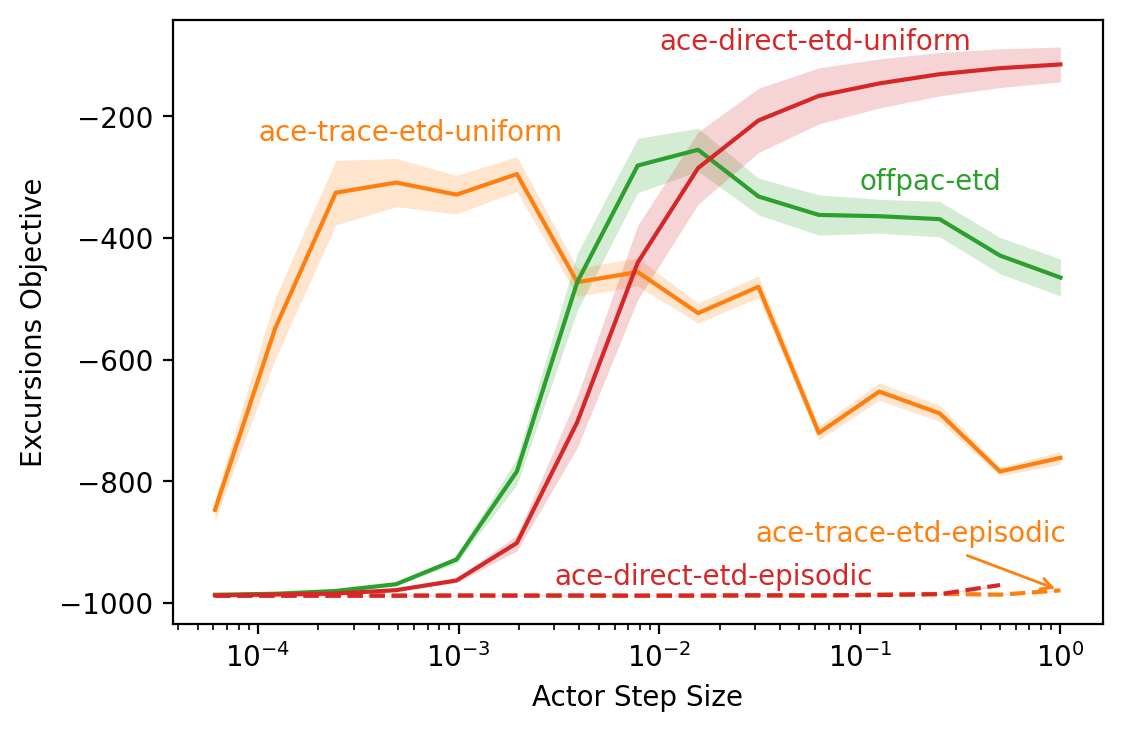}
    \caption{Excursions step-size sensitivity (ETD)}
    \label{fig:mc_exc_etd_sens}
  \end{subfigure}
  \begin{subfigure}[b]{0.45\textwidth}
    \centering
    \includegraphics[width=\textwidth]{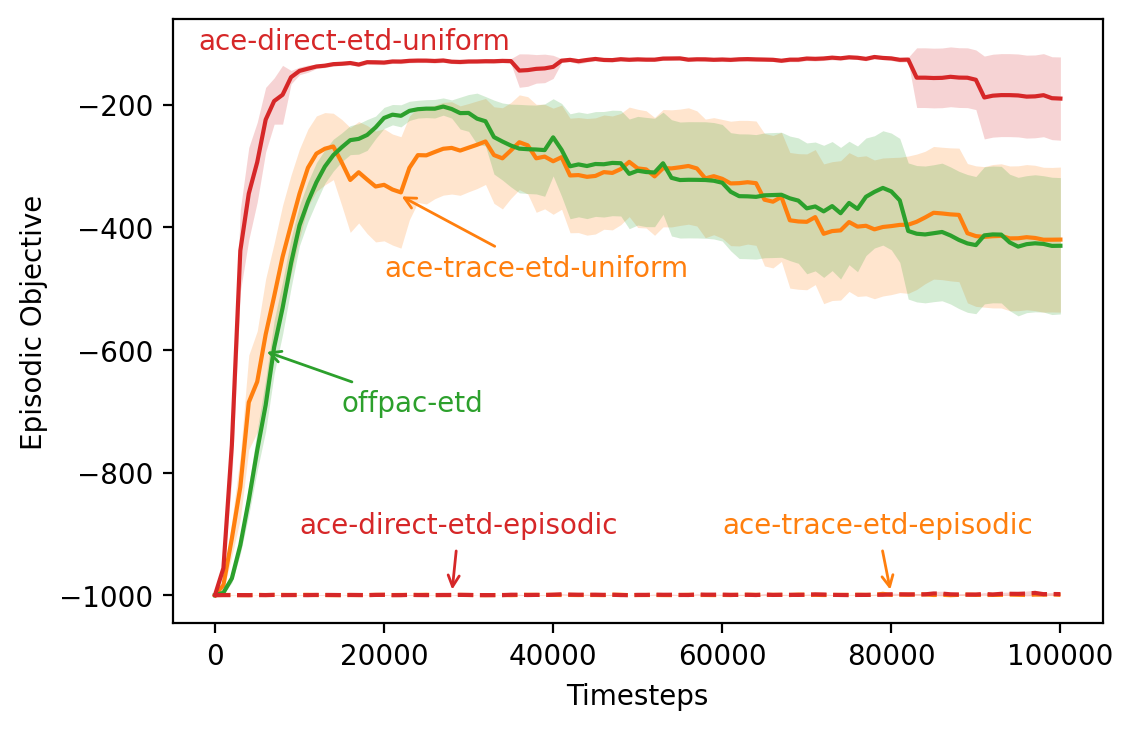}
    \caption{Episodic learning curves (ETD)}
    \label{fig:mc_epi_etd}
  \end{subfigure}
  \hfill
  \begin{subfigure}[b]{0.45\textwidth}
    \centering
    \includegraphics[width=\textwidth]{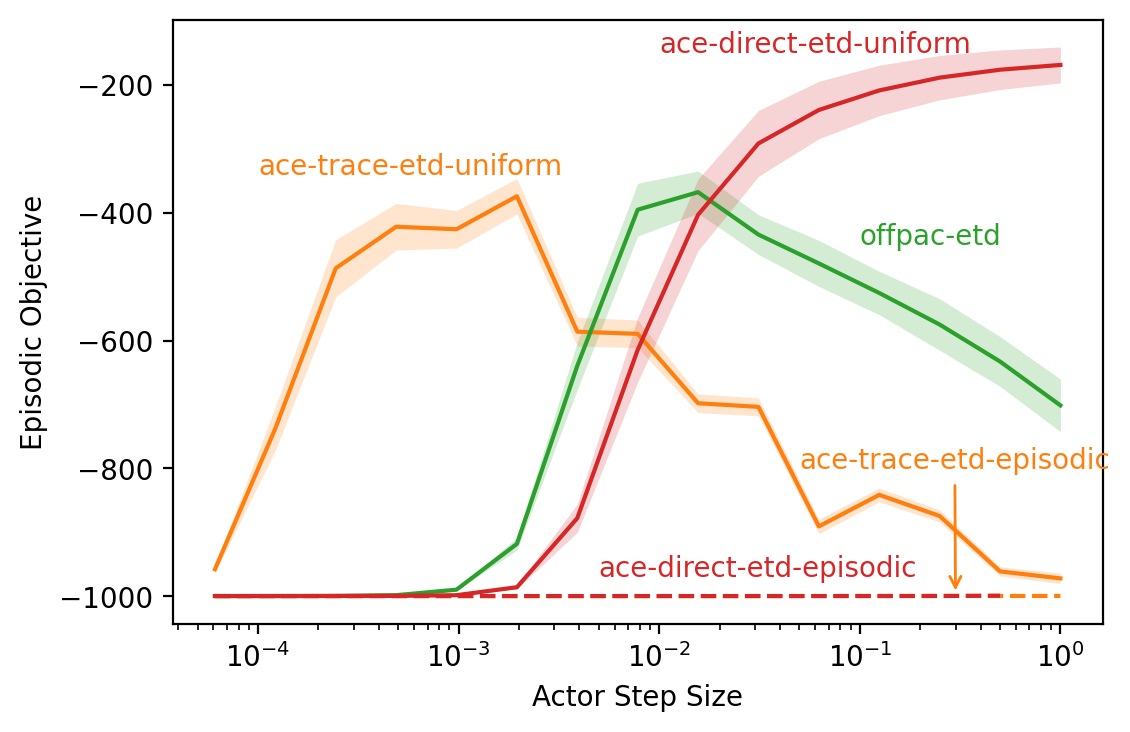}
    \caption{Episodic step-size sensitivity (ETD)}
    \label{fig:mc_epi_etd_sens}
  \end{subfigure}
  \begin{subfigure}[b]{0.45\textwidth}
    \centering
    \includegraphics[width=\textwidth]{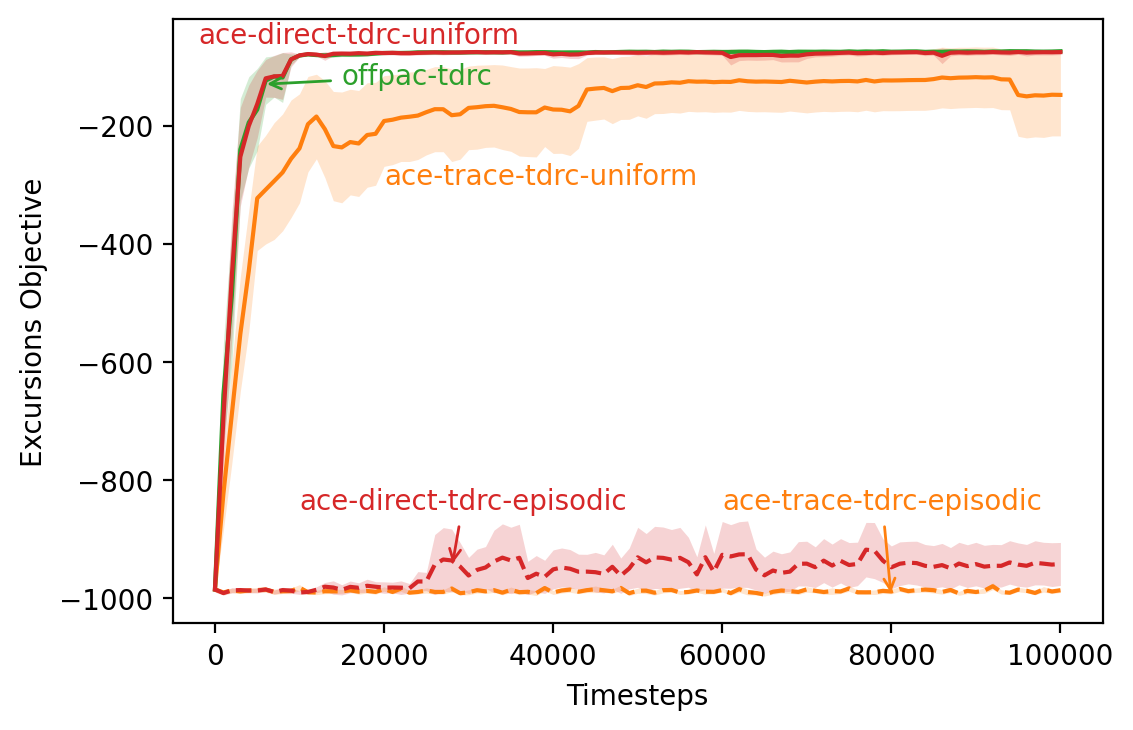}
    \caption{Excursions learning curves (TDRC)}
    \label{fig:mc_exc_tdrc}
  \end{subfigure}
  \hfill
  \begin{subfigure}[b]{0.45\textwidth}
    \centering
    \includegraphics[width=\textwidth]{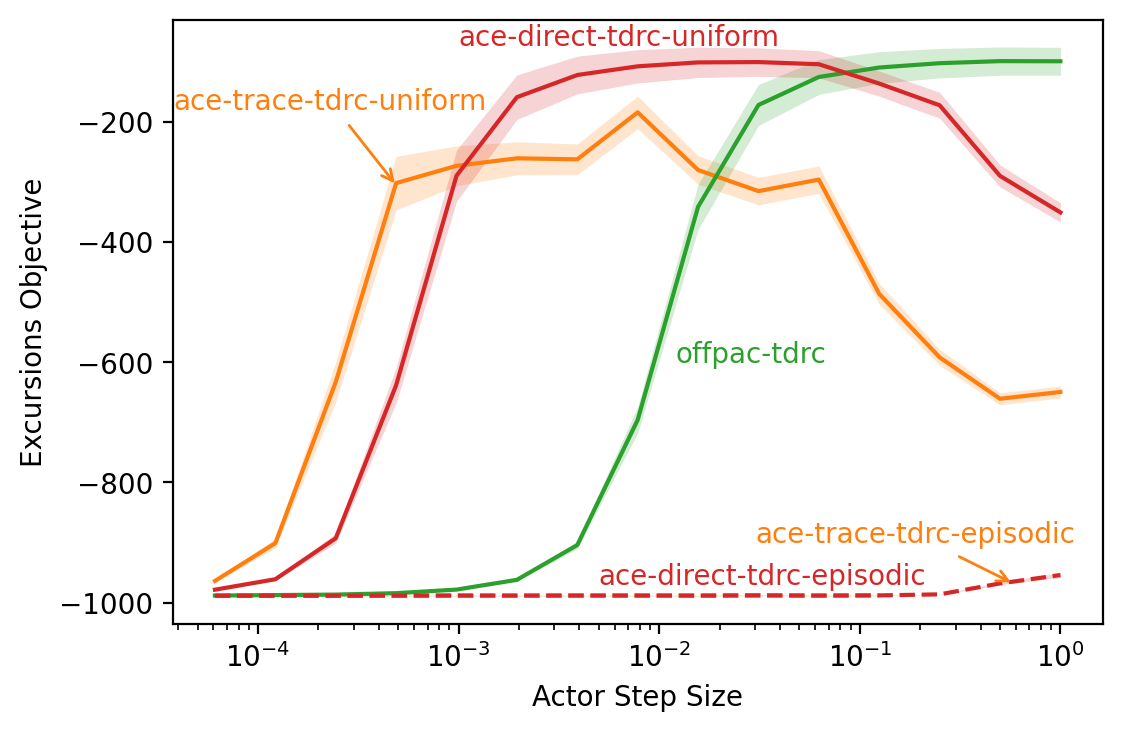}
    \caption{Excursions step-size sensitivity (TDRC)}
    \label{fig:mc_exc_tdrc_sens}
  \end{subfigure}
  \begin{subfigure}[b]{0.45\textwidth}
    \centering
    \includegraphics[width=\textwidth]{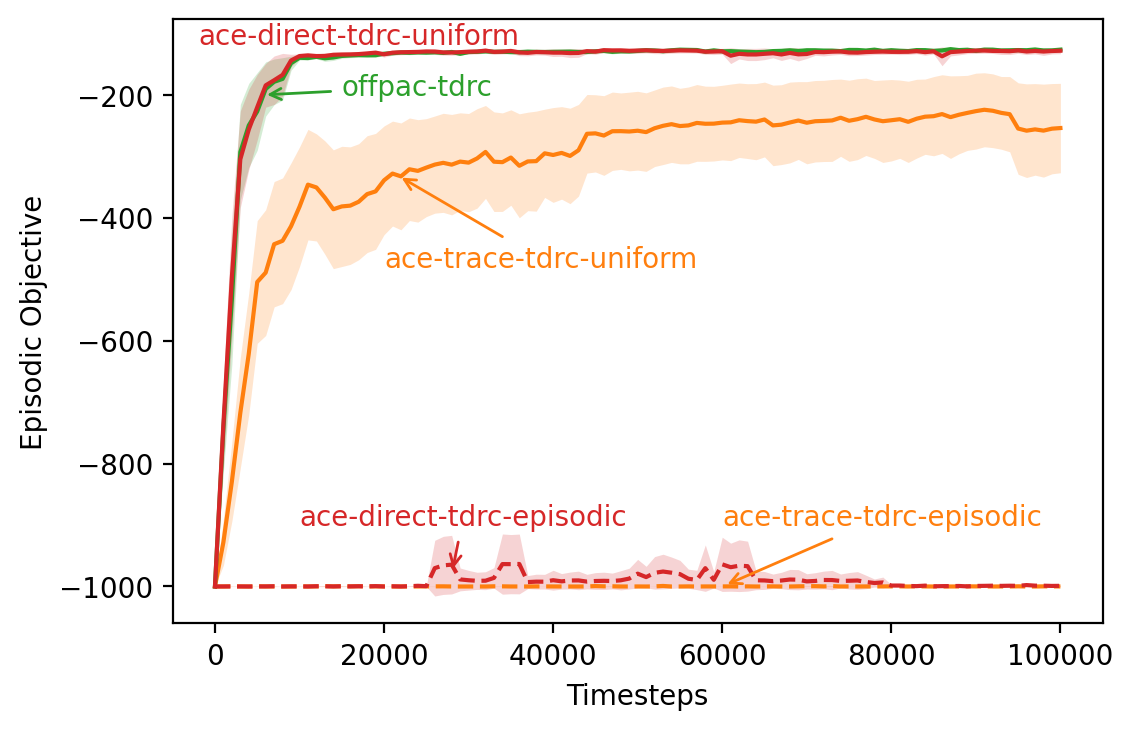}
    \caption{Episodic learning curves (TDRC)}
    \label{fig:mc_epi_tdrc}
  \end{subfigure}
  \hfill
  \begin{subfigure}[b]{0.45\textwidth}
    \centering
    \includegraphics[width=\textwidth]{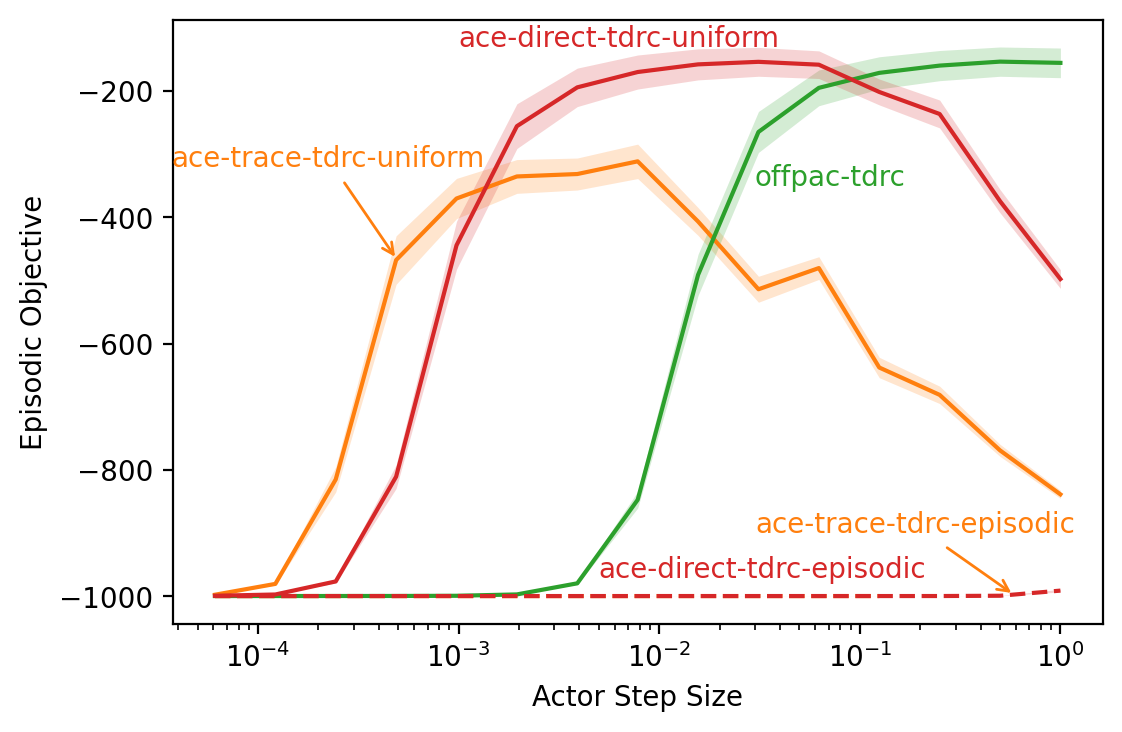}
    \caption{Episodic step-size sensitivity (TDRC)}
    \label{fig:mc_epi_tdrc_sens}
  \end{subfigure}
  \caption{Results on Off-policy Mountain Car. Shaded regions are 95\% confidence intervals.}
  \label{fig:mc}
\end{figure*}

Across all choices of critic and objective function, algorithms trained using the episodic interest function (dashed lines) performed much worse than their counterparts trained using the uniform interest function (solid lines).
This finding is consistent with the results from Puddle World, although much more pronounced.
In this specific experiment, the extremely long trajectories generated by the uniform random behaviour policy caused the emphatic trace (and hence the policy updates) of the algorithms that used 
episodic interest to quickly shrink to the point of irrelevance, whereas in the Puddle World experiment the behaviour policy encounters the goal state more often, limiting trajectory length and shrinking the updates more slowly.
Nevertheless, using the direct method to estimate the emphatic weightings allowed ACE to improve the performance of the policy in some cases, whereas using the emphatic trace prevented the policy from meaningfully improving.

Again, the choice of critic played a role in the performance of all the algorithms to varying degrees.
Using a TDRC critic instead of an ETD critic improved the performance of each of the algorithms, with OffPAC seeing the greatest improvement, followed by ACE-trace.
ACE-direct again performed well regardless of the choice of critic.
Using a TDRC critic improved the sensitivity of each of the algorithms to the actor's step size, allowing a wider range of step sizes to perform well (compare figures \ref{fig:mc_exc_etd_sens} and \ref{fig:mc_exc_tdrc_sens}, and \ref{fig:mc_epi_etd_sens} and \ref{fig:mc_epi_tdrc_sens}).

Comparing ACE-direct (red) to ACE-trace (orange) reveals several advantages of using the direct method to estimate the emphatic weightings.
The direct method allowed ACE to learn faster, with less variance, and find a better-performing target policy by the end of the experiment, regardless of the choice of critic or objective function used for evaluation.
In addition, the direct method allowed a wider range of actor step sizes to perform well, with performance increasing and decreasing more smoothly and predictably than when using the emphatic trace.
Finally, the choice of critic had little effect on the performance of ACE-direct, whereas ACE-trace performed better with a TDRC critic than with an ETD critic.
Overall, ACE with the direct method performed better than or equal to OffPAC (green) and ACE-trace for all critics and objective functions, but was less sensitive to the actor step size and choice of critic than OffPAC.

\subsection{The Virtual Office Environment}

The classic control environments considered in the previous section are well-known environments for testing general off-policy algorithms, but were not designed to probe the weaknesses of the specific algorithms being studied.
Hence, both semi-gradient updates and the direct method of estimating emphatic weightings perform well, and the tradeoffs made in the design of each algorithm are not obvious.
In this section, we introduce an environment designed to 
highlight issues with both semi-gradient updates and 
the direct method of estimating emphatic weightings.

The environment is based on \citet{dung2007reinforcement}'s Virtual Office, a grid world consisting of a hallway and two cubicles that are identical except for the goal locations.
To illustrate the problem with semi-gradient updates,
two key properties of the counterexample from Section \ref{sec_counter} must be incorporated:
the optimal actions in the two identical rooms must be different, and the behaviour policy must visit the suboptimal room more often.
To accomplish the first goal, we introduced hidden goal states in the north-east and south-east corners of the two rooms, and assigned rewards of 1 and 0 respectively to the goal states of the north-east room and rewards of 0 and .5 respectively to the goal states of the south-east room.
To encourage the behaviour policy to visit the suboptimal room (i.e., the south-east room) more often, we fixed the starting state to the left-most, middle-most state and used a behaviour policy that favoured the South and East actions, taking the North, East, South, and West actions with probabilities .2, .4, .3, and .1 respectively.

While \citet{dung2007reinforcement}'s Virtual Office contains some partial observability in the form of the two identical rooms, to fully illustrate the tradeoff made by the Markov assumption used in the derivation of the direct method, we limited the agent's observations to a $3 \times 3$ grid in front of the agent.
The RGB colour values of each square in the agent's view were used as observations, with the agent unable to see through walls or perceive the goal states.

The final environment was implemented using Minigrid \citep{gym_minigrid}, and is depicted in Figure \ref{fig:virtual_office}.
The agent (red triangle) must navigate to one of the terminal states (the north-east and south-east squares of the green rooms) to obtain the associated reward.
The agent observes the RGB colour values of the $3 \times 3$ grid of squares in front of it (the lighter blue squares), and selects the North, East, South, or West action in response.

For this experiment we followed the methodology detailed in the previous section, with minor changes.
First, we restricted the critic trace decay rate to be 0, as eligibility traces can alleviate some of the issues caused by partial observability \citep{loch1998using}.
Second, we ran each combination of parameters for 30 runs of 200,000 time steps each.
The learned policies were saved every 5,000 time steps and evaluated 50 times using both the episodic and excursions objective functions.
The best-performing parameter settings were then re-run 100 times and averaged.

\begin{figure*}[t!]
  \vspace{-0.3cm}
  \centering
  \includegraphics[width=\figwidthfour]{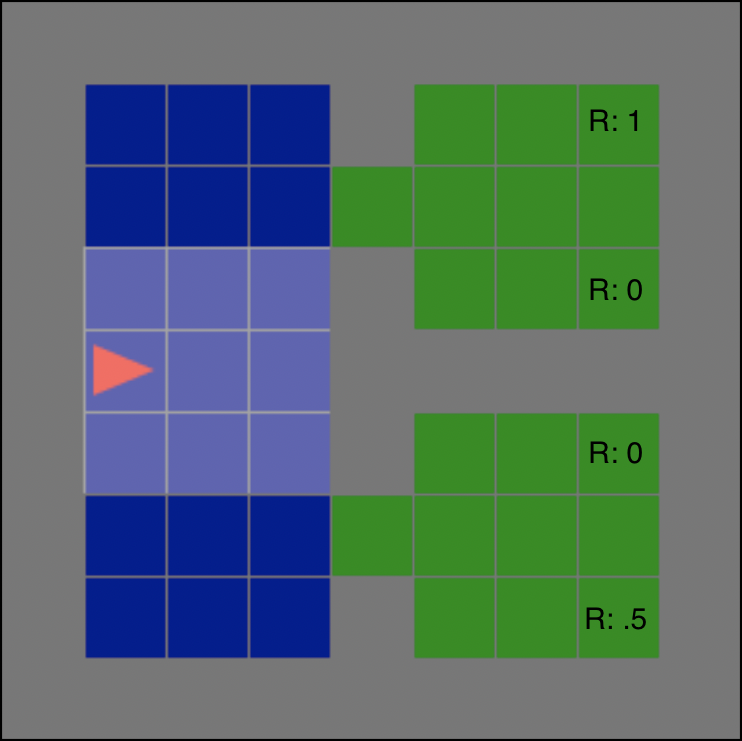}
  \caption{The Virtual Office environment.}
  \label{fig:virtual_office}
  \vspace{-0.3cm}
\end{figure*}

Figure \ref{fig:vo} contains the results of our experiment on the Virtual Office environment.
The left-hand column shows learning curves, while the right-hand column shows sensitivity analyses for the actor's step size.
The first four plots (figures \ref{fig:vo_exc_etd}, \ref{fig:vo_exc_etd_sens}, \ref{fig:vo_epi_etd}, and \ref{fig:vo_epi_etd_sens}) show results when using ETD as the critic, while the last four plots (figures \ref{fig:vo_exc_tdrc}, \ref{fig:vo_exc_tdrc_sens}, \ref{fig:vo_epi_tdrc}, and \ref{fig:vo_epi_tdrc_sens}) show results for a TDRC critic.
The first and third rows show results for the excursions objective function, and the second and fourth rows show the episodic objective function.

\begin{figure*}[hp!]
  \vspace{-0.3cm}
  \centering
  \begin{subfigure}[b]{0.45\textwidth}
    \centering
    \includegraphics[width=\textwidth]{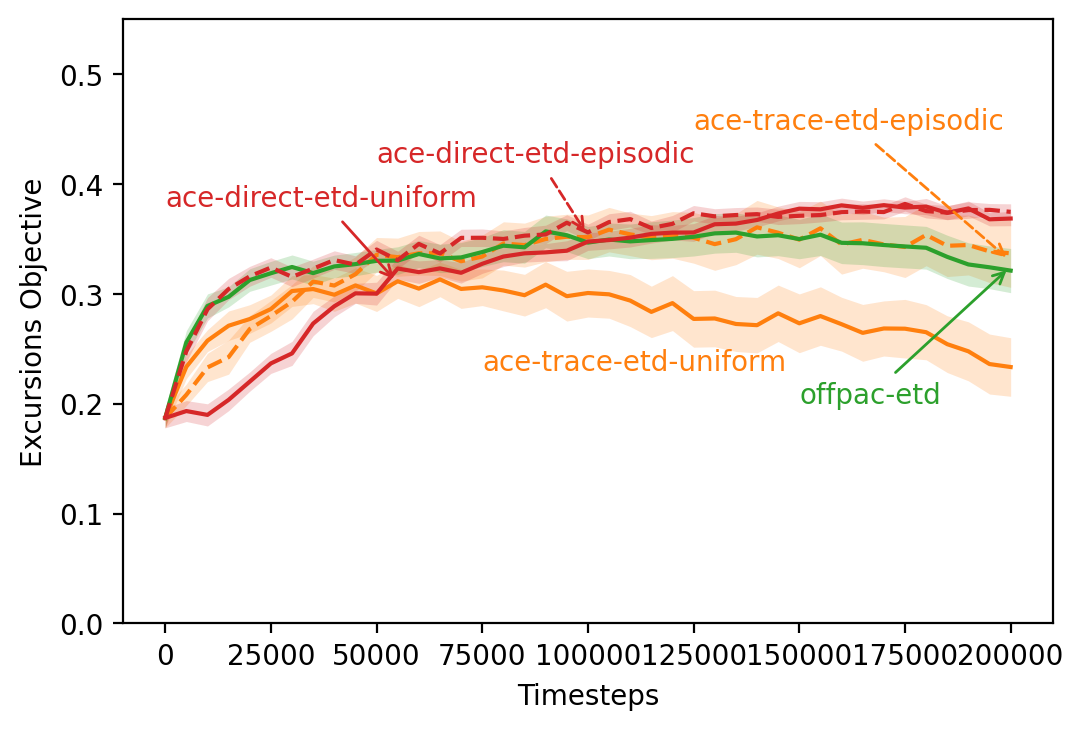}
    \caption{Excursions learning curves (ETD)}
    \label{fig:vo_exc_etd}
  \end{subfigure}
  \hfill
  \begin{subfigure}[b]{0.45\textwidth}
    \centering
    \includegraphics[width=\textwidth]{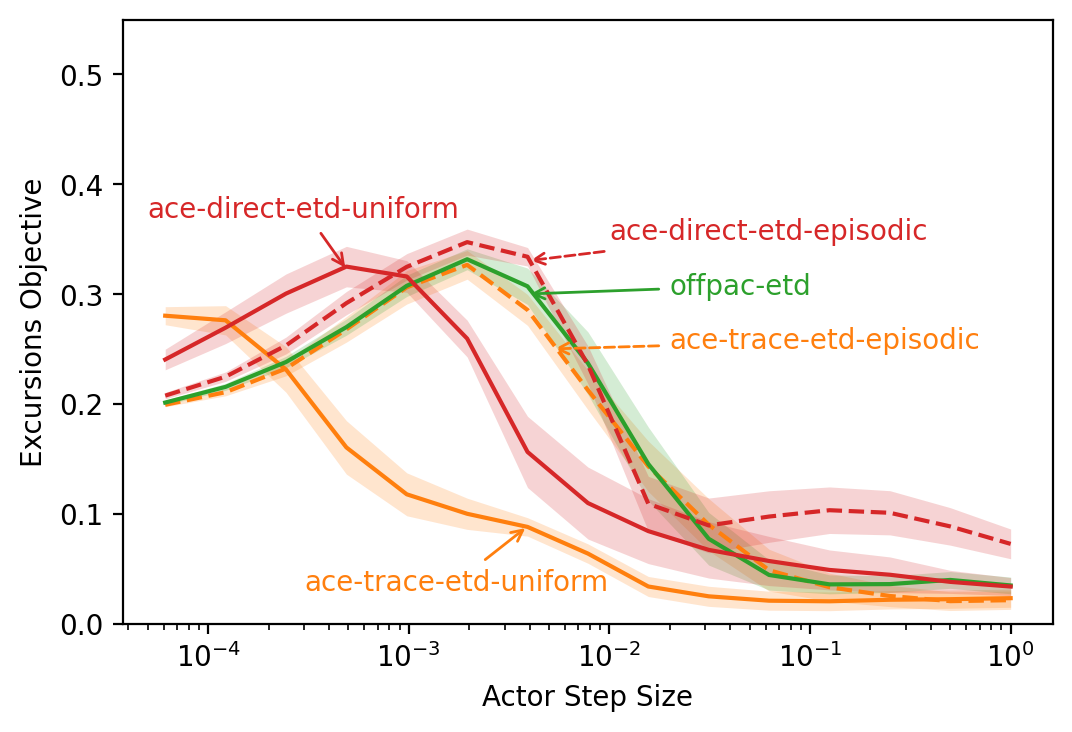}
    \caption{Excursions step-size sensitivity (ETD)}
    \label{fig:vo_exc_etd_sens}
  \end{subfigure}
  \begin{subfigure}[b]{0.45\textwidth}
    \centering
    \includegraphics[width=\textwidth]{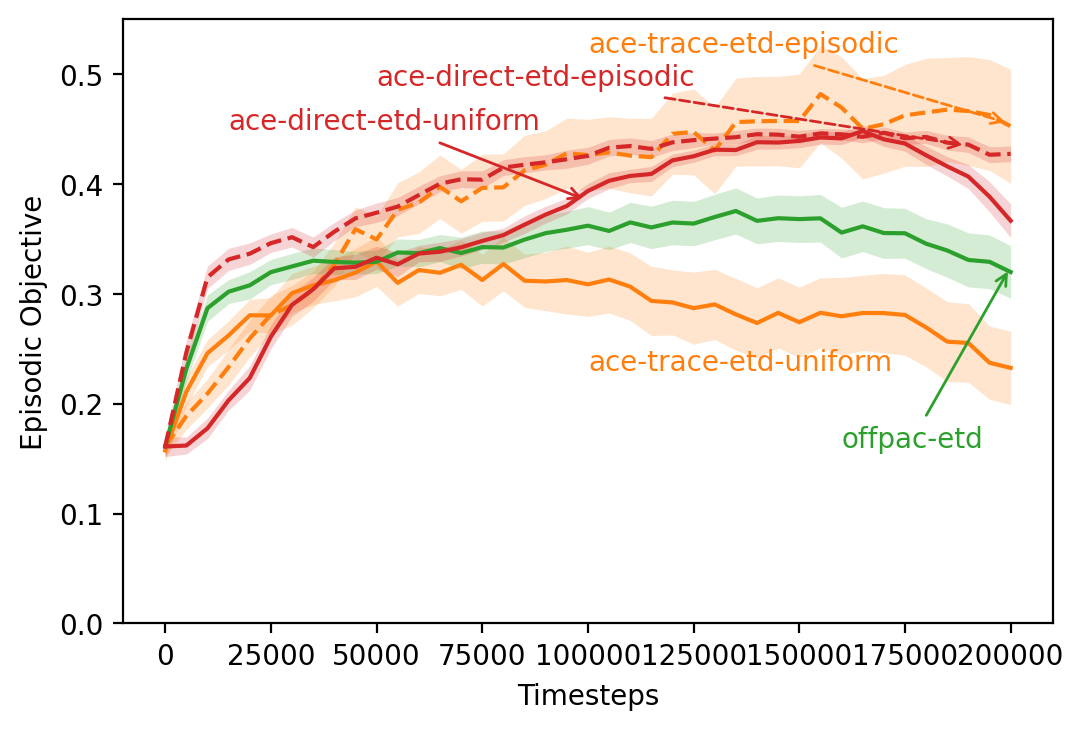}
    \caption{Episodic learning curves (ETD)}
    \label{fig:vo_epi_etd}
  \end{subfigure}
  \hfill
  \begin{subfigure}[b]{0.45\textwidth}
    \centering
    \includegraphics[width=\textwidth]{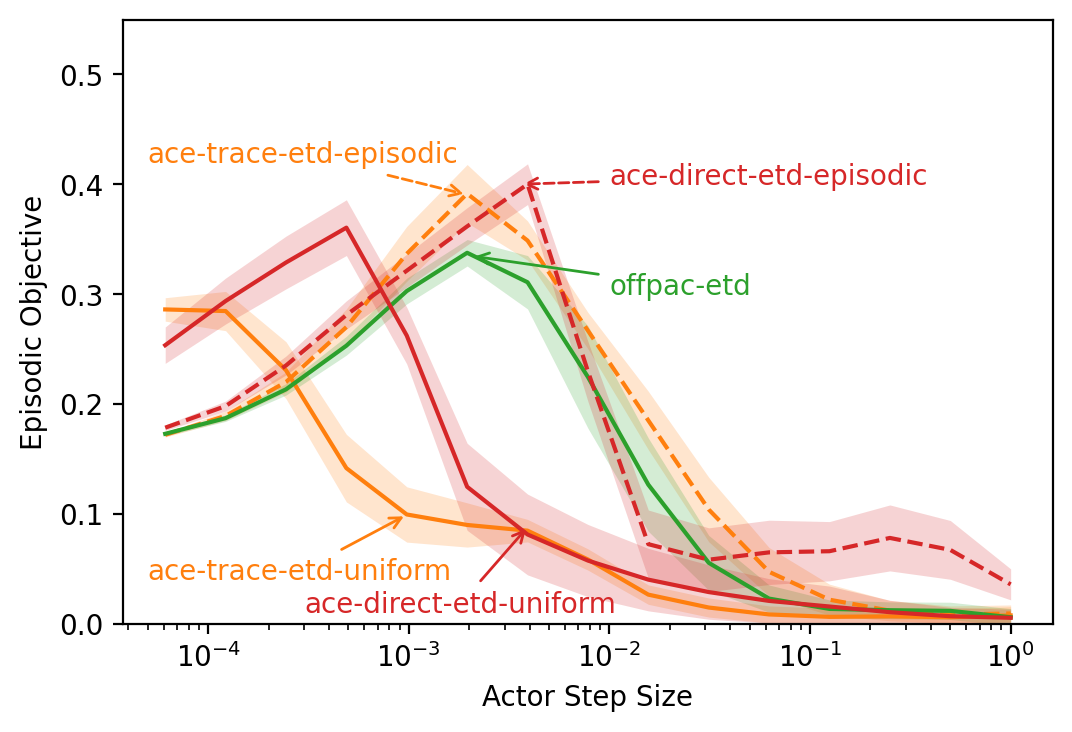}
    \caption{Episodic step-size sensitivity (ETD)}
    \label{fig:vo_epi_etd_sens}
  \end{subfigure}
  \begin{subfigure}[b]{0.45\textwidth}
    \centering
    \includegraphics[width=\textwidth]{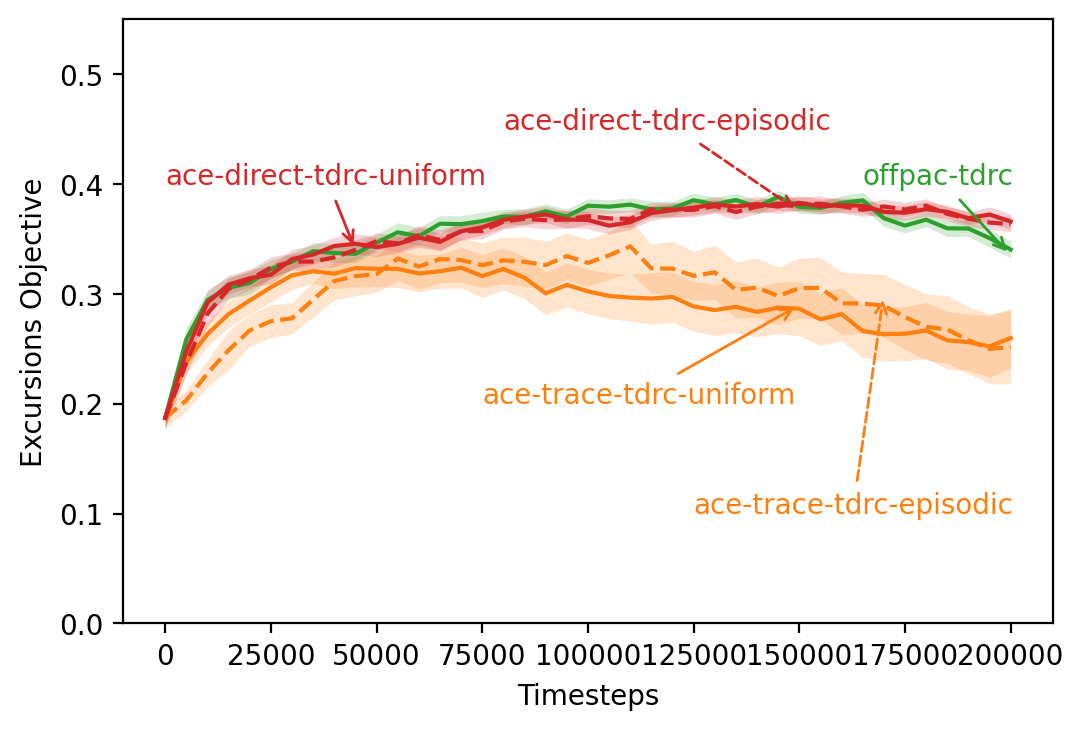}
    \caption{Excursions learning curves (TDRC)}
    \label{fig:vo_exc_tdrc}
  \end{subfigure}
  \hfill
  \begin{subfigure}[b]{0.45\textwidth}
    \centering
    \includegraphics[width=\textwidth]{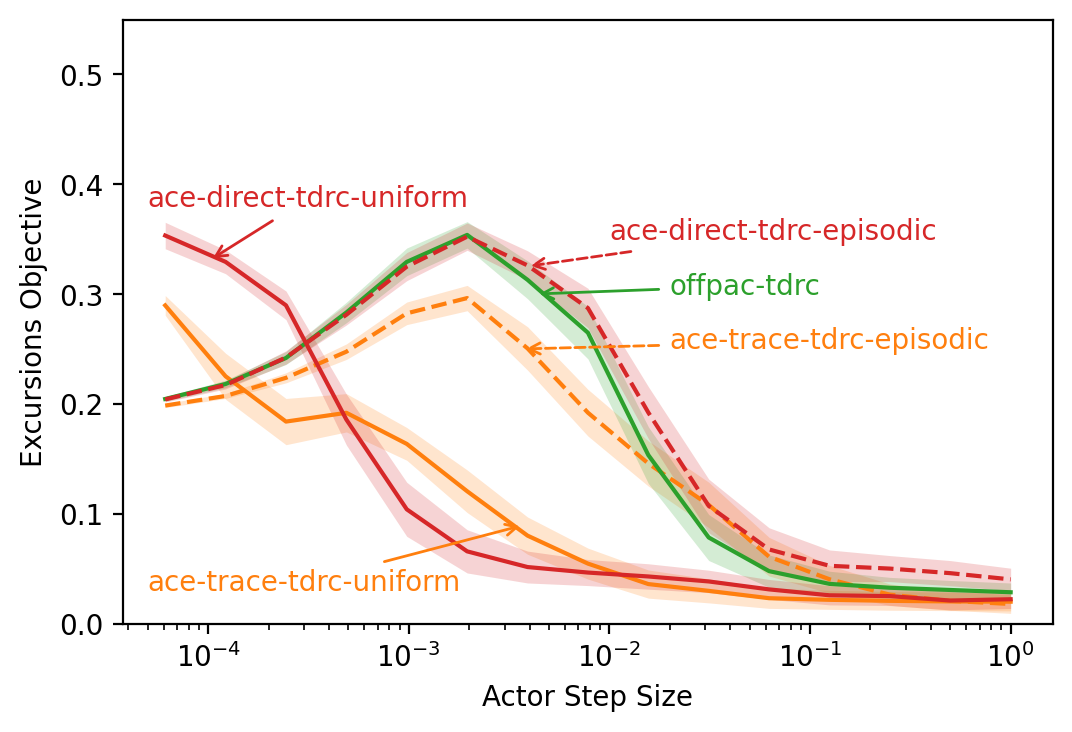}
    \caption{Excursions step-size sensitivity (TDRC)}
    \label{fig:vo_exc_tdrc_sens}
  \end{subfigure}
  \begin{subfigure}[b]{0.45\textwidth}
    \centering
    \includegraphics[width=\textwidth]{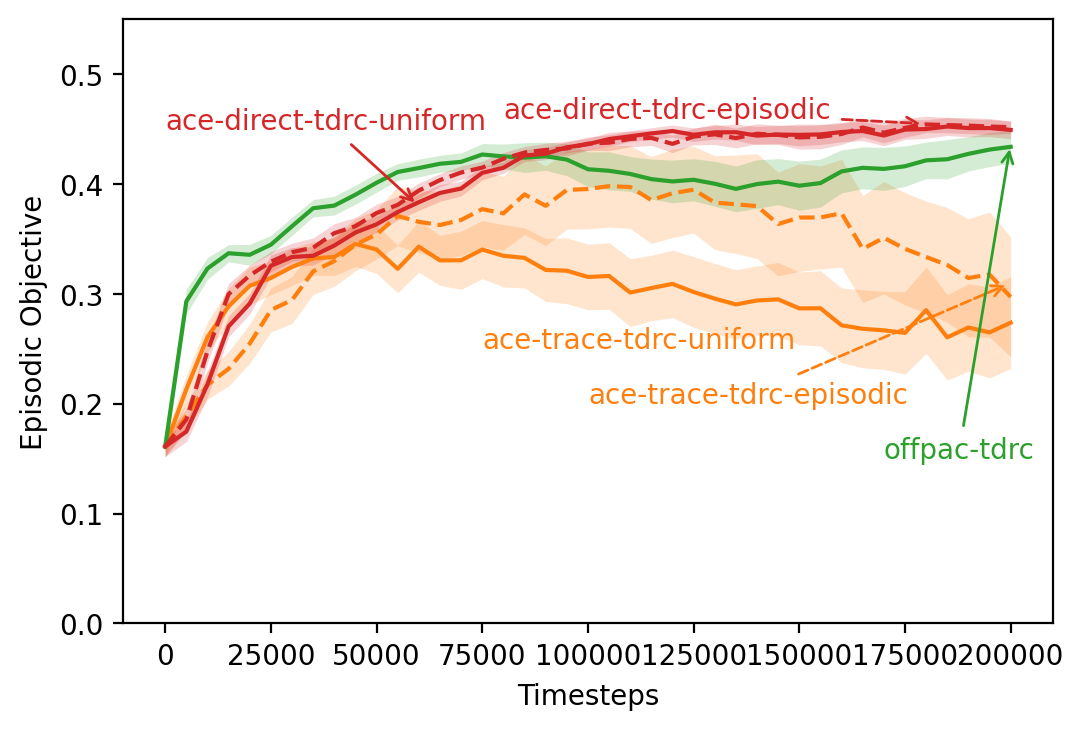}
    \caption{Episodic learning curves (TDRC)}
    \label{fig:vo_epi_tdrc}
  \end{subfigure}
  \hfill
  \begin{subfigure}[b]{0.45\textwidth}
    \centering
    \includegraphics[width=\textwidth]{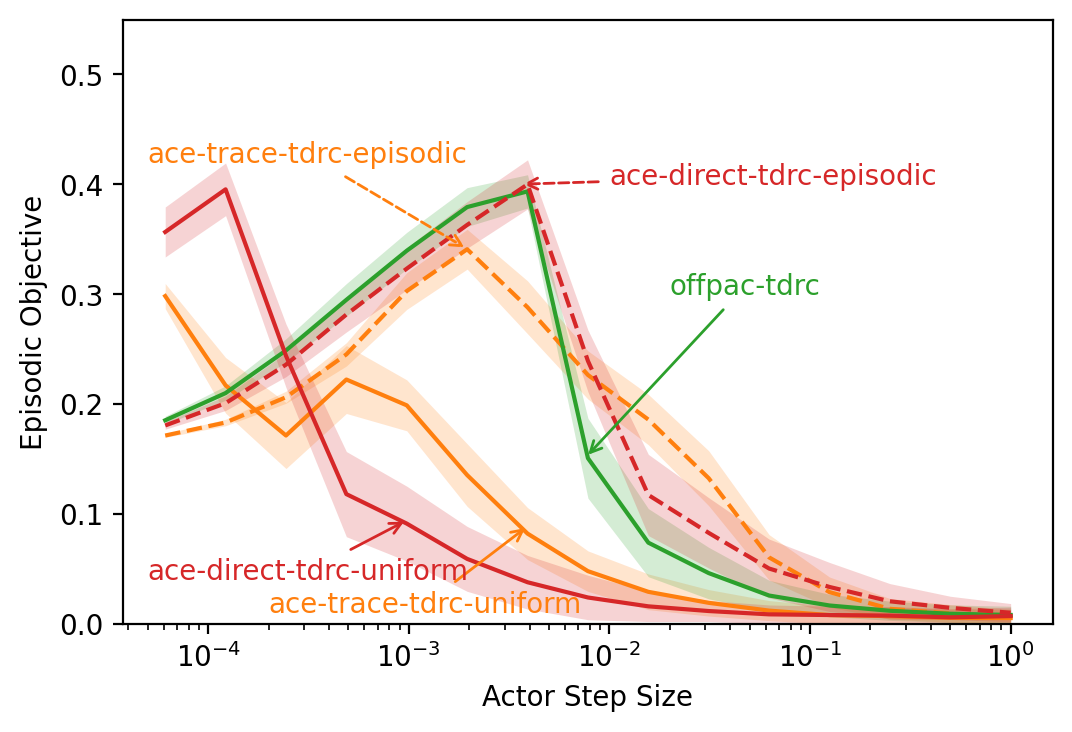}
    \caption{Episodic step-size sensitivity (TDRC)}
    \label{fig:vo_epi_tdrc_sens}
  \end{subfigure}
  \caption{Results on the Virtual Office. Shaded regions are 95\% confidence intervals.}
  \label{fig:vo}
\end{figure*}

Interestingly, algorithms that used the episodic interest function (dashed lines) performed better than or equal to their counterparts that used the uniform interest function (solid lines) for all combinations of objective function and critic, unlike the previous experiments on Mountain Car and Puddle World.
In addition, the algorithms that used the episodic interest function 
worked well for a larger range of step sizes, making it easier to find a value that performs well.
This could be due to the partially observable nature of the environment; many of the agent's observations are similar, which could cause the emphatic weightings with a uniform interest function to quickly become very large, necessitating a very small step size to compensate.

As in the previous experiments, the choice of critic played an important role in the performance of all the algorithms, although this time with mixed results.
Using a TDRC critic instead of an ETD critic improved the performance of OffPAC and ACE-direct with uniform interest, but had very little impact on ACE-trace with uniform interest and ACE-direct with episodic interest, and actually reduced the performance of ACE-trace with episodic interest (compare figures \ref{fig:vo_exc_etd} and \ref{fig:vo_exc_tdrc}, and \ref{fig:vo_epi_etd} and \ref{fig:vo_epi_tdrc}).
Similarly, using a TDRC critic had mixed effects on the sensitivity of each of the algorithms to the actor's step size.
It allowed a wider range of step sizes to perform well for some algorithms, but also led to a narrower range of well-performing step sizes for some algorithms (compare figures \ref{fig:vo_exc_etd_sens} and \ref{fig:vo_exc_tdrc_sens}, and \ref{fig:vo_epi_etd_sens} and \ref{fig:vo_epi_tdrc_sens}).

ACE-direct with the episodic interest function performed better than or equal to the other methods across most combinations of critic and objective function, with the possible exception of the episodic objective function.
With a TDRC critic, OffPAC learned more quickly than the other methods before eventually being surpassed by both ACE-direct variants.
However, OffPAC was able to close the gap near the end of the experiment, with the final learned policy of OffPAC performing comparably to the ACE-direct variants.
With an ETD critic, ACE-trace with episodic interest performed surprisingly well, with the final policy possibly outperforming ACE-direct with episodic interest.
However, the variance of the policy learned by ACE-trace was much larger than the other methods, rendering the difference in performance not statistically significant.
Conversely, ACE-direct with episodic interest learned a low-variance policy on each of the settings tested, performing as well as or better than the other algorithms.

\subsection{Summary}
In this section we investigated several key properties of the ACE algorithm empirically.
First we studied how the parameter $\eta$ trades off bias and variance in the counterexample from Section \ref{sec_counter} and a continuous action version, finding that even small values of $\eta$ (i.e., closer to semi-gradient updates) greatly improved the quality of the final learned policy, both when using the true value function and when using a learned critic.

Next we tested the ability of the emphatic trace to accurately estimate the emphatic weightings on an extended version of the counterexample, showing that while higher values of $\eta$ improved performance, a significant gap remained between ACE with $\eta = 1$ and ACE with the true emphatic weightings.

Then we moved to two classic control domains and tested several variants of ACE to determine how the choice of estimator for the emphatic weightings (and associated bias and variance properties) and the choice of critic affects the learning process.
We found that the emphatic trace's extreme variance greatly reduced the performance of ACE---with its bias playing somewhat less of a role---and replacing the emphatic trace with the direct method from Section \ref{sec_direct_f} resulted in ACE performing as well as or better than the other variants.
The TDRC critic was the better critic overall, but ACE with the direct method of estimating emphatic weightings was insensitive to the choice of critic and performed well regardless of whether ETD or TDRC was used.

Finally, we compared the performance of several variants of ACE on a new environment designed to highlight the weaknesses of each of the variants, finding that the episodic variant of ACE with the direct method of estimating emphatic weightings performed as well as or better than the other variants in almost all the situations tested, suggesting it should be the default method of estimating the emphatic weightings.

\section{Conclusion}

In this paper we introduced a generalized objective for learning policies and proved the off-policy policy gradient theorem using emphatic weightings.
Using this theorem, we derived an off-policy actor-critic algorithm that follows the gradient of the objective function, as opposed to previous methods like OffPAC and DPG that follow an approximate semi-gradient.
We designed a simple MDP to highlight that stationary points for semi-gradients can be highly suboptimal, with semi-gradient updates moving away from the optimal solution.
We also show, though, that the semi-gradient methods can produce reasonable solutions, either because of sufficiently rich policy parameterizations or because in some cases they can be seen as using the true gradient update with a different state weighting. 

We leverage these insights to improve the practicality of the method empirically. 
We design Actor-Critic with Emphatic Weightings (ACE) to have a tuneable parameter $\eta$ that sweeps between the full gradient (at $\eta = 1$) and the semi-gradient (at $\eta = 0$), often obtaining the best performance with a relatively small but non-zero $\eta$ (typically $\eta = 0.1$). We additionally reduce variance by directly estimating emphatic weightings with function approximation rather than using the emphatic trace. 
Empirically, the direct method of estimating the emphatic weightings performed better than the emphatic trace across all objective functions and environments tested. 
In addition, ACE with the direct method was less sensitive to the actor step size and choice of critic than OffPAC.
Overall, ACE with the direct method performed better than or equal to 
OffPAC in all situations, suggesting that incorporating (low-variance) state reweightings might be 
a reasonable direction for future policy gradient methods. 


\acks{We gratefully acknowledge funding from the Natural Sciences and Engineering Research Council of Canada (NSERC), the Canada CIFAR AI Chair program, and the Alberta Machine Intelligence Institute (Amii). This research was enabled in part by support provided by \href{https://www.scinethpc.ca/}{SciNet} and the \href{https://alliancecan.ca/}{Digital Research Alliance of Canada}.}

\appendix

\section{Proof of Deterministic Off-Policy Policy Gradient Theorem}\label{deterministic_pgt_appendix}

\subsection{Assumptions}\label{dpg_assumptions}

We make the following assumptions on the MDP:

\begin{assumption} \label{assumption:continuous}
	$\Pfcn(s'|s,a), r(s,a,s'), \gamma(s,a,s'), \pi(s;\vec\theta)$ and their derivatives are continuous in all variables $s,a,s',\vec\theta$.
\end{assumption}
\begin{assumption} \label{assumption:compact}
	$\States$ is a compact set in $\mathbb{R}^d$, and $\Actions$ is a compact set in $\mathbb{R}$.
\end{assumption}
\begin{assumption} \label{assumption:inverse}
	The policy $\pi$ and discount $\gamma$ are such that the inverse kernel of $\delta(s,s') - \Ppig(s,s')$ exists, where $\Ppig(s,s') = \int_\Actions \pi(a|s) \gamma(s,a,s') \Pfcn(s'|s,a) da$.
\end{assumption}
Under Assumption \ref{assumption:continuous}, $v_{\pi_{\vec\theta}}(s)$ and $\frac{\partial v_{\pi_{\vec\theta}}(s)}{\partial\vec\theta}$ are continuous functions of $\vec\theta$ and $s$. Together, Assumptions \ref{assumption:continuous} and \ref{assumption:compact} imply that $\left\lvert\left\lvert \frac{\partial v_{\pi_{\vec\theta}}(s)}{\partial\vec\theta} \right\rvert\right\rvert$, $\left\lvert\left\lvert \frac{\partial \pi(s;\vec\theta)}{\partial\vec\theta} \right\rvert\right\rvert$, and $\left\lvert\left\lvert \left. \frac{\partial q_{\pi_{\vec\theta}}(s,a)}{\partial a}\right\rvert_{a=\pi(s;\vec\theta)} \right\rvert\right\rvert$ are bounded functions of $s$, which allows us to switch the order of integration and differentiation, and the order of multiple integrations.

\subsection{Proof of Theorem \ref{grad_theorem_deterministic}}
\begin{proof}
  We start by deriving a recursive form for the gradient of the value function with respect to the policy parameters:
  \begin{align}
    \frac{\partial v_{\pi}(s)}{\partial\vec\theta}
    &= \frac{\partial}{\partial\vec\theta} q_{\pi}(s,\pi(s;\vec\theta))
    \notag{}
    \\
    &= \frac{\partial}{\partial\vec\theta} \int_{\States} \Pfcn(s'|s,\pi(s;\vec\theta)) \Big( r(s,\pi(s;\vec\theta),s') + \gamma(s,\pi(s;\vec\theta),s') v_\pi(s') \Big) \intd s'
    \notag{}
    \\
    &= \int_{\States} \frac{\partial}{\partial\vec\theta} \left( \Pfcn(s'|s,\pi(s;\vec\theta)) \Big( r(s,\pi(s;\vec\theta),s') + \gamma(s,\pi(s;\vec\theta),s') v_\pi(s') \Big) \right) \intd s'
    \label{deterministic_eqn_1},
  \end{align}
  where in Equation \ref{deterministic_eqn_1} we used the Leibniz integral rule to switch the order of integration and differentiation.
  We proceed with the derivation using the product rule:
  \begin{align}
    ={}& \int_{\States} \frac{\partial}{\partial\vec\theta} \Pfcn(s'|s,\pi(s;\vec\theta)) \Big( r(s,\pi(s;\vec\theta),s') + \gamma(s,\pi(s;\vec\theta),s') v_\pi(s') \Big) \nonumber
    \\
    &+ \Pfcn(s'|s,\pi(s;\vec\theta)) \frac{\partial}{\partial\vec\theta} \Big( r(s,\pi(s;\vec\theta),s') + \gamma(s,\pi(s;\vec\theta),s') v_\pi(s') \Big) \intd s' \nonumber
    \\
    ={}& \int_{\States} \frac{\partial \pi(s;\vec\theta)}{\partial\vec\theta} \left. \frac{\partial \Pfcn(s'|s,a)}{\partial a}\right\rvert_{a=\pi(s;\vec\theta)} \Big( r(s,\pi(s;\vec\theta),s') + \gamma(s,\pi(s;\vec\theta),s') v_\pi(s') \Big) \nonumber
    + \Pfcn(s'|s,\pi(s;\vec\theta)) \\&\qquad \left( \frac{\partial}{\partial\vec\theta} r(s,\pi(s;\vec\theta),s') + \frac{\partial}{\partial\vec\theta} \gamma(s,\pi(s;\vec\theta),s') v_\pi(s') + \gamma(s,\pi(s;\vec\theta),s') \frac{\partial}{\partial\vec\theta} v_\pi(s') \right) \intd s' \nonumber
    \\
    ={}& \int_{\States} \frac{\partial \pi(s;\vec\theta)}{\partial\vec\theta} \left. \frac{\partial \Pfcn(s'|s,a)}{\partial a}\right\rvert_{a=\pi(s;\vec\theta)} \Big( r(s,\pi(s;\vec\theta),s') + \gamma(s,\pi(s;\vec\theta),s') v_\pi(s') \Big) \nonumber
    \\
    &+ \Pfcn(s'|s,\pi(s;\vec\theta)) \left( \frac{\partial \pi(s;\vec\theta)}{\partial\vec\theta} 
    \left. \frac{\partial r(s,a,s')}{\partial a}\right\rvert_{a=\pi(s;\vec\theta)} + \frac{\partial \pi(s;\vec\theta)}{\partial\vec\theta} 
    \left. \frac{\partial \gamma(s,a,s')}{\partial a}\right\rvert_{a=\pi(s;\vec\theta)} v_\pi(s') \right) \intd s' \nonumber
    \\
    &+ \int_{\States} \Pfcn(s'|s,\pi(s;\vec\theta)) \gamma(s,\pi(s;\vec\theta),s') \frac{\partial}{\partial\vec\theta} v_\pi(s') \intd s' \nonumber
    \\
    ={}& \int_{\States} \frac{\partial \pi(s;\vec\theta)}{\partial\vec\theta} \biggr( \left. \frac{\partial \Pfcn(s'|s,a)}{\partial a}\right\rvert_{a=\pi(s;\vec\theta)} \Big( r(s,\pi(s;\vec\theta),s') + \gamma(s,\pi(s;\vec\theta),s') v_\pi(s') \Big) \nonumber
    \\
    &+ \Pfcn(s'|s,\pi(s;\vec\theta)) \left. \frac{\partial}{\partial a} \Big( r(s,a,s') + \gamma(s,a,s') v_\pi(s') \Big) \right\rvert_{a=\pi(s;\vec\theta)} \biggr) \intd s' \nonumber
    \\
    &+ \int_{\States} \Pfcn(s'|s,\pi(s;\vec\theta)) \gamma(s,\pi(s;\vec\theta),s') \frac{\partial}{\partial\vec\theta} v_\pi(s') \intd s' \nonumber
    \\
    ={}& \frac{\partial \pi(s;\vec\theta)}{\partial\vec\theta} \int_{\States} \left. \frac{\partial}{\partial a} \left( \Pfcn(s'|s,a) \Big( r(s,a,s') + \gamma(s,a,s') v_\pi(s') \Big) \right) \right\rvert_{a=\pi(s;\vec\theta)} \intd s' \nonumber
    \\
    &+
    \int_{\States} \Pfcn(s'|s,\pi(s;\vec\theta)) \gamma(s,\pi(s;\vec\theta),s') \frac{\partial v_\pi(s')}{\partial\vec\theta} \intd s' \nonumber
    \\
    ={}& \frac{\partial \pi(s;\vec\theta)}{\partial\vec\theta} \left. \frac{\partial q_\pi(s,a)}{\partial a}\right\rvert_{a=\pi(s;\vec\theta)} + \int_{\States} \Pfcn(s'|s,\pi(s;\vec\theta)) \gamma(s,\pi(s;\vec\theta),s') \frac{\partial v_\pi(s')}{\partial\vec\theta} \intd s' \label{dpg_recursive}.
  \end{align}
  For simplicity of notation, we will write Equation \ref{dpg_recursive} as
  \begin{align}\label{dpg_recursive_simplified}
    \frac{\partial v_{\pi}(s)}{\partial\vec\theta} = g(s) + \int_{\States} \Ppig(s,s') \frac{\partial v_{\pi}(s')}{\partial\vec\theta} \intd s',
  \end{align}
  since $\Pfcn(s'|s,\pi(s;\vec\theta)) \gamma(s,\pi(s;\vec\theta),s')$ is a function of $s$ and $s'$ for a fixed deterministic policy.
  Then we can write $\frac{\partial v_{\pi}(s)}{\partial\vec\theta}$ as an integral transform using the delta function:
  \begin{equation}\label{dpg_delta}
    \frac{\partial v_{\pi}(s)}{\partial\vec\theta} = \int_{\States} \delta(s,s') \frac{\partial v_{\pi}(s')}{\partial\vec\theta} \intd s'.
  \end{equation}
Plugging Equation \ref{dpg_delta} into the left-hand side of Equation \ref{dpg_recursive_simplified}, we obtain:
  \begin{align}
    \int_{\States} \delta(s,s') \frac{\partial v_{\pi}(s')}{\partial\vec\theta} \intd s' = g(s) + \int_{\States} \Ppig(s,s') \frac{\partial v_{\pi}(s')}{\partial\vec\theta} \intd s' \nonumber
    \\
    \implies \int_{\States} \Big( \delta(s,s') - \Ppig(s,s') \Big) \frac{\partial v_{\pi}(s')}{\partial\vec\theta} \intd s' = g(s) \nonumber
    \\
    \implies \frac{\partial v_{\pi}(s)}{\partial\vec\theta} = \int_{\States} k(s,s') g(s') \intd s' \label{dpg_inv_kernel},
  \end{align}
  where $k(s,s')$ is the inverse kernel of $\delta(s,s') - \Ppig(s,s')$.
  Now, using the continuous version of the weighted excursions objective defined in Equation \ref{exc_obj}, we have:
  \begin{align}
    \frac{\partial J_\mu(\vec\theta)}{\partial\vec\theta} ={}& \frac{\partial }{\partial\vec\theta} \int_{\States} d_\mu(s) i(s) v_\pi(s) \intd s
    \nonumber
    \\
    ={}& \int_{\States} \frac{\partial v_\pi(s)}{\partial\vec\theta} d_\mu(s) i(s) \intd s
    \nonumber
    \\
    ={}& \int_{\States} \int_{\States} k(s,s') g(s') \intd s' \, d_\mu(s) i(s) \intd s
    \nonumber
    \\
    ={}& \int_{\States} \int_{\States} k(s,s') d_\mu(s) i(s) \intd s \, g(s') \intd s'
    \label{dpg_fubini},
  \end{align}
  where in Equation \ref{dpg_fubini} we used Fubini's theorem to switch the order of integration.
  
  Now, we convert the recursive definition of emphatic weightings for deterministic policies over continuous state-action spaces into a non-recursive form. Recall the definition:
  \begin{equation} \label{cont_m}
    m(s') = d_\mu(s')i(s') + \int_{\States} \Ppig(s,s') m(s) \intd s.
  \end{equation}
  Again, we can write $\emweight(s')$ as an integral transform using the delta function:
  \begin{align} \label{m_delta}
    m(s') = \int_{\States} \delta(s, s') m(s) \intd s.
  \end{align}
  Plugging Equation \ref{m_delta} into the left-hand side of Equation \ref{cont_m}, we obtain:
  \begin{align}
    \int_{\States} \delta(s,s') m(s) \intd s = d_\mu(s')i(s') + \int_{\States} \Ppig(s,s') m(s) \intd s
    \nonumber
    \\
    \implies \int_{\States} (\delta(s,s') - \Ppig(s,s')) m(s) \intd s = d_\mu(s')i(s')
    \nonumber
    \\
    \implies m(s') = \int_{\States} k(s,s') d_\mu(s)i(s) \intd s
    \label{dpg_m_non_recursive},
  \end{align}
  where $k(s,s')$ is again the inverse kernel of $\delta(s,s') - \Ppig(s,s')$.
  Plugging Equation \ref{dpg_m_non_recursive} into Equation \ref{dpg_fubini} yields
  \begin{align*}
    \frac{\partial J_\mu(\vec\theta)}{\partial\vec\theta} &= \int_{\States} m(s') g(s') \intd s'
    \\
    &= \int_{\States} m(s) \frac{\partial \pi(s;\vec\theta)}{\partial\vec\theta} \left. \frac{\partial q_\pi(s,a)}{\partial a}\right\rvert_{a=\pi(s;\vec\theta)} \intd s.
  \end{align*}
  
  \end{proof}

\section{Continuous State Off-Policy Policy Gradient Theorem} \label{app_thm_pg_complete}

Given the Deterministic Off-Policy Policy Gradient Theorem in Appendix \ref{deterministic_pgt_appendix}, we can now provide a proof for the continuous-state version of the Off-Policy Policy Gradient Theorem provided in Theorem \ref{grad_theorem}.

\begin{proof}
The gradient of the objective in the continuous-states case is
\begin{equation}
\frac{\partial J_\mu(\pparams) }{\partial \pparams} = \frac{\partial \int_{\States} i(s) \vpi(s) \intd s}{\partial \pparams}  = \int_{\States} i(s) \frac{\partial  \vpi(s) }{\partial \pparams} \intd s \label{continuous_gradient}.
\end{equation}
Therefore, again to compute the gradient of $J_\mu$ we need to compute the gradient of the value function with respect to the policy parameters. A recursive form of the gradient of the value function can be derived, as we show below. Before starting, for simplicity of notation, we will again use
\begin{align*}
 \gvec(s) = \sum_a \frac{\partial \pi(s,a; \pparams) }{\partial \pparams}  \qpi(s,a),
\end{align*}
where $\gvec: \States \rightarrow \RR^\pdim$. Now let us compute the gradient of the value function:
\begin{align}
\frac{\partial  \vpi(s) }{\partial \pparams} &= \frac{\partial }{\partial \pparams} \sum_a  \pi(s,a; \pparams)  \qpi(s,a) \nonumber \\
&=  \sum_a \frac{\partial \pi(s,a; \pparams) }{\partial \pparams}  \qpi(s,a) + \sum_a \pi(s,a; \pparams)  \frac{\partial \qpi(s,a) }{\partial \pparams} \nonumber \\
&=  \gvec(s) + \sum_a \pi(s,a; \pparams)  \frac{\partial \int_{\States} \Pfcn(s'|s,a) (r(s,a,s') + \gamma(s,a,s') \vpi(s')) \intd s'}{\partial \pparams} \nonumber \\
&=  \gvec(s) + \sum_a \pi(s,a; \pparams) \int_{\States} \Pfcn(s'|s,a) \gamma(s,a,s') \frac{\partial  \vpi(s')}{\partial \pparams} \intd s' \nonumber \\
&= \gvec(s) + \int_{\States} \sum_a \pi(s,a; \pparams) \Pfcn(s'|s,a) \gamma(s,a,s') \frac{\partial  \vpi(s')}{\partial \pparams} \intd s' \label{continuous_recursive}.
\end{align}
For simplicity of notation we will write Equation \ref{continuous_recursive} as
\begin{align}
\frac{\partial  \vpi(s) }{\partial \pparams} &= \gvec(s) + \int_{\States} \Ppig(s,s') \frac{\partial  \vpi(s')}{\partial \pparams} \intd s' \label{continuous_recursive_simple},
\end{align}
where $\Ppig(s,s') = \sum_a \pi(s,a; \pparams) \Pfcn(s'|s,a) \gamma(s,a,s')$.
Again, as in the deterministic off-policy policy gradient theorem, we can write $\frac{\partial v_{\pi}(s)}{\partial\vec\theta}$ as an integral transform using the delta function:
  \begin{equation}\label{continuous_delta}
    \frac{\partial v_{\pi}(s)}{\partial\vec\theta} = \int_{\States} \delta(s,s') \frac{\partial v_{\pi}(s')}{\partial\vec\theta} \intd s'.
  \end{equation}
Plugging Equation \ref{continuous_delta} into the left-hand side of Equation \ref{continuous_recursive_simple}, we obtain:
  \begin{align}
    \int_{\States} \delta(s,s') \frac{\partial v_{\pi}(s')}{\partial\vec\theta} \intd s' = \gvec(s) + \int_{\States} \Ppig(s,s') \frac{\partial v_{\pi}(s')}{\partial\vec\theta} \intd s' \nonumber
    \\
    \implies \int_{\States} \Big( \delta(s,s') - \Ppig(s,s') \Big) \frac{\partial v_{\pi}(s')}{\partial\vec\theta} \intd s' = \gvec(s) \nonumber
    \\
    \implies \frac{\partial v_{\pi}(s)}{\partial\vec\theta} = \int_{\States} k(s,s') \gvec(s') \intd s' \label{continuous_inv_kernel}.
  \end{align}
Plugging this gradient from Equation \ref{continuous_inv_kernel} back into Equation \ref{continuous_gradient}, we obtain
\begin{align}
\frac{\partial J_\mu(\pparams) }{\partial \pparams} &= \int_{\States} i(s) \frac{\partial  \vpi(s) }{\partial \pparams} \intd s \nonumber \\
&= \int_{\States} i(s) \int_{\States} k(s,s') \gvec(s') \intd s' \intd s \label{continuous_fubini} \\
&= \int_{\States} \int_{\States} k(s,s') i(s) \intd s \text{ } \gvec(s') \intd s' \label{continuous_non_recursive_emphasis}\\
&= \int_{\States} m(s') \gvec(s') \intd s' \nonumber \\
&= \int_{\States} m(s) \gvec(s) \intd s \nonumber,
\end{align}
where in Equation \ref{continuous_fubini} we used Fubini's theorem to switch the order of integration, and in Equation \ref{continuous_non_recursive_emphasis} we use the non-recursive version of emphasis as shown in Equation \ref{dpg_m_non_recursive}.

\end{proof}

\section{Standard Approaches for Sampling the Gradient from a State}\label{app_sampling}

In this section we overview the standard approach to sample the gradient from a given state, that is used in both ACE and OffPAC. This section reiterates known results, but is included here for completeness. 
 
\subsection{Sampling the Gradient for a Given State}
For a given state, we can use the same strategies to sample the gradient as OffPAC, and many other actor-critic algorithms. We can use the log-likelihood approach, and incorporate baselines. For continuous actions, we can also use other strategies like the reparameterization trick. For concreteness, we explicitly explain how we sample this gradient here, using a value baseline and the log-likelihood approach, but emphasize that this can be replaced with other choices. 

Under finite states, the simplest approach to compute the gradient for a given state is to use what is sometimes called the \emph{all-actions} gradient
\begin{align*}
\sum_a \frac{\partial \pi(a|s; \pparams) }{\partial \pparams}  \qpi(s,a).
\end{align*}
If there are many actions or if the action space is continuous, this gradient is not an appropriate choice. Instead, we can obtain a sample estimate of this gradient using the log-likelihood trick: 
\begin{equation*}
 \sum_a \frac{\partial \pi(a|s; \pparams) }{\partial \pparams}  \qpi(s,a) 
 =  \sum_a \pi(a|s; \pparams) \frac{\partial \log \pi(a|s; \pparams) }{\partial \pparams}  \qpi(s,a),
\end{equation*}
which follows from the fact that the derivative of $\log y$ is $\frac{1}{y}$. Above we called this log likelihood form $\grads{s}$. 
Now we can use one action sampled according to $\pi$ to obtain an unbiased estimate of the all-actions gradient. In the off-policy setting, however, we do not use $\pi$ to select actions. So, additionally, we can incorporate importance sampling $\rho(s,a; \pparams) \defeq \frac{\pi(a|s;\pparams)}{\mu(a|s)}$ to adjust for the fact that the actions are taken according to behaviour policy $\mu$: 
\begin{align*}
 \sum_a \frac{\partial \pi(a|s; \pparams) }{\partial \pparams}  \qpi(s,a) 
 &=  \sum_a \mu(a|s) \frac{1}{ \mu(a|s)} \pi(a|s; \pparams) \frac{\partial \log \pi(a|s; \pparams) }{\partial \pparams}  \qpi(s,a)\\
 &=  \sum_a \mu(a|s) \rho(s,a; \pparams) \frac{\partial \log \pi(a|s; \pparams) }{\partial \pparams}  \qpi(s,a)
 .
\end{align*}
We can obtain an unbiased sample of the gradient by sampling an action $A \sim \mu(s, \cdot)$ and using $g(A) \defeq \rho(s,A; \pparams) \frac{\partial \log \pi(a|s; \pparams) }{\partial \pparams}  \qpi(s,A)$. 

The estimated gradient has more variance when we only use a sampled action rather than all actions. The standard strategy to reduce this variance is to subtract a
baseline. The baseline is a function of state, such as an estimate of the value function $v(s)$, with modified update 
\begin{align*}
\rho(s,a; \pparams) \frac{\partial \log \pi(a|s; \pparams) }{\partial \pparams}  [\qpi(s,a) - v(s)]
 .
\end{align*}
It is straightforward to show that the expected value of this update is the same, i.e. that
\begin{align*}
\sum_a \mu(a | s) \rho(s,a; \pparams) \frac{\partial \log \pi(a|s; \pparams) }{\partial \pparams}  [\qpi(s,a) - v(s)] = \sum_a \mu(a | s) \rho(s,a; \pparams) \frac{\partial \log \pi(a|s; \pparams) }{\partial \pparams}  \qpi(s,a)
 .
\end{align*}
Note that this is not the minimal variance baseline \citep{dick2015policy}, but it nonetheless is one of the most commonly chosen in practice due to its efficacy and simplicity. 

\subsection{Estimating the Advantage}
 
In practice, it is difficult to obtain $\qpi(s,a)$ to compute this gradient exactly. Instead, we obtain (a likely biased) sample of the advantage, $\delta_t \approx \qpi(S_t,A_t) - v(s_t)$. An unbiased but high variance estimate uses a sample of the return $G_t$ from $(s_t,a_t)$, and uses $\delta_t = G_t - v(s_t)$. Then $\mathbb{E}[\delta_t | S_t = s, A_t = a] = q_\pi(s,a) - v(s)$. On the other extreme, we can directly approximate the action-values, $\hat q(s,a)$. In-between, we can use $n$-step returns, or $\lambda$-returns, to balance between using observed rewards and value estimates. 
 
 The simplest choice is to use a 1-step return: $r(s,a) + \gamma(s,a,S_{t+1}) v(S_{t+1})$ is used as an approximation of $\qpi(s,a)$. The value function $v$ is typically called the critic, and plays both the role of estimating the return as well as that of a baseline. The advantage is set to $\delta_t = R_{t+1} + \gamma_{t+1} v(S_{t+1}) - v(S_t)$, which is simply the TD error. 
 
 More generally, however, we can also use multi-step returns. For an $n$-step return, with data sampled on-policy, with $\gamma_{t+1:t+n} \defeq \prod_{j=1}^{n} \gamma_{t+j}$ as short-hand for the product of discounts, we would use 
 \begin{equation*}
 \delta_t = R_{t+1}  + \gamma_{t+1} R_{t+2} + \ldots \gamma_{t+1:t+n-1}R_{t+n} +  \gamma_{t+1:t+n}  v(S_{t+n}) - v(S_t)
 .
 \end{equation*}
  With off-policy data, we would need to incorporate importance sampling ratios 
  \begin{equation*}
  \delta_t = R_{t+1}  + \rho_{t+1}\gamma_{t+1} R_{t+2} + \ldots \rho_{t+1:t+n-1}\gamma_{t+1:t+n-1} R_{t+n} +  \rho_{t+1:t+n}\gamma_{t+1:t+n}v(S_{t+n}) - v(S_t),
  \end{equation*}
 where $\rho_{t+1:t+n} \defeq \prod_{j=1}^{n} \rho_{t+j}$.
 For sufficiently large $n$, we recover the sampled return and so obtain an unbiased---but likely high-variance---estimate of the advantage. An interim $n$ likely provides a reasonable choice between bias and variance, between $n = 0$ with a direct estimation of $q_\pi(s,a)$, and large $n$ that gives an unbiased sample of the return.
 
 We can also average across multiple $n$, and obtain $\lambda$-returns \citep{sutton2018reinforcement}. For larger $\lambda$, more weight is put on the $n$-step returns with larger $n$, and for $\lambda = 1$ we once again recover an unbiased sample of the return. Such an approach was used in OffPAC, which requires storing traces of gradients of policy parameters (see Algorithm 1 in \citealp{degris2012offpolicy}). Note that this $\lambda$-return is used for the policy update. The update to the value estimate (critic) itself can use a different $\lambda$, as it stores a separate trace for the gradient of the value function. The update with eligibility traces is only sound under linear function approximation for the values and linear function approximation for action-preferences. In this work, therefore, we opt for the simplest choice: $n=1$ with $\delta_t = R_{t+1} + \gamma_{t+1} v(S_{t+1}) - v(S_t)$. 
 
 Finally, we need to estimate the values themselves. Any value approximation algorithm can be used to estimate $v$, such as TD, gradient TD or even emphatic TD. Given we already compute the emphasis weighting, it is straightforward to use emphatic TD. We investigate both a gradient TD critic and an emphatic TD critic in the experiments.

\section{Proof of Theorem \ref{ent_grad_theorem}}\label{ent_grad_appendix}

\subsection{Entropy-regularized OPPG theorem with discrete states and actions}
\begin{proof}
The proof for this theorem is mostly similar to the one for Theorem \ref{grad_theorem}. First, since the interest function does not depend on the parameters,
\begin{equation*}
\frac{\partial \tilde{J}_\mu(\pparams) }{\partial \pparams} = \frac{\partial \sum_{s \in \States} i(s) \vtpi(s) }{\partial \pparams} = \sum_{s \in \States} i(s) \frac{\partial  \vtpi(s) }{\partial \pparams}.
\end{equation*}
Recall that $\vtpi(s) = \sum_a \pi(a|s; \pparams) \qtpi(s,a) + \tau \entropy( \pi(\cdot|s;\pparams))$. 
We define the following notation which helps with the recursive expression of the gradient:
\begin{align*}
 \gtvec(s) = \sum_a \frac{\partial \pi(a|s; \pparams) }{\partial \pparams}  \qtpi(s,a) + \tau \frac{\partial }{\partial \pparams}  \entropy( \pi(\cdot|s;\pparams)),
\end{align*}
where $\gtvec: \States \rightarrow \RR^\pdim$. 
The gradient of the entropy regularized value function is then
\begin{align}
\frac{\partial  \vtpi(s) }{\partial \pparams} &= \frac{\partial }{\partial \pparams} \left[ \sum_a  \pi(a|s; \pparams) \qtpi(s,a) + \tau \entropy( \pi(\cdot|s;\pparams))\right]\nonumber \\
=&  \sum_a \frac{\partial \pi(a|s; \pparams) }{\partial \pparams}  \qtpi(s,a) + \sum_a \pi(a|s; \pparams)  \frac{\partial \qtpi(s,a) }{\partial \pparams} + \tau \frac{\partial }{\partial \pparams}  \entropy( \pi(\cdot|s;\pparams)) \label{ent_eq_product_rule} \\
=&  \gtvec(s) + \sum_a \pi(a|s; \pparams)  \frac{\partial \sum_{s'} p(s'|s,a) (r(s,a,s') + \gamma(s,a,s') \vtpi(s'))}{\partial \pparams}  \nonumber \\
=&  \gtvec(s) + \sum_a \pi(a|s; \pparams) \sum_{s'} p(s'|s,a) \gamma(s,a,s') \frac{\partial  \vtpi(s')}{\partial \pparams}  \nonumber.
\end{align}
So the gradient of the regularized value function can be written in vector form similar to the unregularized version. Let $\vtpivecdot \in \RR^{|\States| \times \pdim}$ be the matrix of gradients of $\vtpi$ for each state $s$; and $\Gtmat \in \RR^{|\States| \times \pdim}$ the matrix where each row corresponds to state $s$ in the vector $\gtvec(s)$. Then
\begin{align}
\vtpivecdot  &=  \Gtmat + \Ppig \vtpivecdot \label{ent_eq_bellman_gradient}\\
\implies \vtpivecdot &= (\eye - \Ppig)^\inv \Gtmat \nonumber.
\end{align}
Therefore, we obtain
\begin{align*}
\sum_{s \in \States} i(s) \frac{\partial  \vtpi(s) }{\partial \pparams} 
&= \ivec^\top \vtpivecdot
=  \ivec^\top (\eye - \Ppig)^\inv \Gtmat
= \emvec^\top \Gtmat\\
&= \sum_{s \in \States} \emweight(s) \left[\sum_a \frac{\partial \pi(a|s; \pparams) }{\partial \pparams}  \qtpi(s,a) + \tau \frac{\partial }{\partial \pparams}  \entropy( \pi(\cdot|s;\pparams))\right]
.
\end{align*}
\par
We can further simplify this gradient by explicitly computing the gradient of the negative of the entropy:
\begin{align*}
\frac{\partial \text{-}\entropy(\pi(\cdot |s;\pparams)) }{\partial \pparams} 
&= \sum_a  \frac{\partial \pi(a|s; \pparams)}{\partial \pparams}  \log \pi(a|s;\pparams) + \sum_a \pi(a|s; \pparams) \frac{\partial \log \pi(a|s;\pparams) }{\partial \pparams}
\\
&= \sum_a  \frac{\partial \pi(a|s; \pparams)}{\partial \pparams}  \log \pi(a|s;\pparams) + \sum_a \frac{\partial \pi(a|s; \pparams)}{\partial \pparams}
\\
&= \sum_a  \frac{\partial \pi(a|s; \pparams)}{\partial \pparams}  \log \pi(a|s;\pparams) ,
\end{align*}
where the last line follows from the fact that $\sum_a \frac{\partial \pi(a|s; \pparams)}{\partial \pparams} = \frac{\partial}{\partial \pparams}\sum_a \pi(a|s; \pparams) =  \frac{\partial}{\partial \pparams} 1 = \zerovec$.
\end{proof}

This form looks a bit different than the (on-policy) policy gradient update given by \citep[Theorem 6]{geist2019theory}. We can actually get a similar form by rewriting the inner term above. For a specific $s$, using chain rule on the entropy we obtain
\begin{align*}
\sum_a \frac{\partial \pi(a|s; \pparams) }{\partial \pparams}  \qtpi(s,a) + \tau \frac{\partial }{\partial \pparams}  \entropy( \pi(\cdot|s;\pparams))
&= 
\sum_a \frac{\partial \pi(a|s; \pparams) }{\partial \pparams}  \qtpi(s,a) + \tau \frac{\partial \pi(a|s; \pparams) }{\partial \pparams} \frac{\partial \entropy( \pi(\cdot|s;\pparams))}{\partial \pi(a|s; \pparams)} \\
&= 
\sum_a \frac{\partial \pi(a|s; \pparams) }{\partial \pparams}  \left[\qtpi(s,a) + \tau  \frac{\partial \entropy( \pi(\cdot|s;\pparams))}{\partial \pi(a|s; \pparams)} \right]
.
\end{align*}

\subsection{Regularized OPPG theorem with continuous states and actions}
\begin{theorem}
\begin{align*}
\frac{\partial \tilde{J}_\mu(\pparams) }{\partial \pparams} 
&= \int_{\SS} m(s) \left[ \int_{\AA} \frac{\partial \pi(a|s;\pparams)}{\partial\pparams} \tilde{q}_\pi(s,a) \intd a + \tau \frac{\partial }{\partial\pparams} \Omega(\pi(\cdot|s;\pparams)) \right] \intd s
\end{align*}
\end{theorem}

\begin{proof}
We extend the result from the previous section to the continuous state and action setting using a similar strategy to the one we used in Appendix \ref{app_thm_pg_complete} to extend Theorem \ref{grad_theorem} to the continuous state setting.
In addition, we opt for a generic regularizer $\Omega(\pi(\cdot|s;\pparams)) : \Delta(\AA) \to \mathbb{R}$,
which slightly modifies the definition of the regularized value functions to be
\begin{align*}
  \vtpi(s) &= \int_\AA \pi(a|s; \pparams) \qtpi(s,a) \intd a + \tau \Omega(\pi(\cdot|s;\pparams)) \qquad \forall s \in \SS
  \\
  \qtpi(s,a) &= \int_{\SS} P(s'|s,a) \big[r(s,a,s') + \gamma(s,a,s') \vtpi(s') \big] \intd s' \qquad \forall s \in \SS, \forall a \in \AA
  .
\end{align*}
We use the following notation to simplify the derivation:
\begin{align*}
  \gtvec(s) &= \int_{\AA} \frac{\partial \pi(a|s;\pparams)}{\partial\pparams} \tilde{q}_\pi(s,a) \intd a + \tau \frac{\partial }{\partial\pparams} \Omega(\pi(\cdot|s;\pparams))
  .
\end{align*}
Starting with the continuous objective function from equation (\ref{exc_obj}) and using the Leibniz integral rule to differentiate under the integral (which applies because we assumed the state space is a compact set in Assumption \ref{assumption:compact} from Appendix \ref{dpg_assumptions}):
\begin{align}
  \frac{\partial J_{\mu}(\pparams)}{\partial \pparams}
  = \frac{\partial }{ \partial \pparams} \int_{\SS} i(s) \tilde{v}_\pi(s) \intd s
  = \int_{\SS} i(s) \frac{ \partial \tilde{v}_\pi(s) }{ \partial \pparams } \intd s \label{eq:exc_obj_cont_s_cont_a}.
\end{align}
Next, we find a recursive expression for the gradient of the regularized value function:
\begin{align*}
\frac{\partial  \vtpi(s) }{\partial \pparams} 
&= \frac{\partial }{\partial \pparams} \int_\AA \pi(a|s; \pparams) \qtpi(s,a) \intd a + \tau \Omega(\pi(\cdot|s;\pparams))
\\
&= \int_\AA \frac{\partial \pi(a|s; \pparams)}{\partial \pparams} \qtpi(s,a) \intd a + \int_\AA \pi(a|s; \pparams) \frac{\partial \qtpi(s,a)}{\partial \pparams} \intd a + \tau \frac{\partial }{\partial \pparams} \Omega(\pi(\cdot|s;\pparams))
\\
&= \gtvec(s) + \int_\AA \pi(a|s; \pparams) \frac{\partial \qtpi(s,a)}{\partial \pparams} \intd a
\\
&= \gtvec(s) + \int_\AA \pi(a|s; \pparams) \frac{\partial }{\partial \pparams} \int_{\SS} P(s'|s,a) \big[r(s,a,s') + \gamma(s,a,s') \vtpi(s') \big] \intd s' \intd a
\\
&= \gtvec(s) + \int_\AA \pi(a|s; \pparams) \int_{\SS} P(s'|s,a) \gamma(s,a,s') \frac{\partial \vtpi(s')}{\partial \pparams} \intd s' \intd a
\\
&= \gtvec(s) + \int_{\SS} \int_\AA \pi(a|s; \pparams) P(s'|s,a) \gamma(s,a,s') \intd a \frac{\partial \vtpi(s')}{\partial \pparams} \intd s',
\end{align*}
where in the last step we used Fubini's Theorem to switch the order of integrals (which we can do because the absolute value of the integral is finite due to Assumptions \ref{assumption:continuous} and \ref{assumption:compact}).
Again we will use the shorthand $\Ppig(s,s') = \int_\AA \pi(a|s; \pparams) P(s'|s,a) \gamma(s,a,s') \intd a$ to simplify the derivation.
Next we write $\frac{\partial \vtpi(s)}{\partial \pparams}$ as the integral transform
\begin{align*}
\frac{\partial \vtpi(s)}{\partial \pparams} &= \int_\SS \delta(s,s') \frac{\partial \vtpi(s')}{\partial \pparams} \intd s'
,
\end{align*}
plug it into the previous equation, and solve to get
\begin{align*}
\int_\SS \delta(s,s') \frac{\partial \vtpi(s')}{\partial \pparams} \intd s' &= \gtvec(s) + \int_{\SS} \Ppig(s,s') \frac{\partial \vtpi(s')}{\partial \pparams} \intd s'
\\
\int_\SS \left(\delta(s,s') - \Ppig(s,s')\right) \frac{\partial \vtpi(s')}{\partial \pparams} \intd s' &= \gtvec(s)
\\
\implies \frac{\partial \vtpi(s)}{\partial \pparams} &= \int_\SS k(s,s') \gtvec(s') \intd s',
\end{align*}
where $k(s,s')$ is the inverse kernel of $\delta(s,s') - \Ppig(s,s')$.
Then plugging the above into Equation (\ref{eq:exc_obj_cont_s_cont_a}), interchanging integrals using Fubini's theorem, and simplifying gives
\begin{align*}
\frac{\partial J_{\mu}(\pparams)}{\partial \pparams}
&= \int_{\SS} i(s) \int_\SS k(s,s') \gtvec(s') \intd s' \intd s
\\
&= \int_{\SS} \int_\SS i(s) k(s,s') \intd s \gtvec(s') \intd s'
\\
&= \int_{\SS} m(s') \gtvec(s') \intd s'
\\
&= \int_{\SS} m(s) \gtvec(s) \intd s
\\
&= \int_{\SS} m(s) \left[ \int_{\AA} \frac{\partial \pi(a|s;\pparams)}{\partial\pparams} \tilde{q}_\pi(s,a) \intd a + \tau \frac{\partial }{\partial\pparams} \Omega(\pi(\cdot|s;\pparams)) \right] \intd s.
\end{align*}

\end{proof}

\section{Proof of Proposition \ref{prop_suboptimality}}\label{suboptimality_appendix}

This section proves Proposition \ref{prop_suboptimality}. First, a result by \cite{mei2020global} for tabular domains is extended to state aggregation and transition dependent discounting (which includes the three-state counterexample) to show that entropy-regularized policy gradient converges to a point where the gradient is zero. Then, it is shown that such point is not a stationary point for semi-gradient updates.

\newcommand{\rep}{\text{alias-rep}}
\newcommand{\repset}{\mathcal{S}_{\text{rep}}}

\subsection{Entropy-regularized PG under State Aggregation}

We assume the policy is a softmax policy and additionally specifically characterize the gradient under state aggregation. This specific characterization facilitates showing that the solution to the objective lies on the interior of the simplex, and so that a stationary point exists. 

We define $\alias(s)$ as the set of state that share their representation with $s$ including $s$ itself. We additionally defined $\repset \subset \States$ the set of representative states, one for each bin in the aggregation.
For example, in the three-state counterexample $\alias(s_0) = \{s_0\}$ and $\alias(s_1) = \alias(s_2) = \{s_1, s_2\}$. We simply need to choose one state in each aliased set, giving $\repset = \{s_0, s_1\}$. 

For a parameter set $\pparams$, the policy is a softmax transform defined as follows. For a state $s'$ with representative state $s$, i.e., $s' \in \alias(s)$, we have
\begin{equation}
  \pi(s',a; \pparams) = \frac{\exp\bigl(\pparams(s,a)\bigr)}{\sum_{a'} \exp\bigl(\pparams(s,a')\bigr)} \label{eq_softmax}.
\end{equation}
Using again the three-state counterexample, $\pparams$ has four components: $\pparams(s_0,a_0)$ and $\pparams(s_0,a_1)$ specify the policy for $s_0$, and $\pparams(s_1,a_0)$ and $\pparams(s_1,a_1)$ specify the policy for both $s_1$ and $s_2$ since these two states are aliased. 

The softmax policy has a simple well-known gradient, which we can explicitly write for state-aggregation in the following lemma for easy reference in later proofs.
\begin{lemma} \label{lemma_H}
Assume the policy uses a softmax distribution, as in Equation \ref{eq_softmax}.
For any state $s \in \repset$, for any $s' \in \alias(s)$, 
\begin{equation}
\frac{\partial \pi(s',a'; \pparams)}{\partial \pparams(s,a)} = \left\{ \begin{array}{cc}
\pi(s',a; \pparams) [1 - \pi(s',a; \pparams)] & \mbox{ if $a' = a$} \\
-\pi(s',a; \pparams) \pi(s',a'; \pparams) &  \mbox{ if $a' \neq a$} 
\end{array} \right. \label{eq_grad_softmax}
\end{equation}
and, across all actions, this can be compactly written as
\begin{equation*}
\frac{\partial \pi(s',a'; \pparams)}{\partial \pparams(s,\cdot)} = \diag(\pi(s,.; \pparams)) - \pi(s,.; \pparams)\pi(s,.; \pparams)^\top.
\end{equation*}
\end{lemma}
\begin{proof}
For explicit step to obtain this derivation, notice first that
\begin{align*}
\log \pi(s',a; \pparams) = \pparams(s,a)- \log \sum_{a'} \exp(\pparams(s,a'))
\end{align*}
and so 
\begin{align*}
\frac{\partial \pi(s',a'; \pparams)}{\partial \pparams(s,a)} 
&= \pi(s',a'; \pparams) \frac{\partial \log \pi(s',a; \pparams)}{\partial \pparams(s,a)} \\
&= \pi(s',a'; \pparams) \frac{\partial [\pparams(s,a')- \log \sum_{a''} \exp(\pparams(s,a''))]}{\partial \pparams(s,a)} \\
&= \pi(s',a'; \pparams) \left[\frac{\partial \pparams(s,a')}{\partial \pparams(s,a)} -  \frac{\partial \log \sum_{a''} \exp(\pparams(s,a''))}{\partial \pparams(s,a)}\right] .
\end{align*}
If $a' = a$, then 
\begin{align*}
\frac{\partial \pparams(s,a')}{\partial \pparams(s,a)} -  \frac{\partial \log \sum_{a''} \exp(\pparams(s,a''))}{\partial \pparams(s,a)}
&=  1 -  \frac{1}{\sum_{a''} \exp(\pparams(s,a'')} \frac{\partial \sum_{a''} \exp(\pparams(s,a''))}{\partial \pparams(s,a)} \\
&= 1 -  \frac{1}{\sum_{a''} \exp(\pparams(s,a''))} \exp(\pparams(s,a) \\
&= 1 - \pi(a|s; \pparams)\\
&= 1 - \pi(s',a; \pparams) ,
\end{align*}
where the last step follows from the fact that $\pi(s',a; \pparams) = \pi(a|s; \pparams)$ under state aggregation.  
If $a' \neq a$, then 
\begin{align*}
\frac{\partial \pparams(s,a')}{\partial \pparams(s,a)} -  \frac{\partial \log \sum_{a''} \exp(\pparams(s,a''))}{\partial \pparams(s,a')}
&= 0 - \pi(a|s; \pparams) = - \pi(s',a; \pparams)
.
\end{align*}
\end{proof}

The gradient ascent update using the entropy-regularized objective is
\begin{equation}
\pparams_{t+1} = \pparams_t + \eta \frac{\partial \tilde{J}_\mu(\pparams_t) }{\partial \pparams} \label{eq_pg_update}
\end{equation}
for stepsize $\eta > 0$. 
For an appropriately small stepsize $\eta$, this will converge to a stationary point if one exists, and otherwise move to a point on the simplex, where $\pparams_t(s,a) \rightarrow \infty$ for an action $a$ in state $s$. The standard condition on $\eta$ is that $\eta < 1/L$ for Lipschitz constant $L$ of $\tilde{J}_\mu$, which is a common assumption for policy gradient methods. In the next section, we show that this update converges to a stationary on the interior of the policy simplex. 

We provide one more result, that is useful for this characterization.
Lemma \ref{lemma_breakdown} breaks down the gradient into its components. \cite{mei2020global}, in their Lemma 10, proved this result for entropy regularized policy gradients, assuming tabular policies in the on-policy setting. We extend it to state aggregation, with our off-policy gradient. Their presentation has an extra term $- \tau \log \pi(a|s;\pparams)$. This discrepancy is just a difference in notation. In our formulation, entropy regularized action values are defined to contain the entropy term so that the connection between the regularized and unregularized policy gradient is clearer.
\begin{lemma}\label{lemma_breakdown}
For each state $s \in \repset$ and action $a$
\begin{align*}
\frac{\partial \tilde{J}_\mu(\pparams) }{\partial \pparams(s,a)} &= \sum_{s'\in\alias(s)}\emweight(s') \pi(s',a;\pparams)  \biggl[\qtpi(s',a) - \vtpi(s') \biggr].
\end{align*}
\end{lemma}
\begin{proof}
According to Theorem \ref{ent_grad_theorem}
\begin{align*}
\frac{\partial \tilde{J}_\mu(\pparams) }{\partial \pparams} = \sum_{s \in \States} \emweight(s) \sum_a \frac{\partial \pi(a|s; \pparams) }{\partial \pparams}  \qtpi(s,a).
\end{align*}
Notice that only for states $s'\in \alias(s)$, and any $a'$, we have $\frac{\partial \pi(s',a'; \pparams) }{\partial \pparams(s,a)} \neq 0$. We can write partial derivatives w.r.t. the parameters of each state and action, using Equation \ref{eq_grad_softmax}:
\begin{align*}
\frac{\partial \tilde{J}_\mu(\pparams) }{\partial \pparams(s,a)} 
=& \sum_{s' \in \alias(s)} \emweight(s') \sum_{a'} \frac{\partial \pi(s',a'; \pparams) }{\partial \pparams(s,a)}  \qtpi(s',a') \\
=& \sum_{s' \in \alias(s)} \emweight(s') \Bigg[ \pi(s',a; \pparams) (1-\pi(s',a; \pparams)) \qtpi(s',a) 
\\
&\qquad - \sum_{a' \neq a} \pi(s',a; \pparams) \pi(s',a'; \pparams) \qtpi(s',a') \Bigg]\\
=& \sum_{s' \in \alias(s)} \emweight(s') \pi(s',a; \pparams) \left[ \qtpi(s',a)  - \sum_{a'} \pi(s',a'; \pparams) \qtpi(s',a') \right] \\
=& \sum_{s' \in \alias(s)} \emweight(s') \pi(s',a; \pparams) \left[  \qtpi(s',a)  - \vtpi(s')  \right].
\end{align*}
\end{proof}

\subsection{Existence of stationary points for the entropy-regularized PG objective}

We are now ready to prove that a stationary point exists for $\tau>0$. \cite{mei2020global} (Lemma 17) proved that, under tabular representation and softmax parameterization, entropy regularized policy gradient updates converge to a point where the gradient is zero. We extend their proof to state aggregation and transition-dependent discounting. We additionally avoid assuming that the rewards are non-negative.

\begin{assumption} [Lipschitz continuity] \label{lipsc_assu}
$\tilde{J}_\mu$ is Lipschitz continuous with Lipschitz constant $L$. 
\end{assumption}

\begin{assumption} [Bounded reward] \label{rew_assu}
$|r(s,a)| \leq R_{max}$.
\end{assumption}

\begin{assumption} [Bounded expected sum of future discounting] \label{disc_assu} There exists $C \in \mathbb{R}$ such that for any $s \in \States$, 
$\E_{\pi}[\gamma_{t+1} + \gamma_{t+1} \gamma_{t+2} + \ldots  | S_t = s] \leq C$. This assumption is satisfied if $\pi$ is proper, or if $\gamma_{t+i} < 1$ after a bounded number of steps $i \in \mathbb{N}$.
\end{assumption}

\newcommand{\numactions}{n_a}
\newcommand{\numstates}{n_s}

We first prove that the state values are bounded and show that the action-values are lower bounded by the scaled log of the probabilities. We use this to show the main result in Proposition \ref{prop_stationary}.
\begin{lemma}\label{lem_vals_bounded}
Under Assumptions \ref{rew_assu} and \ref{disc_assu}, for any given policy parameters $\pparams$, 
\begin{align*}
\vtpi(s) &\leq C R_{max} + C \tau \log n_a \\
\qtpi(s,a) - \vtpi(s) &\geq  - \tau \log \pi(a|s;\pparams) - (2C+1) R_{max} - C \tau \log n_a 
.
\end{align*}
\end{lemma}
\begin{proof} 
First notice that for any distribution $p$ over actions
\begin{align*}
\text{Entropy}(p) \defeq
-\sum_{a\in \Actions} p(a) \log p(a) 
\leq  -\sum_{a\in \Actions} \frac{1}{\numactions} \log \left(\frac{1}{\numactions}\right)
=  \log \numactions,
\end{align*}
where $\numactions$ is the number of actions.
The inequality follows from the fact that the uniform distribution has the highest entropy. 
We use this inequality to bound $\E_{\pi}[-\log \pi(A_t|S_t; \pparams) | S_t = s] = -\sum_a \pi(a|s; \pparams) \log \pi(a|s; \pparams)$, which is the entropy of $\pi(s,\cdot; \pparams)$. 
Using this bound on entropy, and the facts that the entropy is nonnegative and $\gamma_{t+i} \le 1$ for all $i \in \mathbb{N}$, we obtain
\begin{align*}
\vtpi(s) 
&= \E_{\pi}[G_{t}  | S_t = s] - \tau \E_{\pi}[\log \pi(A_t|S_t; \pparams) + \gamma_{t+1} \log \pi(S_{t+1},A_{t+1}; \pparams) + \ldots | S_t = s] \\
\leq& C R_{max} + \tau \E_{\pi}[-\log \pi(A_t|S_t; \pparams) | S_t = s] + \tau \E_{\pi}[\gamma_{t+1} (-\log \pi(S_{t+1},A_{t+1}; \pparams)) | S_t = s] + \ldots  \\
\leq& C R_{max} +\tau \log n_a + \tau \log n_a  \E_{\pi}[\gamma_{t+1} | S_t = s] + \tau \log n_a  \E_{\pi}[\gamma_{t+2} | S_t = s]  + \ldots  \\
\leq& C R_{max} + C \tau \log n_a 
.
\end{align*}
Finally, we can lower bound the advantage function, using $\vtpi(s) \ge \E_{\pi}[G_{t}  | S_t = s] \ge -C R_{max}$ and 
\begin{align*}
\qtpi(s,a) &= \E_{\pi}[R_{t+1} + \gamma_{t+1} \vtpi(S_{t+1})  | S_t = s, A_t=a] - \tau \log \pi(a|s;\pparams) \\
&\geq -R_{max} - C R_{max} - \tau \log \pi(a|s;\pparams),
\end{align*}
giving 
\begin{align*}
\qtpi(s,a) - \vtpi(s) &\geq  -R_{max} - C R_{max} - \tau \log \pi(a|s;\pparams) - \vtpi(s) \\
&\geq -R_{max} - C R_{max} - \tau \log \pi(a|s;\pparams) - (C R_{max} + C \tau \log n_a )\\
&= - \tau \log \pi(a|s;\pparams) - (2C+1) R_{max} - C \tau \log n_a .
\end{align*}
\end{proof}
\begin{proposition}\label{prop_stationary}
Under Assumptions \ref{lipsc_assu}, \ref{rew_assu} and \ref{disc_assu}, the entropy-regularized policy update in Equation \ref{eq_pg_update} with $\tau>0$ and $\eta < 1/L$ converges to a point with zero gradient.
\end{proposition}
\begin{proof} 
Let $\{\pparams_t\}_{t=1}^\infty$ be the trajectory of parameters under the gradient ascent update, where $\pparams_t \rightarrow \ppinf$. This trajectory either converges to finite $\ppinf$ that provides a policy on the interior of the policy simplex, or it pushes the weights for a subset of actions to infinity to converge to a point on the simplex, where certain actions have zero probability. The gradient is only zero for the solution on the interior, and so to prove the result we simply need to show that this process converges to the interior. We show this is true for every parameter $\pparams_t(s,a)$ in the vector $\pparams_t$, where $s \in \repset$ and $a \in \Actions$. 

Define $\Actions_0(s)$ and $\Actions_+(s)$ as the sets of actions with zero and nonzero probability under $\pi(s, .; \ppinf)$. We use a proof by contradiction to show that $\Actions_0(s) = \emptyset$ for any $s$ as long as $\tau>0$. 
Suppose there exist $s$ and $a_0$ such that $a_0 \in \Actions_0(s)$. Note that $\pi(s', a; \pparams) = \pi(s, a; \pparams)$ for all $s'\in \alias(s)$.  We know that $\pi(s, a_0; \pparams_t) \rightarrow 0$ and $- \log \pi(s, a_0; \pparams_t) \rightarrow \infty$ as $t \rightarrow \infty$. Therefore there exists $t_0 > 0$ such that for all $s' \in \alias(s)$ and $t\geq t_0$
\begin{align*}
- \log \pi(s', a_0; \pparams_t) \geq \frac{(2C+1) R_{max} + C \tau \log n_a }{\tau}.
\end{align*}
According to Lemma \ref{lemma_breakdown}, for all $t\geq t_0$
\begin{align} \label{eq_gradnonneg}
&\frac{\partial \tilde{J}_\mu(\pparams_t) }{\partial \pparams_t(s,a_0)} = \sum_{s' \in \alias(s)} \emweight(s') \pi(s',a_0; \pparams_t) \biggl[\qtpi(s',a_0) - \vtpi(s') \biggr] \ \ \ \ \ \triangleright \text{And applying Lemma \ref{lem_vals_bounded}} \nonumber \\  
&= \sum_{s' \in \alias(s)} \emweight(s') \pi(s',a_0; \pparams_t) \biggl[ - \tau \log \pi(a|s;\pparams) - (2C+1) R_{max} - C \tau \log n_a  \biggr]\nonumber \\
&\geq \sum_{s' \in \alias(s)} \emweight(s') \pi(s',a_0; \pparams_t) \biggl[ \tau \frac{(2C+1) R_{max} + C \tau \log n_a }{\tau} - ((2C+1) R_{max} + C \tau \log n_a) \biggr] \nonumber\\
&= 0 .
\end{align}
So $\pparams_t(s,a_0)$ is non-decreasing after $t_0$, because we only add this non-negative gradient to $\pparams_t(s,a_0)$. This means that $\ppinf(s,a_0)$ is lower bounded by a constant $c$:
\begin{equation*}
\exp(\ppinf(s,a_0))>e^c>0
.
\end{equation*}
Next, we bound $|\sum_{a} \exp\bigl(\ppinf(s,a)\bigr)|$. Notice that the sum of gradient components over all actions at a state is zero:
\begin{align} \label{eq_sumgradzero}
\sum_a \frac{\partial \tilde{J}_\mu(\pparams_t) }{\partial \pparams_t(s,a)}
&= \sum_{s'\in\alias(s)}\emweight(s') \sum_a \pi(s',a;\pparams_t) [\qtpi(s',a) - \vtpi(s')] \nonumber \\
=&\sum_{s'\in\alias(s)}\emweight(s') \bigl( \vtpi(s') - \vtpi(s') \bigr) =0
.
\end{align}
Because 
\begin{equation*} 
\sum_a \frac{\partial \tilde{J}_\mu(\pparams_t) }{\partial \pparams_t(s,a)} = \sum_{a_0\in\Actions_0} \frac{\partial \tilde{J}_\mu(\pparams_t) }{\partial \pparams_t(s,a_0)} + \sum_{a_+\in\Actions_+} \frac{\partial \tilde{J}_\mu(\pparams_t) }{\partial \pparams_t(s,a_+)} 
\end{equation*} 
we get that for all $t>t_0$,
\begin{align*}
\sum_{a_+\in\Actions_+} \frac{\partial \tilde{J}_\mu(\pparams_t) }{\partial \pparams_t(s,a_+)} &= \sum_a \frac{\partial \tilde{J}_\mu(\pparams_t) }{\partial \pparams_t(s,a)} - \sum_{a_0\in\Actions_0} \frac{\partial \tilde{J}_\mu(\pparams_t) }{\partial \pparams_t(s,a_0)} \\
&= 0 - \sum_{a_0\in\Actions_0} \frac{\partial \tilde{J}_\mu(\pparams_t) }{\partial \pparams_t(s,a_0)} \leq 0 ,
\end{align*}
where the last line follows from Equation \ref{eq_gradnonneg}. This means that $\sum_{a_+\in\Actions_+(s)}\pparams_t(s,a_+)$ is non-increasing after $t_0$. 
For $a_+\in\Actions_+(s)$, we know that there exists some constant that lower bounds $\pparams_t(s,a_+)$ for all $t$. Otherwise, $a_+\in\Actions_+(s)$ would not have non-zero probability under $\pi(s, .; \ppinf)$. Therefore, non-increasing $\sum_{a_+\in\Actions_+(s)}\pparams_t(s,a_+)$ cannot be due to some $\pparams_t(s,a_+)$ approaching $-\infty$ while others approach $\infty$. Instead, each $\pparams_t(s,a_+)$ must also be bounded above. This means there exists some $b_+ > 0$ such that $\exp(-b_+) < \exp(\pparams_t(s,a_+)) \exp(b_+)$ for every $a_+\in\Actions_+(s)$. 

For all $a_0\in\Actions_0$, we also know that there exists $b_0$ such that $\exp(\pparams_t(s,a_0)) < b_0$ for all $t > t_0$. Otherwise, if $\exp(\pparams_t(s,a_0))$ approaches infinity, then $a_0$ would not be in $\Actions_0$. Putting this together with the above, and knowing there is at least one $a_+\in\Actions_+(s)$,
\begin{align*}
\sum_{a} \exp\bigl(\pparams_t(s,a)\bigr) &\le \numactions \exp(b_+) + b_0\\
\sum_{a} \exp\bigl(\pparams_t(s,a)\bigr) &\ge  \exp(b_+) > 0.
\end{align*}
Finally, this gives
\begin{align*}
\pi(s, a_0; \pparams_t) 
&= \frac{\exp\bigl(\pparams_t(s,a_0)\bigr)}{\sum_{a} \exp\bigl(\pparams_t(s,a)\bigr)} \\
&\ge \frac{c}{ \numactions \exp(b_+) + b_0} .
\end{align*}
Taking the limit as $t \rightarrow \infty$, this lower bound remains positive,
implying that $\pi(s, a_0; \ppinf) > 0$, which is a contradiction. 
 
 Therefore, no such $a_0$ can exist and so $\Actions_0(s) = \emptyset$ for all $s\in\States$. Since the policy converges to the interior of the probability simplex and the objective function is Lipschitz continuous, the gradient has to be zero at the point of convergence \citep{mei2020global} (Lemma 17). 
\end{proof}

\subsection{Proof for Three-State Counterexample}

In the previous section, we showed that a stationary point for the entropy-regularized PG objective exists. 
We now show that in the three-state counterexample that such a point is not a stationary point under semi-gradient updates.
\begin{lemma} \label{lemma_not_stationary}
A point with zero gradient is not a stationary point under semi-gradient updates in the three-state counterexample.
\end{lemma}
\begin{proof} Suppose $\pparams$ is a point with zero gradient, i.e.: 
\begin{align*}
\frac{\partial \tilde{J}_\mu(\pparams) }{\partial \pparams} = \0.
\end{align*}
Every component of the gradient vector is zero, thus
\begin{align*}
\frac{\partial \tilde{J}_\mu(\pparams) }{\partial \pparams(s_1,.)} = \sum_{s \in \alias(s_1)} \emweight(s) \sum_a \frac{\partial \pi(a|s; \pparams) }{\partial \pparams(s,.)}  \qtpi(s,a) = \0.
\end{align*}
Since taking any action from $s_1$ or $s_2$ results in termination, the following holds for $s\in\{s_1,s_2\}$ and any action $a$:
\begin{align*}
\qtpi(s,a) =& \sum_{s'}\text{P}(s'|s,a) \bigl[r(s,a,s') - \tau \log \pi(a|s;\pparams) + \gamma(s,a,s') \vtpi(s')\bigr]\\
=& \sum_{s'}\text{P}(s'|s,a) \bigl[ r(s,a,s') - \tau \log \pi(a|s;\pparams) + \gamma(s,a,s') \vtpi(s') \bigr]\\
=& r(s,a) - \tau \log \pi(a|s;\pparams) \qquad \triangleright \text{ Recall that } r(s,a) \defeq \sum_{s'}\text{P}(s'|s,a) r(s,a,s') .
\end{align*}
States $s_1$ and $s_2$ have the same representation, meaning that $\frac{\partial \pi(s_1,.; \pparams) }{\partial \pparams(s_1,.)} = \frac{\partial \pi(s_2,.; \pparams) }{\partial \pparams(s_2,.)}$ and $\pi(s_1,.; \pparams) = \pi(s_2,.; \pparams)$. So
\begin{align*}
\frac{\partial \tilde{J}_\mu(\pparams) }{\partial \pparams(s_1,.)} =& \sum_{s \in \alias(s_1)} \emweight(s) \sum_a \frac{\partial \pi(s_1,a; \pparams) }{\partial \pparams(s_1,.)}  [r(s,a) - \tau \log \pi(s_1,a; \pparams) ] \\
 =& \sum_{s \in \alias(s_1)} \emweight(s) \sum_a \frac{\partial \pi(s_1,a; \pparams) }{\partial \pparams(s_1,.)}  r(s,a) \\
&- \biggl(\sum_{s \in \alias(s_1)} \emweight(s) \biggr) \sum_a \frac{\partial \pi(s_1,a; \pparams) }{\partial \pparams(s_1,.)} \tau \log \pi(s_1,a; \pparams) .
\end{align*}
For simplicity of notation, define 
\begin{align*}
M(s_1) &\defeq \sum_{s \in \alias(s_1)} \emweight(s)\\
D(s_1) &\defeq \sum_{s \in \alias(s_1)} d_\mu(s),
\end{align*}
letting us write the above as 
\begin{align*}
\frac{\partial \tilde{J}_\mu(\pparams) }{\partial \pparams(s_1,.)} = \sum_{s \in \alias(s_1)} \emweight(s) \sum_a \frac{\partial \pi(s_1,a; \pparams) }{\partial \pparams(s_1,.)}  r(s,a) - M(s_1) \sum_a \frac{\partial \pi(s_1,a; \pparams) }{\partial \pparams(s_1,.)} \tau \log \pi(s_1,a; \pparams)
.
\end{align*}
Similarly, the semi-gradient update to $\pparams(s_1,.)$ becomes
\begin{align*}
&\sum_{s \in \alias(s_1)} d_\mu(s) \sum_a \frac{\partial \pi(s_1,a; \pparams) }{\partial \pparams(s_1,.)}  r(s,a) - D(s_1) \sum_a \frac{\partial \pi(s_1,a; \pparams) }{\partial \pparams(s_1,.)} \tau \log \pi(s_1,a; \pparams)  .
\end{align*}
We show that, even though $\pparams$ is a stationary point under the true gradient update, this semi-gradient under $\pparams$ is not zero. Notice first that 
\begin{align*}
&D(s_1) \sum_a \frac{\partial \pi(s_1,a; \pparams) }{\partial \pparams(s_1,.)} \tau \log \pi(s_1,a; \pparams)  \nonumber \\
&= D(s_1) \sum_a \frac{\partial \pi(s_1,a; \pparams) }{\partial \pparams(s_1,.)} \tau \log \pi(s_1,a; \pparams) + \frac{D(s_1)}{M(s_1)} \frac{\partial \tilde{J}_\mu(\pparams) }{\partial \pparams(s_1,.)} \ \ \ \  \triangleright \text{because } \frac{\partial \tilde{J}_\mu(\pparams) }{\partial \pparams(s_1,.)} = 0 \nonumber \\
&= D(s_1) \sum_a \frac{\partial \pi(s_1,a; \pparams) }{\partial \pparams(s_1,.)} \tau \log \pi(s_1,a; \pparams) \\
&+ \frac{D(s_1)}{M(s_1)}  \left[\sum_{s \in \alias(s_1)} \emweight(s) \sum_a \frac{\partial \pi(s_1,a; \pparams) }{\partial \pparams(s_1,.)}  r(s,a) 
- M(s_1) \sum_a \frac{\partial \pi(s_1,a; \pparams) }{\partial \pparams(s_1,.)} \tau \log \pi(s_1,a; \pparams) \right]\nonumber \\
&= D(s_1) \sum_a \frac{\partial \pi(s_1,a; \pparams) }{\partial \pparams(s_1,.)} \tau \log \pi(s_1,a; \pparams) + \frac{D(s_1)}{M(s_1)} \left[\sum_{s \in \alias(s_1)} \emweight(s) \sum_a \frac{\partial \pi(s_1,a; \pparams) }{\partial \pparams(s_1,.)}  r(s,a) \right]\\
&- D(s_1) \sum_a \frac{\partial \pi(s_1,a; \pparams) }{\partial \pparams(s_1,.)} \tau \log \pi(s_1,a; \pparams) \\
&= \frac{D(s_1)}{M(s_1)} \left[\sum_{s \in \alias(s_1)} \emweight(s) \sum_a \frac{\partial \pi(s_1,a; \pparams) }{\partial \pparams(s_1,.)}  r(s,a) \right].
\end{align*}
Therefore, 
\begin{align}\label{nonzero_update}
&\sum_{s \in \alias(s_1)} d_\mu(s) \sum_a \frac{\partial \pi(s_1,a; \pparams) }{\partial \pparams(s_1,.)}  r(s,a) - D(s_1) \sum_a \frac{\partial \pi(s_1,a; \pparams) }{\partial \pparams(s_1,.)} \tau \log \pi(s_1,a; \pparams)   \\
=& \sum_{s \in \alias(s_1)} d_\mu(s) \sum_a \frac{\partial \pi(s_1,a; \pparams) }{\partial \pparams(s_1,.)}  r(s,a) - \frac{D(s_1)}{M(s_1)} \left[\sum_{s \in \alias(s_1)} \emweight(s) \sum_a \frac{\partial \pi(s_1,a; \pparams) }{\partial \pparams(s_1,.)}  r(s,a) \right] \nonumber \\
=& \sum_{s \in \alias(s_1)} \delta(s) \sum_a \frac{\partial \pi(s_1,a; \pparams) }{\partial \pparams(s_1,.)}  r(s,a) \ \ \ \ \ \triangleright \text{ where } \delta(s) \defeq d_\mu(s) - \frac{\sum_{s' \in \alias(s_1)} d_\mu(s')}{\sum_{s' \in \alias(s_1)} \emweight(s')} \emweight(s)\nonumber \\
=& \delta(s_1) \frac{\partial \pi(s_1,a_0; \pparams) }{\partial \pparams(s_1,.)} \cdot 2 + \delta(s_2) \frac{\partial \pi(s_1,a_1; \pparams) }{\partial \pparams(s_1,.)} \cdot 1 \nonumber \\
&\qquad \qquad \qquad \qquad \triangleright \text{ because } r(s_1,a_0) = 2, r(s_2,a_1) = 1, r(s_1,a_1) = r(s_2,a_0) = 0 \nonumber \\ 
=& 2\bigl(\delta(s_1) - \delta(s_2)\bigr) \frac{\partial \pi(s_1,a_0; \pparams) }{\partial \pparams(s_1,.)}. \qquad \triangleright \text{ as } \frac{\partial \pi(s_1,a_1; \pparams) }{\partial \pparams(s_1,.)} = -  \frac{\partial \pi(s_1,a_0; \pparams) }{\partial \pparams(s_1,.)} \text{ by Lemma \ref{lemma_H}} \nonumber
\end{align}
The second factor $\frac{\partial \pi(s_1,a_0; \pparams) }{\partial \pparams(s_1,.)}$ is nonzero due to softmax parametrization. It therefore suffices to show that the first term $2\bigl(\delta(s_1) - \delta(s_2)\bigr)$ is not zero.

In each episode in the counterexample, regardless of the policy, the agent spends one step in $s_0$ and one step in either $s_1$ or $s_2$. Therefore
\begin{align*}
d_\mu(s_0)=0.5, \, d_\mu(s_1) + d_\mu(s_2) =0.5.
\end{align*}
Unfolding the emphatic weightings gives
\begin{align*}
\emweight(s_0) &= d_\mu(s_0)=0.5\\
\emweight(s_1) &= d_\mu(s_1) + d_\mu(s_0) \pi(s_0,a_0; \pparams) =  d_\mu(s_1) + 0.5 \pi(s_0,a_0; \pparams)\\
\emweight(s_2) &= d_\mu(s_2) + d_\mu(s_0) \pi(s_0,a_1; \pparams) =  d_\mu(s_2) + 0.5 \pi(s_0,a_1; \pparams)\\
&= 0.5 - d_\mu(s_1) + 0.5 (1- \pi(s_0,a_0; \pparams)) = 1 - d_\mu(s_1) - 0.5 \pi(s_0,a_0; \pparams).
\end{align*}
Plugging these values into $\delta(s)$ results in
\begin{align*}
\delta(s) &\defeq d_\mu(s) - \frac{\sum_{s' \in \alias(s_1)} d_\mu(s')}{\sum_{s' \in \alias(s_1)} \emweight(s')} \emweight(s)\\
&= d_\mu(s) - \frac{d_\mu(s_1) + d_\mu(s_2)}{\emweight(s_1) + \emweight(s_2)}\emweight(s)\\
&= d_\mu(s) - \frac{0.5}{1}\emweight(s) = d_\mu(s) - 0.5\emweight(s).
\end{align*}
Therefore,
\begin{align*}
\delta(s_1) &= d_\mu(s_1) - 0.5\emweight(s_1) = 0.5d_\mu(s_1) - 0.25 \pi(s_0,a_0; \pparams)\\
\delta(s_2) &= d_\mu(s_2) - 0.5\emweight(s_2) = 0.5 - d_\mu(s_1) - 0.5 + 0.5d_\mu(s_1) + 0.25 \pi(s_0,a_0; \pparams)\\
&= -0.5d_\mu(s_1) +  0.25 \pi(s_0,a_0; \pparams),
\end{align*}
and we get that
\begin{align*}
2\delta(s_1) - \delta(s_2) &= d_\mu(s_1) - 0.5 \pi(s_0,a_0; \pparams) + 0.5d_\mu(s_1) -  0.25 \pi(s_0,a_0; \pparams)\\
&= 1.5d_\mu(s_1) - 0.75 \pi(s_0,a_0; \pparams)\\
&= 0.75 \bigl(\mu(s_0,a_0) - \pi(s_0,a_0; \pparams) \bigr). \qquad \triangleright \text{ because } d_\mu(s_1) = 0.5  \mu(s_0,a_0)
\end{align*}

As long as $\pi(s_0,a_0; \pparams) \neq \mu(s_0,a_0)$, the semi-gradient is not zero. For example, as in Section \ref{sec_counter}, we can choose $\mu(s_0,a_0) = 0.25$. Now we show that $\pi(s_0,a_0; \pparams) \neq 0.25$ for a stationary point under the true gradient. 

Let us first write the partial derivative w.r.t. the first parameter given that $\pi(s_0,a_0; \pparams) = 0.25$:
\begin{align*}
\frac{\partial \tilde{J}_\mu(\pparams) }{\partial \pparams(s_0,a_0)} &= \sum_{s \in \alias(s_0)} \emweight(s) \sum_a \frac{\partial \pi(a|s; \pparams) }{\partial \pparams(s_0,a_0)}  \qtpi(s,a)\\
&= \emweight(s_0) \biggl[ \frac{\partial \pi(s_0,a_0; \pparams) }{\partial \pparams(s_0,a_0)}  \qtpi(s_0,a_0) + \frac{\partial \pi(s_0,a_1; \pparams) }{\partial \pparams(s_0,a_0)}  \qtpi(s_0,a_1) \biggr]\\
&= 0.5 \biggl[ \pi(s_0,a_0; \pparams) (1-\pi(s_0,a_0; \pparams)) \qtpi(s_0,a_0) \\
&- \pi(s_0,a_1; \pparams)\pi(s_0,a_0; \pparams) \qtpi(s_0,a_1) \biggr]\\
&= 0.5*0.25*0.75 \bigl[ \qtpi(s_0,a_0) - \qtpi(s_0,a_1) \bigr].
\end{align*}
For this derivative to be zero, $\qtpi(s_0,a_0) - \qtpi(s_0,a_1)$ must be zero.
\begin{align*}
&\qtpi(s_0,a_0) - \qtpi(s_0,a_1) \\
&= -\tau \log \pi(s_0,a_0; \pparams) + \pi(s_1,a_0,\pparams) \qtpi(s_1,a_0) + \pi(s_1,a_1,\pparams) \qtpi(s_1,a_1) \\
&+ \tau \log \pi(s_0,a_1; \pparams) - \pi(s_2,a_0,\pparams) \qtpi(s_2,a_0) - \pi(s_2,a_1,\pparams) \qtpi(s_2,a_1)\\
&= -\tau \log 0.25 + \pi(s_1,a_0,\pparams) (2 - \tau \log \pi(s_1,a_0,\pparams)) + \pi(s_1,a_1,\pparams) (-\tau \log \pi(s_1,a_1,\pparams)) \\
&+ \tau \log 0.75 - \pi(s_2,a_0,\pparams) (-\tau \log \pi(s_2,a_0,\pparams)) - \pi(s_2,a_1,\pparams) (1 - \tau \log \pi(s_2,a_1,\pparams))
\end{align*}
Recall that $\pi(s_2,\cdot,\pparams) = \pi(s_1,\cdot,\pparams)$ and $\pi(s_1,a_1,\pparams) = 1 - \pi(s_1,a_0,\pparams)$.
Then
\begin{align}
&\qtpi(s_0,a_0) - \qtpi(s_0,a_1) \nonumber\\
&= -\tau \log 0.25 + \pi(s_1,a_0,\pparams) (2 - \tau \log \pi(s_1,a_0,\pparams)) + \pi(s_1,a_1,\pparams) (-\tau \log \pi(s_1,a_1,\pparams)) \nonumber\\
&+ \tau \log 0.75 - \pi(s_1,a_0,\pparams) (-\tau \log \pi(s_1,a_0,\pparams)) - \pi(s_1,a_1,\pparams) (1 - \tau \log \pi(s_1,a_1,\pparams)) \nonumber\\
&= \tau (\log 0.75 - \log 0.25) + 2 \pi(s_1,a_0,\pparams) - \pi(s_1,a_1,\pparams)\nonumber\\
&= \tau (\log 3) + 3 \pi(s_1,a_0,\pparams) - 1 \label{eq:first_condition},
\end{align}
because $\log 0.75 - \log 0.25 = \log(0.75/0.25) = \log 3$. 
For small values of $\tau$, $\pi(s_1,a_0,\pparams)$ should be close to $1/3$ to make $\pparams$ a stationary point.

Now let us write the partial derivative w.r.t. $\pparams(s_1,a_0)$, again assuming $\pi(s_0,a_0; \pparams) = 0.25$ and by noting that $s_1$ and $s_2$ are aliased:
\begin{align*}
\frac{\partial \tilde{J}_\mu(\pparams) }{\partial \pparams(s_1,a_0)} &= \sum_{s \in \alias(s_1)} \emweight(s) \sum_a \frac{\partial \pi(a|s; \pparams) }{\partial \pparams(s_1,a_0)}  \qtpi(s,a)\\
&= \emweight(s_1) \biggl[ \frac{\partial \pi(s_1,a_0; \pparams) }{\partial \pparams(s_1,a_0)}  \qtpi(s_1,a_0) + \frac{\partial \pi(s_1,a_1; \pparams) }{\partial \pparams(s_1,a_0)}  \qtpi(s_1,a_1) \biggr]\\
&+ \emweight(s_2) \biggl[ \frac{\partial \pi(s_2,a_0; \pparams) }{\partial \pparams(s_1,a_0)}  \qtpi(s_2,a_0) + \frac{\partial \pi(s_2,a_1; \pparams) }{\partial \pparams(s_1,a_0)}  \qtpi(s_2,a_1) \biggr]\\
&= \emweight(s_1) \biggl[ \pi(s_1,a_0; \pparams) (1 - \pi(s_1,a_0; \pparams))  \qtpi(s_1,a_0) \\
&- \pi(s_1,a_0; \pparams) (1 - \pi(s_1,a_0; \pparams))  \qtpi(s_1,a_1) \biggr]\\
&+ \emweight(s_2) \biggl[ \pi(s_1,a_0; \pparams) (1 - \pi(s_1,a_0; \pparams))  \qtpi(s_2,a_0) \\
&- \pi(s_1,a_0; \pparams) (1 - \pi(s_1,a_0; \pparams))  \qtpi(s_2,a_1) \biggr]\\
&= \emweight(s_1) \biggl[ \pi(s_1,a_0; \pparams) (1 - \pi(s_1,a_0; \pparams))  \qtpi(s_1,a_0) \\
&- \pi(s_1,a_0; \pparams) (1 - \pi(s_1,a_0; \pparams))  \qtpi(s_1,a_1) \biggr]\\
&+ \emweight(s_2) \biggl[ \pi(s_1,a_0; \pparams) (1 - \pi(s_1,a_0; \pparams))  \qtpi(s_2,a_0) \\
&- \pi(s_1,a_0; \pparams) (1 - \pi(s_1,a_0; \pparams))  \qtpi(s_2,a_1) \biggr]\\
&= \pi(s_1,a_0; \pparams) (1 - \pi(s_1,a_0; \pparams))\\
&\biggl[ \emweight(s_1) \bigl( \qtpi(s_1,a_0) - \qtpi(s_1,a_1) \bigr) + \emweight(s_2) \bigl( \qtpi(s_2,a_0) - \qtpi(s_2,a_1) \bigr) \biggr].
\end{align*}
Making the partial derivative zero requires either $\pi(s_1,a_0; \pparams)$, $1 - \pi(s_1,a_0; \pparams)$, or the contents of the brackets to be zero. The first two are incompatible with the requirement for making Equation \ref{eq:first_condition} zero so we will continue with the third one. Note that
\begin{align*}
\emweight(s_1) &= d_\mu(s_1) + d_\mu(s_0) \pi(s_0,a_0; \pparams) = 0.125 + 0.5*0.25 = 0.25,\\
\emweight(s_2) &= d_\mu(s_2) + d_\mu(s_0) \pi(s_0,a_1; \pparams) = 0.375 + 0.5*0.75 = 0.75,
\end{align*}
and the factor in the brackets becomes
\begin{align*}
&\emweight(s_1) \bigl( \qtpi(s_1,a_0) - \qtpi(s_1,a_1) \bigr) + \emweight(s_2) \bigl( \qtpi(s_2,a_0) - \qtpi(s_2,a_1) \bigr)\\
&= 0.25 \bigl(2 -\tau \log \pi(s_1,a_0,\pparams) + \tau \log \pi(s_1,a_1,\pparams)\bigr)\\
&+ 0.75 \bigl(-\tau \log \pi(s_2,a_0,\pparams) - 1 + \tau \log \pi(s_2,a_1,\pparams)\bigr)\\
&= -0.25 + \tau \bigl(-\log \pi(s_1,a_0,\pparams) + \log \pi(s_1,a_1,\pparams)\bigr)\\
&= -0.25 + \tau \bigl( \log (1 - \pi(s_1,a_0,\pparams)) -\log \pi(s_1,a_0,\pparams)\bigr)\\
&= -0.25 + \tau \log\bigl( \frac{1 - \pi(s_1,a_0,\pparams)}{\pi(s_1,a_0,\pparams)}\bigr).
\end{align*}
Making this zero is also at odds with the requirement for Equation \ref{eq:first_condition}. To see why, let's consider both conditions, where we use $p = \pi(s_1,a_0,\pparams)$ to simplify notation. 
\begin{align*}
&\tau \log 3 + 3 p - 1 = 0\\
&-0.25 + \tau \log\frac{1-p}{p} = 0
\end{align*}
%
The first equation requires $p = (1- \tau \log 3)/3$. The second equation requires $p = (\exp(0.25/ \tau) + 1)^\inv$. These two equations intersect when $\tau_{i} \approx 0.2779$, but otherwise do not equal each other, meaning we cannot satisfy both of these criteria. 
 Therefore, for any $\tau \neq \tau_i$, a stationary point $\pparams$ under the true gradient cannot have $\pi(s_0,a_0,\pparams)=\mu(s_0,a_0)=0.25$ and thus cannot be a stationary point under semi-gradient.


\end{proof}
This counterexample shows one environment where the sets of stationary points under the true gradient is disjoint from the set of stationary points under the semi-gradient.

\section{Descriptions of Algorithms in Section \ref{sec:experiments}}\label{app:algorithm_descriptions}
\begin{table}[h!]
  \begin{tabular}{l p{.85\textwidth}}
    ACE($\lambdaa$) & Actor-Critic with Emphatic weightings. $\lambdaa$ interpolates between semi-gradient updates ($\eta=0$) and gradient updates ($\eta=1$). \\
    OffPAC & Off-Policy Actor-Critic \citep{degris2012offpolicy}. Equivalent to ACE($0$). \\
    GTD($\lambda$) & Gradient Temporal-Difference learning \citep{sutton2009fast}. $\lambda$ is the decay rate of the eligibility trace vector. \\
    ETD($\lambda$) & Emphatic Temporal-Difference learning \citep{sutton2016anemphatic}. $\lambda$ is the decay rate of the eligibility trace vector. \\
    TDRC($\lambda$) & Temporal-Difference learning with Regularized Corrections \citep{ghiassian2020gradient,patterson2022ageneralized}. $\lambda$ is the decay rate of the eligibility trace vector. \\
    DPG & Deterministic Policy Gradient \citep{silver2014deterministic}. Uses a deterministic policy parameterization and semi-gradient updates. \\
    True-DPGE & A version of DPG using true gradient updates (i.e., scaling the update by the true emphatic weightings). \\
    True-ACE & A version of ACE using the true emphatic weightings. \\
    ACE-direct & A version of ACE that uses the direct method from Section \ref{sec_direct_f} to estimate the emphatic weightings.\\
    ACE-trace & A version of ACE that uses the emphatic trace from \citet{sutton2016anemphatic} to estimate the emphatic weightings.\\
    ACE-ideal & A version of ACE that re-computes the emphatic trace on each time step using the current policy to remove the influence of old versions of the policy.\\
  \end{tabular}
  \caption{Descriptions of the algorithms used in the experiments.}
  \label{tab:algorithm_descriptions}
\end{table}

\bibliography{ac}

\begin{thebibliography}{74}
\providecommand{\natexlab}[1]{#1}
\providecommand{\url}[1]{\texttt{#1}}
\expandafter\ifx\csname urlstyle\endcsname\relax
  \providecommand{\doi}[1]{doi: #1}\else
  \providecommand{\doi}{doi: \begingroup \urlstyle{rm}\Url}\fi

\bibitem[Ahmed et~al.(2019)Ahmed, Le~Roux, Norouzi, and
  Schuurmans]{ahmed2019understanding}
Zafarali Ahmed, Nicolas Le~Roux, Mohammad Norouzi, and Dale Schuurmans.
\newblock Understanding the impact of entropy on policy optimization.
\newblock In \emph{International Conference on Machine Learning}, 2019.

\bibitem[Barto et~al.(1983)Barto, Sutton, and Anderson]{barto1983neuronlike}
Andrew~G. Barto, Richard~S. Sutton, and Charles~W. Anderson.
\newblock Neuronlike adaptive elements that can solve difficult learning
  control problems.
\newblock \emph{IEEE Transactions on Systems, Man, and Cybernetics}, \penalty0
  (5), 1983.

\bibitem[Bhatnagar et~al.(2007)Bhatnagar, Sutton, Ghavamzadeh, and
  Lee]{bhatnagar2007incremental}
Shalabh Bhatnagar, Richard~S. Sutton, Mohammad Ghavamzadeh, and Mark Lee.
\newblock Incremental natural actor--critic algorithms.
\newblock In \emph{Conference on Neural Information Processing Systems}, 2007.

\bibitem[Bhatnagar et~al.(2009)Bhatnagar, Sutton, Ghavamzadeh, and
  Lee]{bhatnagar2009natural}
Shalabh Bhatnagar, Richard~S. Sutton, Mohammad Ghavamzadeh, and Mark Lee.
\newblock Natural actor--critic algorithms.
\newblock \emph{Automatica}, 45\penalty0 (11), 2009.

\bibitem[Boyan and Moore(1995)]{boyan1995generalization}
Justin Boyan and Andrew~W. Moore.
\newblock Generalization in reinforcement learning: Safely approximating the
  value function.
\newblock \emph{Conference on Neural Information Processing Systems}, pages
  369--376, 1995.

\bibitem[Brockman et~al.(2016)Brockman, Cheung, Pettersson, Schneider,
  Schulman, Tang, and Zaremba]{brockman2016openai}
Greg Brockman, Vicki Cheung, Ludwig Pettersson, Jonas Schneider, John Schulman,
  Jie Tang, and Wojciech Zaremba.
\newblock {OpenAI Gym}.
\newblock \emph{arXiv preprint arXiv:1606.01540}, 2016.

\bibitem[Chevalier-Boisvert et~al.(2018)Chevalier-Boisvert, Willems, and
  Pal]{gym_minigrid}
Maxime Chevalier-Boisvert, Lucas Willems, and Suman Pal.
\newblock Minimalistic gridworld environment for {OpenAI Gym}.
\newblock \url{https://github.com/maximecb/gym-minigrid}, 2018.

\bibitem[Degris et~al.(2012{\natexlab{a}})Degris, Pilarski, and
  Sutton]{degris2012model}
Thomas Degris, Patrick~M. Pilarski, and Richard~S. Sutton.
\newblock Model--free reinforcement learning with continuous action in
  practice.
\newblock In \emph{American Control Conference}, 2012{\natexlab{a}}.

\bibitem[Degris et~al.(2012{\natexlab{b}})Degris, White, and
  Sutton]{degris2012offpolicy}
Thomas Degris, Martha White, and Richard~S. Sutton.
\newblock Off--policy actor--critic.
\newblock In \emph{International Conference on Machine Learning},
  2012{\natexlab{b}}.

\bibitem[Dick(2015)]{dick2015policy}
Travis~B. Dick.
\newblock Policy gradient reinforcement learning without regret.
\newblock Master's thesis, University of Alberta, 2015.

\bibitem[Dung et~al.(2007)Dung, Komeda, and Takagi]{dung2007reinforcement}
Le~Tien Dung, Takashi Komeda, and Motoki Takagi.
\newblock Reinforcement learning in non-{Markovian} environments using
  automatic discovery of subgoals.
\newblock In \emph{Society of Instrument and Control Engineers Annual
  Conference}, pages 2601--2605. IEEE, 2007.

\bibitem[et~al.(2020)]{harris2020array}
Charles R.~Harris et~al.
\newblock Array programming with {NumPy}.
\newblock \emph{Nature}, 585\penalty0 (7825):\penalty0 357--362, September
  2020.
\newblock \doi{10.1038/s41586-020-2649-2}.
\newblock URL \url{https://doi.org/10.1038/s41586-020-2649-2}.

\bibitem[Geist et~al.(2019)Geist, Scherrer, and Pietquin]{geist2019theory}
Matthieu Geist, Bruno Scherrer, and Olivier Pietquin.
\newblock A theory of regularized markov decision processes.
\newblock In \emph{International Conference on Machine Learning}, pages
  2160--2169, 2019.

\bibitem[Gelada and Bellemare(2019)]{gelada2019off}
Carles Gelada and Marc~G. Bellemare.
\newblock Off-policy deep reinforcement learning by bootstrapping the covariate
  shift.
\newblock In \emph{AAAI Conference on Artificial Intelligence}, volume~33,
  2019.

\bibitem[Ghiassian et~al.(2018)Ghiassian, Patterson, White, Sutton, and
  White]{ghiassian2018online}
Sina Ghiassian, Andrew Patterson, Martha White, Richard~S. Sutton, and Adam
  White.
\newblock Online off-policy prediction.
\newblock \emph{arXiv:1811.02597}, 2018.

\bibitem[Ghiassian et~al.(2020)Ghiassian, Patterson, Garg, Gupta, White, and
  White]{ghiassian2020gradient}
Sina Ghiassian, Andrew Patterson, Shivam Garg, Dhawal Gupta, Adam White, and
  Martha White.
\newblock Gradient temporal-difference learning with regularized corrections.
\newblock In \emph{International Conference on Machine Learning}, pages
  3524--3534. PMLR, 2020.

\bibitem[Ghosh et~al.(2020)Ghosh, Machado, and Le~Roux]{ghosh2020operator}
Dibya Ghosh, Marlos~C. Machado, and Nicolas Le~Roux.
\newblock An operator view of policy gradient methods.
\newblock In \emph{Conference on Neural Information Processing Systems}, 2020.

\bibitem[Greensmith et~al.(2004)Greensmith, Bartlett, and
  Baxter]{greensmith2004variance}
Evan Greensmith, Peter~L. Bartlett, and Jonathan Baxter.
\newblock Variance reduction techniques for gradient estimates in reinforcement
  learning.
\newblock \emph{The Journal of Machine Learning Research}, 5, 2004.

\bibitem[Grondman et~al.(2012)Grondman, Busoniu, Lopes, and
  Babuska]{grondman2012survey}
Ivo Grondman, Lucian Busoniu, Gabriel~A.D. Lopes, and Robert Babuska.
\newblock A survey of actor--critic reinforcement learning: Standard and
  natural policy gradients.
\newblock \emph{IEEE Transactions on Systems, Man, and Cybernetics, Part C
  (Applications and Reviews)}, 42\penalty0 (6), 2012.

\bibitem[Gu et~al.(2017{\natexlab{a}})Gu, Lillicrap, Turner, Ghahramani,
  Sch{\"o}lkopf, and Levine]{gu2017interpolated}
Shixiang Gu, Tim Lillicrap, Richard~E. Turner, Zoubin Ghahramani, Bernhard
  Sch{\"o}lkopf, and Sergey Levine.
\newblock Interpolated policy gradient: Merging on--policy and off--policy
  gradient estimation for deep reinforcement learning.
\newblock In \emph{Conference on Neural Information Processing Systems},
  2017{\natexlab{a}}.

\bibitem[Gu et~al.(2017{\natexlab{b}})Gu, Lillicrap, Ghahramani, Turner, and
  Levine]{gu2016q}
Shixiang Gu, Timothy Lillicrap, Zoubin Ghahramani, Richard~E. Turner, and
  Sergey Levine.
\newblock Q--prop: Sample--efficient policy gradient with an off--policy
  critic.
\newblock In \emph{International Conference on Learning Representations},
  2017{\natexlab{b}}.

\bibitem[Haarnoja et~al.(2018)Haarnoja, Zhou, Abbeel, and
  Levine]{haarnoja2018soft}
Tuomas Haarnoja, Aurick Zhou, Pieter Abbeel, and Sergey Levine.
\newblock Soft actor-critic: Off-policy maximum entropy deep reinforcement
  learning with a stochastic actor.
\newblock In \emph{International Conference on Machine Learning}, 2018.

\bibitem[Hallak and Mannor(2017)]{hallak2017consistent}
Assaf Hallak and Shie Mannor.
\newblock Consistent on-line off-policy evaluation.
\newblock In \emph{International Conference on Machine Learning}, volume~70,
  2017.

\bibitem[Huang(2020)]{bojun2020steady}
Bojun Huang.
\newblock Steady state analysis of episodic reinforcement learning.
\newblock In \emph{Conference on Neural Information Processing Systems}, 2020.

\bibitem[Imani et~al.(2018)Imani, Graves, and White]{imani2018off}
Ehsan Imani, Eric Graves, and Martha White.
\newblock An off--policy policy gradient theorem using emphatic weightings.
\newblock In \emph{Conference on Neural Information Processing Systems}, 2018.

\bibitem[Kallus and Uehara(2020)]{kallus2020statistically}
Nathan Kallus and Masatoshi Uehara.
\newblock Statistically efficient off--policy policy gradients.
\newblock In \emph{International Conference on Machine Learning}, 2020.

\bibitem[Kirsch(2017)]{blackhc.mdp}
Andreas Kirsch.
\newblock Mdp environments for the {OpenAI Gym}.
\newblock Technical report, 2017.
\newblock URL \url{https://arxiv.org/abs/1709.09069}.

\bibitem[Klissarov and Precup(2021)]{klissarov2021flexible}
Martin Klissarov and Doina Precup.
\newblock Flexible option learning.
\newblock \emph{Conference on Neural Information Processing Systems}, pages
  4632--4646, 2021.

\bibitem[Konda and Tsitsiklis(2000)]{konda2000actor}
Vijay~R. Konda and John~N. Tsitsiklis.
\newblock Actor--critic algorithms.
\newblock In \emph{Conference on Neural Information Processing Systems}, 2000.

\bibitem[Laroche and Tachet~des Combes(2021)]{laroche2021dr}
Romain Laroche and Remi Tachet~des Combes.
\newblock Dr jekyll \& mr hyde: the strange case of off-policy policy updates.
\newblock \emph{Conference on Neural Information Processing Systems}, 2021.

\bibitem[Lillicrap et~al.(2015)Lillicrap, Hunt, Pritzel, Heess, Erez, Tassa,
  Silver, and Wierstra]{lillicrap2015continuous}
Timothy~P. Lillicrap, Jonathan~J. Hunt, Alexander Pritzel, Nicolas Heess, Tom
  Erez, Yuval Tassa, David Silver, and Daan Wierstra.
\newblock Continuous control with deep reinforcement learning.
\newblock In \emph{International Conference on Learning Representations}, 2015.

\bibitem[Lin(1992)]{lin1992self}
Long-Ji Lin.
\newblock Self-improving reactive agents based on reinforcement learning,
  planning and teaching.
\newblock \emph{Machine learning}, 8\penalty0 (3-4), 1992.

\bibitem[Lin(1993)]{lin1993reinforcement}
Long-Ji Lin.
\newblock Reinforcement learning for robots using neural networks.
\newblock Technical report, Carnegie-Mellon University, 1993.

\bibitem[Liu et~al.(2018)Liu, Li, Tang, and Zhou]{liu2018breaking}
Qiang Liu, Lihong Li, Ziyang Tang, and Dengyong Zhou.
\newblock Breaking the curse of horizon: Infinite-horizon off-policy
  estimation.
\newblock In \emph{Conference on Neural Information Processing Systems}, 2018.

\bibitem[Liu et~al.(2019)Liu, Swaminathan, Agarwal, and Brunskill]{liu2019off}
Yao Liu, Adith Swaminathan, Alekh Agarwal, and Emma Brunskill.
\newblock Off-policy policy gradient with stationary distribution correction.
\newblock In \emph{Conference on Uncertainty in Artificial Intelligence}, 2019.

\bibitem[Liu et~al.(2020{\natexlab{a}})Liu, Bacon, and
  Brunskill]{liu2020understanding}
Yao Liu, Pierre-Luc Bacon, and Emma Brunskill.
\newblock Understanding the curse of horizon in off-policy evaluation via
  conditional importance sampling.
\newblock In \emph{International Conference on Machine Learning}, pages
  6184--6193. PMLR, 2020{\natexlab{a}}.

\bibitem[Liu et~al.(2020{\natexlab{b}})Liu, Swaminathan, Agarwal, and
  Brunskill]{liu2020offpolicy}
Yao Liu, Adith Swaminathan, Alekh Agarwal, and Emma Brunskill.
\newblock Off--policy policy gradient with stationary distribution correction.
\newblock In \emph{Conference on Uncertainty in Artificial Intelligence},
  2020{\natexlab{b}}.

\bibitem[Loch and Singh(1998)]{loch1998using}
John Loch and Satinder~P. Singh.
\newblock Using eligibility traces to find the best memoryless policy in
  partially observable {Markov} decision processes.
\newblock In \emph{International Conference on Machine Learning}, pages
  323--331, 1998.

\bibitem[Maei(2011)]{maei2011gradient}
Hamid Maei.
\newblock \emph{{Gradient Temporal--Difference Learning Algorithms}}.
\newblock PhD thesis, University of Alberta, 2011.

\bibitem[Maei(2018)]{maei2018convergent}
Hamid~Reza Maei.
\newblock Convergent actor--critic algorithms under off-policy training and
  function approximation.
\newblock \emph{arXiv:1802.07842}, 2018.

\bibitem[Mahmood et~al.(2015)Mahmood, Yu, White, and
  Sutton]{mahmood2015emphatic}
A.~Rupam Mahmood, Huizhen Yu, Martha White, and Richard~S. Sutton.
\newblock Emphatic temporal-difference learning.
\newblock \emph{arXiv:1507.01569}, 2015.

\bibitem[Marbach and Tsitsiklis(2001)]{marbach2001simulation}
Peter Marbach and John~N. Tsitsiklis.
\newblock Simulation--based optimization of {Markov} reward processes.
\newblock \emph{IEEE Transactions on Automatic Control}, 46\penalty0 (2), 2001.

\bibitem[Mei et~al.(2020)Mei, Xiao, Szepesvari, and Schuurmans]{mei2020global}
Jincheng Mei, Chenjun Xiao, Csaba Szepesvari, and Dale Schuurmans.
\newblock On the global convergence rates of softmax policy gradient methods.
\newblock In \emph{International Conference on Machine Learning}, 2020.

\bibitem[Meuleau et~al.(2000)Meuleau, Peshkin, Kaelbling, and
  Kim]{meuleau2000off}
Nicolas Meuleau, Leonid Peshkin, Leslie~P. Kaelbling, and Kee-Eung Kim.
\newblock Off--policy policy search.
\newblock Technical report, MIT Artificial Intelligence Laboratory, 2000.

\bibitem[Mnih et~al.(2015)Mnih, Kavukcuoglu, Silver, Rusu, Veness, Bellemare,
  Graves, Riedmiller, Fidjeland, Ostrovski, et~al.]{mnih2015human}
Volodymyr Mnih, Koray Kavukcuoglu, David Silver, Andrei~A. Rusu, Joel Veness,
  Marc~G. Bellemare, Alex Graves, Martin Riedmiller, Andreas~K. Fidjeland,
  Georg Ostrovski, et~al.
\newblock Human--level control through deep reinforcement learning.
\newblock \emph{Nature}, 518\penalty0 (7540), 2015.

\bibitem[Moore(1990)]{moore1990efficient}
Andrew~William Moore.
\newblock \emph{Efficient memory-based learning for robot control}.
\newblock PhD thesis, University of Cambridge, 1990.

\bibitem[Nota and Thomas(2020)]{nota2020isthe}
Chris Nota and Philip~S. Thomas.
\newblock Is the policy gradient a gradient?
\newblock In \emph{International Conference on Autonomous Agents and Multiagent
  Systems}, 2020.

\bibitem[Patterson et~al.(2022)Patterson, White, and
  White]{patterson2022ageneralized}
Andrew Patterson, Adam White, and Martha White.
\newblock A generalized projected bellman error for off-policy value estimation
  in reinforcement learning.
\newblock \emph{The Journal of Machine Learning Research}, 2022.

\bibitem[Peters et~al.(2005)Peters, Vijayakumar, and Schaal]{peters2005natural}
Jan Peters, Sethu Vijayakumar, and Stefan Schaal.
\newblock Natural actor-critic.
\newblock In \emph{European Conference on Machine Learning}, 2005.

\bibitem[Precup(2000)]{precup2000temporal}
Doina Precup.
\newblock \emph{Temporal Abstraction In Reinforcement Learning}.
\newblock PhD thesis, University of Massachusetts Amherst, 2000.

\bibitem[Precup et~al.(2000)Precup, Sutton, and Singh]{precup2000eligibility}
Doina Precup, Richard~S. Sutton, and Satinder~P. Singh.
\newblock {Eligibility Traces for Off--Policy Policy Evaluation}.
\newblock In \emph{International Conference on Machine Learning}, 2000.

\bibitem[Precup et~al.(2001)Precup, Sutton, and Dasgupta]{precup2001offpolicy}
Doina Precup, Richard~S. Sutton, and Sanjoy Dasgupta.
\newblock Off--policy temporal--difference learning with function
  approximation.
\newblock \emph{International Conference on Machine Learning}, 2001.

\bibitem[Schaul et~al.(2016)Schaul, Quan, Antonoglou, and
  Silver]{schaul2015prioritized}
Tom Schaul, John Quan, Ioannis Antonoglou, and David Silver.
\newblock Prioritized experience replay.
\newblock In \emph{International Conference on Learning Representations}, 2016.

\bibitem[Schmidt and Roux(2013)]{schmidt2013fast}
Mark Schmidt and Nicolas~Le Roux.
\newblock Fast convergence of stochastic gradient descent under a strong growth
  condition.
\newblock \emph{arXiv:1308.6370}, 2013.

\bibitem[Silver et~al.(2014)Silver, Lever, Heess, Degris, Wierstra, and
  Riedmiller]{silver2014deterministic}
David Silver, Guy Lever, Nicolas Heess, Thomas Degris, Daan Wierstra, and
  Martin Riedmiller.
\newblock Deterministic policy gradient algorithms.
\newblock In \emph{International Conference on Machine Learning}, 2014.

\bibitem[Sutton and Barto(2018)]{sutton2018reinforcement}
Richard~S. Sutton and Andrew~G. Barto.
\newblock \emph{Reinforcement Learning: An Introduction}.
\newblock 2nd edition, 2018.

\bibitem[Sutton et~al.(1999{\natexlab{a}})Sutton, McAllester, Singh, and
  Mansour]{sutton2000policy}
Richard~S. Sutton, David McAllester, Satinder Singh, and Yishay Mansour.
\newblock Policy gradient methods for reinforcement learning with function
  approximation.
\newblock In \emph{Conference on Neural Information Processing Systems},
  1999{\natexlab{a}}.

\bibitem[Sutton et~al.(1999{\natexlab{b}})Sutton, Precup, and
  Singh]{sutton1999between}
Richard~S. Sutton, Doina Precup, and Satinder Singh.
\newblock Between mdps and semi-mdps: A framework for temporal abstraction in
  reinforcement learning.
\newblock \emph{Artificial intelligence}, 112\penalty0 (1-2),
  1999{\natexlab{b}}.

\bibitem[Sutton et~al.(2009)Sutton, Maei, Precup, Bhatnagar, Silver,
  Szepesv{\'a}ri, and Wiewiora]{sutton2009fast}
Richard~S. Sutton, Hamid~Reza Maei, Doina Precup, Shalabh Bhatnagar, David
  Silver, Csaba Szepesv{\'a}ri, and Eric Wiewiora.
\newblock Fast gradient-descent methods for temporal--difference learning with
  linear function approximation.
\newblock In \emph{International Conference on Machine Learning}, 2009.

\bibitem[Sutton et~al.(2011)Sutton, Modayil, Delp, Degris, Pilarski, White, and
  Precup]{sutton2011horde}
Richard~S. Sutton, Joseph Modayil, Michael Delp, Thomas Degris, Patrick~M.
  Pilarski, Adam White, and Doina Precup.
\newblock Horde: A scalable real--time architecture for learning knowledge from
  unsupervised sensorimotor interaction.
\newblock In \emph{International Conference on Autonomous Agents and MultiAgent
  Systems}, 2011.

\bibitem[Sutton et~al.(2016)Sutton, Mahmood, and White]{sutton2016anemphatic}
Richard~S. Sutton, A.~Rupam Mahmood, and Martha White.
\newblock An emphatic approach to the problem of off--policy
  temporal--difference learning.
\newblock \emph{The Journal of Machine Learning Research}, 17, 2016.

\bibitem[Thomas(2014)]{thomas2014bias}
Philip Thomas.
\newblock Bias in natural actor--critic algorithms.
\newblock In \emph{International Conference on Machine Learning}, 2014.

\bibitem[Tosatto et~al.(2020)Tosatto, Carvalho, Abdulsamad, and
  Peters]{tosatto2020nonparametric}
Samuele Tosatto, Joao Carvalho, Hany Abdulsamad, and Jan Peters.
\newblock A nonparametric off--policy policy gradient.
\newblock In \emph{International Conference on Artificial Intelligence and
  Statistics}, 2020.

\bibitem[Wang et~al.(2016)Wang, Bapst, Heess, Mnih, Munos, Kavukcuoglu, and
  de~Freitas]{wang2016sample}
Ziyu Wang, Victor Bapst, Nicolas Heess, Volodymyr Mnih, R{\'e}mi Munos, Koray
  Kavukcuoglu, and Nando de~Freitas.
\newblock Sample efficient actor--critic with experience replay.
\newblock In \emph{International Conference on Learning Representations}, 2016.

\bibitem[Watkins and Dayan(1992)]{watkins1992q}
Christopher~J.C.H. Watkins and Peter Dayan.
\newblock Q--learning.
\newblock \emph{Machine Learning}, 8\penalty0 (3-4), 1992.

\bibitem[Weaver and Tao(2001)]{weaver2001optimal}
Lex Weaver and Nigel Tao.
\newblock The optimal reward baseline for gradient--based reinforcement
  learning.
\newblock In \emph{Conference on Uncertainty in Artificial Intelligence}, 2001.

\bibitem[White(2015)]{white2015developing}
Adam White.
\newblock \emph{Developing A Predictive Approach To Knowledge}.
\newblock PhD thesis, University of Alberta, 2015.

\bibitem[White(2017)]{white2017unifying}
Martha White.
\newblock Unifying task specification in reinforcement learning.
\newblock In \emph{International Conference on Machine Learning}, 2017.

\bibitem[Williams(1992)]{williams1992simple}
Ronald~J. Williams.
\newblock Simple statistical gradient--following algorithms for connectionist
  reinforcement learning.
\newblock \emph{Machine Learning}, 8, 1992.

\bibitem[Williams and Peng(1991)]{williams1991function}
Ronald~J. Williams and Jing Peng.
\newblock Function optimization using connectionist reinforcement learning
  algorithms.
\newblock \emph{Connection Science}, 3\penalty0 (3), 1991.

\bibitem[Witten(1977)]{witten1977adaptive}
Ian~H. Witten.
\newblock An adaptive optimal controller for discrete--time {Markov}
  environments.
\newblock \emph{Information and Control}, 34\penalty0 (4), 1977.

\bibitem[Yu(2015)]{yu2015onconvergence}
Huizhen Yu.
\newblock On convergence of emphatic temporal--difference learning.
\newblock In \emph{Conference on Learning Theory}, 2015.

\bibitem[Zhang et~al.(2019)Zhang, Boehmer, and Whiteson]{zhang2019generalized}
Shangtong Zhang, Wendelin Boehmer, and Shimon Whiteson.
\newblock Generalized off-policy actor-critic.
\newblock In \emph{Conference on Neural Information Processing Systems}, 2019.

\bibitem[Zhang et~al.(2020)Zhang, Liu, Yao, and Whiteson]{zhang2020provably}
Shangtong Zhang, Bo~Liu, Hengshuai Yao, and Shimon Whiteson.
\newblock Provably convergent two--timescale off-policy actor--critic with
  function approximation.
\newblock In \emph{International Conference on Machine Learning}, 2020.

\end{thebibliography}

\nocite{harris2020array}
\nocite{brockman2016openai}
\nocite{blackhc.mdp}
\nocite{gym_minigrid}

\end{document}